\documentclass[12pt,a4paper]{article}
\usepackage{amsmath, amssymb, amsthm}
\usepackage{mathtools}
\usepackage{mathrsfs}
\usepackage{geometry}
\usepackage{enumitem}
\usepackage{hyperref}
\usepackage{notations}
\usepackage{graphicx}
\usepackage{subcaption}
\usepackage{bbm}
\usepackage{csquotes}
\usepackage{authblk}
\numberwithin{equation}{section}

\usepackage[
backend=biber,
style=numeric,sorting=none,maxbibnames=99
]{biblatex}
\usepackage{color}
\theoremstyle{remark}
\newtheorem{remark}{Remark}

\usepackage{comment}

\usepackage{dsfont}
\title{Variational Bounds for Perceptron Learning from Structured Data}

\author{Francesco Camilli\thanks{Corresponding author: \texttt{francesco.camilli2@unibo.it}}}
\author{Pierluigi Contucci}
\author{Federica Gerace}
\author{Emanuele Mingione}

\affil{Department of Mathematics, Alma Mater Studiorum -- Università di Bologna,
Piazza di Porta San Donato 5, 40126 Bologna, Italy}

\begin{document}

\maketitle

\begin{abstract}
We introduce a variational approach to a finite-temperature continuous-spin perceptron trained on a Gaussian mixture. The model allows for a broad class of concave utilities and log-concave separable prior measures on the spins. By combining the interpolation method with log-concavity and concentration estimates, we derive lower and upper minimax variational bounds for the limiting quenched pressure. Remarkably, the two bounds differ only in the order of optimization of two variational parameters, while all remaining extrema are controlled by the concave--convex structure of the variational potential. Whenever the two optimizations commute, the two bounds match and identify the solution of the model. The same potential yields the fixed-point equations as stationarity conditions and provides a unified route to the computation of the ground-state energy, training loss, and generalization error. 
\end{abstract}

\section{Introduction}

The perceptron is one of the earliest models of artificial learning and remains a fundamental testing ground for the mathematical theory of high-dimensional classification. Introduced by Rosenblatt as a model of pattern recognition \cite{Rosenblatt1958}, it was soon connected to the geometry of linear separability \cite{Cover1965}. Its statistical-mechanics formulation, initiated by Gardner and Derrida \cite{Gardner1988Space,GardnerDerrida1988}, revealed that even this elementary neural architecture displays nontrivial collective behavior when the number of patterns and the number of weights grow proportionally. The perceptron has since become a canonical model at the interface of spin glasses, statistics, and machine learning, and has played a central
role in the statistical-mechanical theory of learning from examples \cite{SeungSompolinskyTishby1992,WatkinRauBiehl1993,EngelVanDenBroeck2001}

This proportional regime has recently attracted renewed attention in high-dimensional statistics. Precise asymptotic results are now available for logistic regression, regularized $M$-estimators, generalized linear models, and convex empirical-risk minimization \cite{SurCandes2019,ThrampoulidisAbbasiHassibi2018,TaheriPedarsaniThrampoulidis2021,Aubin2020Perceptron,JeanGLM_PNAS}. Related AMP-based results characterize the zero-temperature asymptotics of convex teacher--student generalized linear models with rotationally invariant design matrices, under suitable uniqueness and concentration assumptions \cite{GerbelotAbbaraKrzakala2023}. For Gaussian-mixture data, in particular, the roles of the loss, regularization, signal strength, and sampling ratio have been characterized sharply \cite{MaiLiao2020,TaheriPedarsaniThrampoulidis2021GMM,mignacco20a_GMM_Gordon,Loureiro2021GMM}. Related works have studied universality beyond independent Gaussian covariates and the extent to which Gaussian mixtures describe learning curves on structured or real data \cite{Gerace2024Universality,Pesce2023Universality}, including estimators obtained by sampling from the associated finite-temperature Gibbs measure \cite{DandiEtAl2024}. These results illustrate the usefulness of simple random-data models for isolating phenomena that are difficult to detect through worst-case generalization bounds.

From the spin-glass perspective, the rigorous analysis of the perceptron was developed through the cavity method.  The cavity formulation of Gardner's computation was developed by M{\'e}zard, who also derived the associated TAP equations \cite{Mezard1989Cavity}. Talagrand proved the Gardner--Derrida formula at sufficiently high temperature for a broad class of binary perceptron models \cite{Talagrand2000Halfspaces} and for the Gaussian perceptron, he established overlap concentration in the high-temperature, low-density regime \cite{Talagrand2002GaussianPerceptron}. Shcherbina and Tirozzi proved the Gardner formula for a continuous-spin model \cite{ShcherbinaTirozzi2003}. These works obtain sharp thermodynamic results and identify the relevant scalar self-consistency equations. Complementary approaches based on Gordon's comparison theorem and the convex Gaussian min--max theorem provide scalar descriptions of several convex learning problems, especially at zero temperature \cite{ThrampoulidisAbbasiHassibi2018,Aubin2020Perceptron,mignacco20a_GMM_Gordon}. However, a direct finite-temperature variational formulation was not available in the present non-Bayes-optimal continuous-spin setting, with a general separable log-concave prior spin measure. For discrete spins, recent progress includes a new proof of the Gardner formula for Ising perceptrons at small densities, based on moment methods conditioned on approximate-message-passing iterates
\cite{BolthausenNakajimaSunXu2022}. In a closely related direction, Sáenz and Sur characterize finite-dimensional posterior marginals in high-dimensional generalized linear models, proving through leave-one-out arguments that they converge to Gaussian tilts of the prior \cite{SaenzSur2025}.

In this paper we consider a perceptron trained on a mixture of two Gaussian clouds whose centers are aligned with a random direction. The weights are distributed according to a Gibbs measure determined by a concave utility, a log-concave single-site prior measure. It can thus be thought as a mismatched inference problem, where the inferential model, the perceptron, does not match the data generating process, a Gaussian mixture. Nevertheless, the perceptron manages to reconstruct the hyperparameters of the generating process, thus solving an inverse problem. In this work we study the thermodynamic properties, specifically the limit of the log partition function, or quenched pressure, associated with the Gibbs measure of this inverse problem.

We derive nested minimax variational bounds for the limiting quenched pressure, both at positive and zero temperature. In both cases, the two bounds differ only in the order of the $\min$ and the $\max$ computation. All the remaining extrema can be exchanged through the concave--convex structure of the potential. Whenever the two outer optimizations commute, the thermodynamic limit therefore exists and is given by a single variational principle.

The variational nature of the result is a central feature of our approach. The usual fixed-point equations are recovered as stationarity conditions of one scalar potential, rather than introduced as a system of self-consistency relations. The potential retains global information: it provides a rule for comparing different stationary points and places the pressure, the ground-state energy, the training loss, and the generalization error within the same framework. It also transfers the possible obstruction to an exact formula from uniqueness of the fixed-point system to the exchangeability of two explicitly identified optimizations. In the Gaussian-prior and random-label limit, our equations reduce to those obtained by the classical cavity analysis (see e.g.\ \cite{Tala_vol1}).

Our proof uses an evolution of Guerra and Toninelli's interpolation technique \cite{interp_guerra_2002,Guerra_upper_bound}, called adaptive interpolation, introduced by Barbier and Macris \cite{adaptive_original,adaptive_CW}. The method is very effective for Bayes-optimal inference problems, where Nishimori identities \cite{nishimori01} and correlation inqualities \cite{Morita_Nishi_Griffith,ContucciMoritaNishimori2006} provide order parameter concentrations \cite{BarbierPanchenko2022}. Here these tools are unavailable. The argument instead relies on log-concavity, Brascamp--Lieb estimates \cite{BrascampLieb1976}, and adaptive choices of the interpolation path. Related results establish multioverlap concentration and strong
replica symmetry for broad classes of disordered log-concave Gibbs measures
\cite{BarbierPanchenkoSaenz2022}. This continues a line of work showing that adaptive interpolation can also be useful outside the Bayes-optimal setting \cite{Camilli_mismatch,CamilliContucciMingione2024}. We mention that the problem lacks the usual convexity à la Guerra, that in many cases (e.g. \cite{adaptive_CW} and \cite{MSKNL} as examples) complements the standard thermodynamic convexity of generating functionals leading to matching bounds. This complementary behavior is in some sense recovered through log-concavity, that imposes another type of correlation inequalities, implying in turn new concavity properties on the log partition function.

Besides providing a new route to the perceptron free energy, this formulation may offer a useful starting point for more structured architectures. Extending cavity computations from a single perceptron to multilayer systems requires controlling several coupled families of order parameters and cavity removals. A variational interpolation, in contrast, may allow these parameters to be organized directly within a global potential. Recent progress on the information-theoretic analysis of fully trained multilayer networks shows that such models are becoming accessible in controlled high-dimensional regimes \cite{CamilliTieplovaBergaminBarbier2025,BarbierCamilliNguyenPastoreSkerk2026}. We leave the discussion of this perspective to the conclusions.

The paper is organized as follows. In Section~\ref{sec:definitions}, we define the model and state the variational bounds, the stationarity equations, and the formulas for the training loss, generalization error, and ground-state energy. Section~\ref{sec:proofs} contains the adaptive interpolation argument, the convexity analysis, and the zero-temperature and perturbative limits. In Section~\ref{sec:numerics}, we investigate the matching of the variational bounds for representative loss functions. The appendices collect the ODE, moment, log-concavity, and concentration estimates used in the proofs.

\section{Definitions and main results}\label{sec:definitions}
Let $N,M$ be two positive integers. We introduce the \emph{patterns}, or \emph{inputs}, as a set of $M$ i.i.d. random $N$-dimensional vectors, that can be grouped in the $M\times N$ matrix
\begin{align}
    G=(g_{\mu})_{\mu\leq M}=(g_{\mu,i})_{\mu\leq M}^{i\leq N}\,.
\end{align}
The law of the inputs is determined by the following relation
\begin{align}\label{eq:GMM}
    g_\mu=\frac{\lambda}{\sqrt{N}}y_\mu \theta+Z_\mu
\end{align}where $\lambda\geq 0$, $y_\mu\iid(\delta_{-1}+\delta_1)/2$,  $Z_\mu\iid\mathcal{N}(0,\mathbbm{1}_N)$ and $\theta\in\mathbb{R}^N$, with its components $\theta_i\iid P_\theta$ with $P_\theta$ satisfying Poincarè inequality ($\int f^2dP_\theta\leq c_\theta \int (f')^2dP_\theta$ for a ceneterd function $f$, with Poincaré constant $c_\theta$) with unit second moment. The variable $y_\mu$ is called the \emph{label} of the $\mu$-th input, while $\theta$ is called \emph{centroid}. We denote the set of input-label couples as $\mathcal{D}=\{ (g_\mu,y_\mu)_{\mu\leq M}\}$.

Equation \eqref{eq:GMM} defines a Gaussian mixture, with two isotropic Gaussian clouds centered at $\pm\theta/\sqrt{N}$. If $\lambda=0$ there is no dependency between inputs and labels, which would set us in the perceptron with random labels setting, where the two clouds are perfectly superposed. The scaling of the noise part $Z_\mu$ and the signal part containing $\theta$ is such that separating the two clouds with a hyperplane is nor impossible for some $\lambda>0$, nor always trivial. It is feasible admitting a certain number of errors. In thermodynamic terms, we shall see this is equivalent to impose standard scalings on the Hamiltonian of the problem. For    $y\in\{-1,1\}$  let  $u_y,\phi:\R\to \R$ such that $u_y,\phi\in C^2(\R)$ and 
\begin{itemize}

\item[H1)] $u_y$ is concave and non affine  and $\exists\, C > 0$ such that

\be -C(1 + s^2)\leq  u_y(s) \leq 0\,,\,\quad |u_y'(s)| \le C(1 + |s|)\ee
    
    \item[H2)] $\phi$ is nonnegative, convex and $L$-Lipschitz;
    \item[H3)] $\max(\|u_y''\|_{\infty},\|\phi''\|_{\infty})\leq C$.

 \item[H4)] 
 \begin{equation}
    |u_{+1}(s)-u_{-1}(s)|
    \leq
    C(1+|s|),
    \qquad s\in\mathbb R.
    \label{eq:linear-label-discrepancy}
\end{equation}

\end{itemize}

\begin{remark}
Relating to machine learning and optimization literature, the utility function is the so-called \emph{loss function} up to a sign: $\ell(x,y)=-u_y(x)$. The loss functions used for classification purposes, such as the logistic loss of Section \ref{sec:numerics}, verifies all the above hypotheses. The class H1-H4 is big enough to include also the quadratic case $\ell(x,y)=(x-y)^2/(2\Delta)$ with $\Delta$ a properly tuned parameter, and the smoothed hinge loss.
\end{remark}

Given a realization of $\mathcal{D}$ we define the random Hamiltonian on the configurations space $\R^N\ni \vs=(\sigma_i)_{i\leq N}$ as 
\begin{align}\label{eq:Hamiltonian}
-H_N (\vs) = \sum_{\mu \leq M} 
u_{y_\mu}\left(\frac{{g}_{\mu}\cdot\vs}{\sqrt{N}} \right) - \sum_{i\leq N} \phi(\sigma_i) -\frac{\kappa}{2}\|\vs\|^2 
\end{align}
where $\kappa >0$. The variables in $\sigma$ will be referred to as \emph{spins} or \emph{weights}. Given $\beta>0$ the (random) Gibbs measure induced by the Hamiltonian $H$ is the measure on $\R^N$ with density
\begin{align}
    \label{eq:Gibbs}
\nu_N(\sigma)=\frac{e^{-\beta H_N(\sigma)}}{Z_N(\beta,\lambda)}
\end{align}over the Lebesgue measure, where $Z_N$ is the normalization called \emph{partition function}. Observe that, if $u_y$ is affine, then the Hamiltonian decouples over spin sites, thus making the model explicitly integrable. The notation $\langle \cdot \rangle$ denotes the averages over independent (conditionally on the quenched disorder) samples from $\nu_N(\sigma)d\sigma$, called \emph{replicas}. For instance, for a function $f$ of two replicas:
\begin{align}\label{eq:randomaverage}
\left\langle f\right\rangle_N =\int_{\R^N\times\R^N} d x\, d y \,\nu_N(x)\nu_N(y) f(x,y)\,.
\end{align}

Our goal is the computation of the quenched pressure of the model
\begin{align}
    p_N(\beta,\lambda) = \frac{1}{N}\, \E_{\mathcal{D}} \log Z_N(\beta,\lambda)\label{eq: pressure}
\end{align}
in the propotional limit
\begin{align}
    N,M\to\infty\quad\text{with}\quad \frac{M}{N}\to\alpha\,.
\end{align}This scaling regime is the natural one for studying the linear separability and storage capacity of random patterns \cite{Cover1965,Gardner1988Space,GardnerDerrida1988}.

We now introduce our variational potential. 
\begin{defn}[Variational potential]
For $r,\delta\geq0, \rho\geq q\geq 0$ and $m,h\in\mathbb{R}$ define 
\begin{align}\label{eq:def_psi_u}
&\Psi_{\beta u}(\rho, q, m)
  = \E_{y,v} \log \E_\xi\, e^{\,\beta u_y\,\left(\lambda ym+\,\sqrt{q}\,v + \sqrt{\rho-q}\,\xi\right)}\\
  \label{eq:def_psi_phi}
    &\Lambda_{\phi}(r, \delta, h) = \E_{z,\theta}\log\int d\sigma\;
    e^{-\beta \mathcal{H}_\phi (\sigma;r,\delta,h)},
\end{align}
where 
\begin{align}\label{eq:1body_Hamiltonian}
    \mathcal{H}_\phi (\sigma;r,\delta,h)=\phi(\sigma)+\frac{\kappa+\delta}{2}\sigma^2-(z\sqrt{r}+\lambda h\theta)\sigma
\end{align}and $\xi,v,z$ are independent, $\mathcal{N}(0,1)$, $\theta\sim P_\theta$, and $y\sim(\delta_{-1}+\delta_1)/2$. The variational potential is
\begin{align}
    \Phi_u (\rho, q, m, r, \delta, h)
  = \frac{\beta \delta\, \rho}{2}
    - \frac{\beta^2 r(\rho-q)}{2}
    + \alpha\,\Psi_{\beta u}(\rho, q, m)
    + \Lambda_{\phi}(r, \delta, h)-\lambda \beta mh.
\end{align}\label{def:var_pot}
For later convenience we also introduced the reduced potential
\begin{align}\label{eq:reduced_potential}
    \Phi_u^\star(\rho,r) =\sup_{q,m}\inf_{\delta,h}\Phi_u(\rho,q,m,r,\delta,h)
\end{align}with $\delta\in[0,\alpha C]$, $h^2\in[0,\alpha r]$, $q\in[\max(0,\rho-(\beta\kappa)^{-1}),\rho]$, $m^2\in[0,\rho]$.
\end{defn}

The parameters $\rho,q$ and $m$ have the usual meanings of self-overlap, replica overlap, and alignment with the centroid, respectively. The parameters $r,\delta$ and $h$ are their conjugate scalar-channel variables. Their precise relations are given by the stationarity equations in Proposition \ref{prop:FP_equations} below.

We are finally in place to state our main result:
\begin{thm}[Lower and upper bounds for the quenched pressure]\label{thm:main}
For any $\beta,\alpha>0,\lambda\geq 0$ and under the hypothesis $\mathrm{H1-H4}$ one has 
\begin{align}
  \liminf_{N\to\infty} p_N(\beta,\lambda)
  &\;\geq\;
  \sup_{\rho\geq0}\inf_{r\geq0} \Phi_u^\star(\rho,r)\,,\label{eq:lower_bound}\\
  \limsup_{N\to\infty} p_N(\beta,\lambda)
  &\;\leq\;
  \inf_{r\geq 0} \sup_{\rho\geq0} \Phi_u^\star(\rho,r)\,.\label{eq:upper_bound}
\end{align}
\end{thm}
In particular, whenever the two outer optimizations over $\rho$ and $r$ can be exchanged, the thermodynamic limit of the pressure exists and is given by the common variational value. The theorem does not require uniqueness of the full stationary-point system.

Exact high-dimensional asymptotics for related regularized classification problems have also been obtained through Gaussian comparison and approximate-message-passing methods \cite{Aubin2020Perceptron,Loureiro2021GMM}.

The proof of Theorem \ref{thm:main} uses the celebrated interpolation technique, originally introduced by Guerra and Toninelli for the Sherrington-Kirkpatrick model \cite{interp_guerra_2002,Guerra_upper_bound}, in its adaptive version, due to Barbier and Macris \cite{adaptive_original,adaptive_CW}. Leaving out the Bayes-optimal setting analysed in \cite{JeanGLM_PNAS}, from which this paper also draws inspiration, forcing interpolation to work on such a general perceptron is the main novelty of our work. In fact, not only we are out of optimal inferential settings (see \cite{JeanGLM_PNAS}) which introduce many crucial symmetries, called Nishimori identities \cite{nishimori01}, but we also allow for a rather general class of prior measure on the spins. The robustness w.r.t.\ this class of priors is inherited by the interpolation method. 

As a consequence of Nishimori identities, one can also prove useful correlation inequalities \cite{ContucciMoritaNishimori2006,Morita_Nishi_Griffith} that in some cases are \enquote{all the convexity you need} to prove Guerra-type bounds, namely by discarding terms of definite sign \cite{MSKNL,DBMNL,ChenMourratXia2022}. Here though, since they are not available, the useful structural convexity is provided by log-concavity. We shall indeed see that a repeated use of Prékopa-Leindler theorem (see \cite{BrascampLieb1976} for instance), asserting that the convolution of log-concave distributions is still log-concave, yields important properties of the variational potential, see Proposition \ref{prop:propertiesPhi} and Lemma \ref{lem:inequality_S}, that enable us to obtain a bound in the direction opposite to the usual convexity of the thermodynamic generating functionals.

The two bounds we have in Theorem \ref{thm:main} differ only for an exchange of $\sup_\rho$ and $\inf_r$. The other optimizations actually commute as a consequence of Sion's theorem \cite{Sion1958} and are guaranteed to have unique solutions. We shall indeed see that $\Phi_u$ is jointly convex in $\delta,h$ and jointly concave in $q,m$, all of which are bounded inside convex and compact sets. Hence optimizations over $r$ and $\rho$ are the only source of possible mismatches of the two bounds. Unfortunately, we were not able to prove the matching under sufficiently general criteria, but for functions $u_y$ that are of interest for optimization and machine learning we provide numerical evidence that the matching occurs for several $\beta,\alpha,\kappa$ in Section \ref{sec:numerics}. The possibility of exchanging these two optimizations has a similar role to that of the uniqueness requirement of Talagrand \cite{Tala_vol1} and \cite{Barbier_regression} on the solutions of the fixed point equation system contained in the Proposition below.

\begin{prop}[Stationarity conditions]\label{prop:FP_equations}
The critical points of $\Phi$ satisfy the following equations:
\begin{align}\label{eq:FP_equations}
\begin{split}
    &(\delta)\quad\rho=-\frac{2}{\beta}\partial_\delta \Lambda_\phi(r,\delta,h)\\
    &(h)\quad\lambda\beta m= \partial_h\Lambda_\phi(r,\delta,h)\\
    &(q)\quad r=-\frac{2\alpha}{\beta^2}\partial_q\Psi_{\beta u}(\rho,q,m)\\
    &(m)\quad\lambda\beta h=\alpha\partial_m\Psi_{\beta u}(\rho,q,m)\\
    &(r)\quad\rho-q=\frac{2}{\beta^2}\partial_r\Lambda_\phi(r,\delta,h)\\
    &(\rho)\quad\delta=\beta r-\frac{2\alpha}{\beta}\partial_\rho\Psi_{\beta u}(\rho,q,m)
\end{split}
\end{align}
Introduce the notations
\begin{align}\label{eq:bracket_phi}
    &\langle\cdot\rangle_\phi=\frac{\int d\sigma e^{-\beta \mathcal{H}_\phi (\sigma;r,\delta,h)}(\cdot)}{\int d\tau e^{-\beta \mathcal{H}_\phi (\tau;r,\delta,h)}}\\
    \label{eq:bracket_u}
    &\langle\cdot\rangle_{\beta u}=\frac{\E_\xi e^{\beta u_y(\nu)}(\cdot)}{\E_\xi e^{\beta u_y(\nu)}}\,,\quad \nu:=\lambda y m+\sqrt{q}v+\sqrt{\rho-q}\xi\,.
\end{align}
Then the stationarity equations rewrite as
\begin{align}\label{eq:FP_equations_bis}
    \begin{split}
        &\rho=\E\langle\sigma^2\rangle_\phi\,, \quad 
        m= \E\theta\langle\sigma\rangle_\phi\,,\quad
        q=\E\langle\sigma\rangle_\phi^2
        \\
        &r=\alpha\E\langle u'(\nu)\rangle_{\beta u}^2\,,\quad \delta =\beta \alpha\Big[\E\langle u'(\nu)\rangle_{\beta u}^2-\E\Big\langle (u'(\nu))^2+\frac{u''(\nu)}{\beta}\Big\rangle_{\beta u}\Big]\,,\quad
        h=\alpha \E y\langle u_{y}'(\nu)\rangle_{\beta u} .
    \end{split}
\end{align}    
\end{prop}
\begin{remark}
    The above set of equations and the variational potential match the results obtained by Talagrand \cite{Tala_vol1} for the case of Gaussian spin prior ($\phi\equiv0$) and random labels ($\lambda=0$). Indeed, when $\lambda=0$ both $m$ and $h$ disappear from the variational potential. $r-\delta=\bar r$ in Talagrand's notation and for Gaussian spin prior $\langle\cdot\rangle_\phi$ is actually a Gaussian measure, so expressions for $\rho$ and $q$ in \eqref{eq:FP_equations_bis} can be explicitly computed yielding the cavity equations.
    
\end{remark}

From our master Theorem \ref{thm:main}, with some further mild assumptions, we are able to derive formulae for the so-called \emph{traning loss} and, more importantly, \emph{generalization error} of the perceptron. Both observables are obtained by adding a scalar perturbation to the Hamiltonian, identifying the corresponding derivative of the pressure, and then differentiating the limiting variational formula.

\begin{corollary}[Training loss]\label{cor:Training_error}
Suppose $u$ is such that for $u_{\gamma}=\gamma u$ the two bounds in Theorem \ref{thm:main} match for all $\gamma\in(1-\epsilon,1+\epsilon)$ for a fixed $\epsilon>0$, thus identifying the thermodynamic limit. Define $\nu^*=\lambda y m^*+\sqrt{q^*}v+\sqrt{\rho^*-q^*}\xi$, with $\rho^*,q^*,m^*$ and $r^*,\delta^*,h^*$ unique solvers of the variational principle for $u_1=u$.

Then, the training loss has the following asymptotics
\begin{align}\label{eq:training_error}
    -\varepsilon_t:=\lim_{N\to\infty}\frac{1}{M}\E\Big\langle \sum_{\mu=1}^M u_{y_\mu}\Big(\frac{g_\mu\cdot\sigma}{\sqrt{N}}\Big)\Big\rangle=\E\frac{\E_\xi u_y(\nu^*)e^{\beta u_y(\nu^*)}}{\E_\xi e^{\beta u_y(\nu^*)}}.
\end{align}
\end{corollary}

\begin{corollary}[Generalization error]\label{cor:gen_error}
Let $\alpha'>0$, $y'\sim\frac{\delta_{-1}+\delta_1}{2}$ and $g'=\frac{\lambda}{\sqrt{N}}y' \theta+Z'$, with $Z'\sim\mathcal{N}(0,\mathbbm{1}_N)$ all independent of the quenched disorder, and a function $W_y:\mathbb{R}\longrightarrow\mathbb{R}$ with $\|W_y\|_\infty$,$\|W'_y\|_\infty$, $\|W''_y\|_\infty\leq C_w$ for a.e.\ $y$. 
Introduce
    \begin{align}\label{eq:phi_aux_gen_error}
        \Phi_{\tilde u,\gamma}(\rho,q,m,r,\delta,h)&:=\frac{\beta\delta\rho}{2}-\frac{\beta^2r(\rho-q)}{2}+\alpha \Psi_{\beta u}(\rho,q,m)+\alpha'\Psi_{\beta \tilde u_\gamma}(\rho,q,m)\nonumber\\
        &\qquad\qquad\qquad \qquad\qquad\qquad \qquad\qquad+\Lambda_\phi(r,\delta,h)-\lambda\beta mh\,,
    \end{align}and suppose there exists an $\eta>0$ such that for all $\gamma\in[0,\eta)$ and $t\in[-\eta,\eta]$, $\tilde u_{y,\gamma}=\gamma(u_y+t W_y)$ is concave and the optimizations in \eqref{eq:lower_bound} and \eqref{eq:upper_bound} yield the same result for $\Phi_{\tilde u,\gamma}$. Let $\rho^*,q^*,m^*$ be as in Corollary \ref{cor:Training_error}. Then
    \begin{align}\label{eq:gen_error}
    \varepsilon_g^W := \lim_{N\to\infty}\E\Big\langle  W_{y'}\Big(\frac{g'\cdot\sigma}{\sqrt{N}}\Big) \Big\rangle= \E\E_\xi W_{y'}(\lambda y' m^*+\sqrt{q^*}v'+\sqrt{\rho^*-q^*}\xi)\,.
\end{align}
\end{corollary}
In the above, $\alpha'$ has to be interpreted as resulting from additional $M'$ observations of input-output couples. The introduction of these fresh samples is necessary to generate the generalization error \eqref{eq:gen_error} from derivatives of the quenched pressure and of its limit. The influence of these new examples on the Gibbs measure is then forced to disappear by sending $\gamma\to0^+$ in such derivatives.

The additional assumptions we need for these corollaries are \enquote{local} validity of variational principles for small variations of the function $u_y$. These are needed because the averages of interest are generated through derivatives, which require existence of a proper variational formula in a whole neighborhood of $u_y$, though as small as you want.

From Theorem \ref{thm:main} one can compute a properly rescaled $\beta\to\infty$ limit, i.e.\ the $0$-temperature limit, that yields the ground state energy of the perceptron.

\begin{prop}[Ground state energy]\label{prop:GS}
    Introduce the rescaled potential
    \begin{align}
        \label{eq:def_tildePhi}
        \tilde\Phi_u(X,\rho,m,r,\delta,h) =\frac{\delta \rho}{2}-\frac{rX}{2}+\alpha\tilde\Psi_{u}(\rho,X,m)+\tilde \Lambda_\phi(r,\delta,h)-\lambda mh
    \end{align}with
    \begin{align}
    &\tilde\Psi_u(\rho,X,m)=\E_{y,v}\sup_\xi \Big[u_y(\lambda y m+\sqrt{\rho}v+\sqrt{X}\xi)-\frac{\xi^2}{2}\Big]\\
    &\tilde \Lambda_\phi (r,\delta,h)=\E_z\sup_\sigma\Big[
        (\sqrt{r}z+\lambda\theta h)\sigma-\phi(\sigma)-\frac{\kappa+\delta}{2}\sigma^2
        \Big]\,.
    \end{align}Denote the average ground state energy by
    \begin{align}
        \label{eq:def_GS}
        e_N=\mathbb{E}\inf_\sigma \frac{H_N(\sigma)}{N}=\lim_{\beta\to\infty}-\frac{1}{\beta}p_N(\beta,\lambda)\,,
    \end{align}
    where the average is w.r.t. the quenched randomness in the Hamiltonian. Then
    \begin{align}
    \liminf_{N\to\infty} -e_N&\;\geq\;
    \sup_{\rho\geq0} \inf_{r\geq 0} \sup_{X\in[0,\kappa^{-1}]}\sup_{m^2\leq \rho} \inf_{\delta\in[0,\alpha C]}\inf_{h^2\leq\alpha r}\; \tilde\Phi_u(X,\rho,m,r,\delta,h) \,,\\
    \limsup_{N\to\infty} -e_N&\;\leq\; \inf_{r\geq 0}  \sup_{\rho\geq0} \sup_{X\in[0,\kappa^{-1}]}\sup_{m^2\leq \rho} \inf_{\delta\in[0,\alpha C]}\inf_{h^2\leq\alpha r}\; \tilde\Phi_u(X,\rho,m,r,\delta,h)\,.
    \end{align}
\end{prop}
The proof of Proposition \ref{prop:GS} closely mimics ideas from physics. It involves a reparameterization of the variational bounds \eqref{eq:lower_bound} and \eqref{eq:upper_bound}. The key is that, thanks to log-concavity, the new parameter $X\approx \beta(\rho-q)$ remains bounded even when $\beta\to\infty$. In fact, the two parameters $\rho$, the \enquote{self overlap}, and $q$, the \enquote{overlap}, collapse onto one another when $\beta$ diverges.

\begin{remark}
    Concerning Corollary \ref{cor:gen_error}, using the class of functions we chose for $W_y$, we can approximate arbitrarily well, via a density argument, also the usual generalization error, which amounts to $W_y(x)=(y-\text{sign}(x))^2/4$. Then, letting $\beta\to\infty$ in \eqref{eq:gen_error}, and assuming the $\beta$-rescaled variational bounds in Proposition \ref{prop:GS} match, one can prove the generalization error approaches
    \begin{align*}
        \varepsilon_g=\frac{1}{2}\Big(1-\E \,y'\, \text{sign}(\lambda y' \tilde m+\sqrt{\tilde \rho}v')\Big)=\frac{1}{2}\Big(1-\E \, \text{sign}(\lambda \tilde m+\sqrt{\tilde \rho}v')\Big)
    \end{align*}with $\tilde m$ and $\tilde \rho$ solving the variational principle of Proposition \ref{prop:GS}. Calling $Q(x)=\int_x^\infty\frac{ds}{\sqrt{2\pi}}e^{-\frac{s^2}{2}}$ one readily gets
    \begin{align}\label{eq:gen_err_0T}
        \varepsilon_g=Q\Big(\frac{\lambda \tilde m}{\sqrt{\tilde \rho}}\Big)
    \end{align}which matches the result of \cite{mignacco20a_GMM_Gordon}, which uses Gordon's minimax theorem \cite{Gordon1985,ThrampoulidisOymakHassibiCGMT} to reduce the original high-dimensional optimization problem to a scalar auxiliary problem. 

    At $\beta\to\infty$ the Gibbs distribution concentrates on the minimum of the Hamiltonian function, which is convex. $\tilde m$ is thus the expected alignment of this minimum point with the direction of the centroid $\theta$ and is directly connected to the generalization performance \eqref{eq:gen_err_0T}. This is telling us that what the perceptron is learning through its weights is actually the direction $\theta$, that is orthogonal to the hyperplane that separates at best the two clouds of points.
    
    Despite \eqref{eq:gen_err_0T} appears in the same way as \cite{mignacco20a_GMM_Gordon}, our results apply to a general class of prior measures on the spins, or equivalently, regularizations. The choice of the regularization is thus hidden in the values $\tilde \rho$ and $\tilde m$ solving the $\beta\to\infty$ variational principle. 
    Furthermore, approaches based on Gordon's minimax theorem, require uniqueness of a system of fixed point equations, thus suffering from the same limitations of the cavity approaches \cite{Tala_vol1,Barbier_regression}. We, on the other hand, get that very system from a variational principle, where the limitation is transferred to the exchangeability of two optimizations.
\end{remark}

\section{Proofs}\label{sec:proofs}

\paragraph{Roadmap of the proof} The proof is organized around the exact interpolation identity of Proposition~\ref{prop:sum_rule}, the so called \emph{sum rule}. The interpolating Hamiltonian \eqref{eq:interpolating_H} connects the original perceptron at $t=0$ to a collection of independent scalar channels at $t=1$, or gas of spins, whose pressure is computed in Lemma~\ref{lemma:endpoints}. Differentiating along the interpolation path and performing Gaussian integration by parts yields the sum rule \eqref{eq:p_N}, in which the error term \eqref{eq:omega_N} is written as a sum of products between fluctuations of the order parameters and discrepancies between their Gibbs averages and the velocities of the interpolation path. The upper and lower bounds are then obtained through two complementary adaptive choices. For the upper bound, we fix $\dot r(t)=r$, choose $(\dot\delta,\dot h)$ through the scalar minimization in $(\delta,h)$, and evolve $(\rho,q,m)$ according to the Gibbs averages of $(Q_{11},Q_{12},M_1)$, as in \eqref{eq:ODEs_upper}. For the lower bound, we instead fix $\dot\rho(t)=\rho$, choose $(\dot q,\dot m)$ through the scalar maximization in $(q,m)$, and evolve $(r,\delta,h)$ according to the Gibbs averages of $(S_{12},S_{12}-S_{11},H_1)$, as in \eqref{eq:ODEs_lower}. Proposition~\ref{prop:propertiesPhi} provides the convex--concave structure needed to apply Jensen's and minimax arguments, while Lemma~\ref{lem:inequality_S} guarantees that the lower-bound path remains in the physical region $r,\delta\geq0$. The moment estimates and concentration results collected in the appendices ensure global existence of the two interpolation paths and, through Remark~\ref{rem:concentration_remainder}, make $\Omega_N$ vanish in the thermodynamic limit. The two constructions therefore produce bounds that differ only in the order of the outer optimizations over $\rho$ and $r$. The remaining parts of the section derive the zero-temperature result by introducing $X=\beta(\rho-q)$ and applying uniform Laplace asymptotics, and obtain the training and generalization observables by differentiating suitable scalar perturbations of the pressure.

\subsection{Adaptive interpolation}
Our proof relies on the interpolation method introduced in \cite{adaptive_original}, which builds on the celebrated Guerra interpolation, successfully employed for the Sherrington-Kirkpatrick model \cite{Guerra_upper_bound}. For perceptron-like models, our adaptive interpolation is inspired by Talagrand's cavity computation \cite{Tala_vol1} for the Shcherbina Tirozzi model \cite{ShcherbinaTirozzi2003}.

We introduce the interpolating Hamiltonian:
\begin{align}
    - H_N^t(\sigma,\xi)=\sum_{\mu\leq M}u_{y_\mu}(S_{t\mu})+\sqrt{r(t)} \ z\cdot\sigma -\sum_{i\leq N}\phi(\sigma_i)-\frac{\kappa+\delta(t)}{2}\|\sigma\|^2+\lambda h(t)\theta\cdot\sigma
\label{eq:interpolating_H}
\end{align}
where
\begin{align}
    S_{t\mu}=\sqrt{\frac{1-t}{N}}Z_\mu\cdot\sigma +(1-t)\lambda y_\mu \frac{\theta\cdot\sigma}{N}+ \lambda y_\mu m(t) + v_\mu\sqrt{q(t)}+ \xi_\mu\sqrt{\rho(t)-q(t)}.
\label{eq:interpolating_gen_spin}
\end{align}
with $v_\mu, \xi_\mu \iid\mathcal{N}(0, 1)$, and $r(t)$, $\delta(t)$, $h(t)$, $q(t)$, $m(t)$, $\rho(t)$ are generic functions of $t\in[0,1]$, which, by construction, have to satisfy the following conditions:

\begin{enumerate}[label=\roman*)]
    \item $q(t), r(t),\delta(t)\geq 0$;
    \item $q(t) \leq \rho(t)$;
    \item $q(0)=\rho(0)=r(0)=\delta(0)=m(0)=h(0) = 0$.
\end{enumerate}
In the following we denote $\alpha_N=M/N$.

From the above Hamiltonian, the definitions of interpolating partition function, Gibbs measure density and pressure respectively follow:
\begin{align}\label{eq:interpolating_Z}
    Z^t_N(\beta, \lambda) &= \mathbb{E}_{\xi} \int d^N\sigma e^{- \beta H^t_N(\sigma,\xi)}\\
    \label{eq:Gibbs_Interp}
     \nu_t(\sigma,\xi)&=\frac{e^{-\beta H^t_N(\sigma,\xi)}}{Z^t_N(\beta, \lambda)}\\
    \label{def:interpolating_pressure}
    p_N^t(\beta, \lambda)&=\frac{1}{N}\E\log\E_\xi\int d^N\sigma e^{-\beta H_N^t(\sigma,\xi)}\,.
\end{align}The outer quenched expectation in \eqref{def:interpolating_pressure} averages over the variables $(Z_\mu,v_\mu,y_\mu)_{\mu\leq M}, z,\theta$. For later convenience we define the following quantities
\begin{align}
    \label{eq:def_S11-S12}
    &S_{11}=\frac{1}{N}\sum_{\mu=1}^M[\beta u''_{y_\mu}(S_{t\mu})+ \big(\beta u'_{y\mu}(S_{t\mu})\big)^2],\quad
    S_{12}=\frac{\beta^2}{N}\sum_{\mu=1}^Mu'_{y_\mu}(S_{t\mu}^{(1)})u'_{y_\mu}(S_{t\mu}^{(2)})\\
    \label{eq:def_Qab}
    &Q_{ab}=\frac{\sigma^{(a)}\cdot\sigma^{(b)}}{N}\,,\quad a,b=1,2\\
    \label{eq:def_Ma_Ha}
    &M_{a}=\frac{\theta\cdot\sigma^{(a)}}{N}\,,\quad H_a=\frac{1}{N}\sum_{\mu=1}^My_\mu u'_{y_\mu}(S_{t\mu}^{(a)})\,,\quad a=1,2
\end{align}where superscripts like $(a)$ denote replica indices, i.e.\ independent samples from \eqref{eq:Gibbs_Interp} conditionally on the quenched disorder.

\begin{lemma}[Endpoints]

At the endpoints of the interpolation path it holds:
\begin{align}
    p^0_N(\beta, \lambda) &= p_N(\beta, \lambda) \label{eq:pressure_t=0}\\
    p^1_N(\beta, \lambda) &= \alpha_N \Psi_{\beta u} (\rho(1), q(1), m(1)) + \Lambda_\phi(r(1), \delta(1), h(1))\label{eq:pressure_t=1}
\end{align}
with $\Psi_{\beta u}$ and $\Lambda_{\phi}$ as in Definition \ref{def:var_pot}.
\label{lemma:endpoints}
\end{lemma}

\begin{proof} To prove Lemma~\ref{lemma:endpoints}, it suffices to evaluate $H_N^t(\sigma,\xi)$ at the endpoints $t=0$ and $t=1$. The case $t=0$ is straightforward. Indeed, condition iii) immediately implies that
$H_N^0(\sigma,\xi)=H_N(\sigma)$, from which \eqref{eq:pressure_t=0} follows directly once we use eq. \eqref{def:interpolating_pressure}. 

At $t=1$, the quenched and annealed random variables in the interpolating Hamiltonian decouple. Consequently,

\begin{equation}
    - H^1_N(\sigma) = \sum_{\mu = 1}^M u_{y_\mu}\left( S_{1\mu} \right) + z\cdot \sigma \sqrt{r(1)} -\sum_{i \leq N} \phi(\sigma_i) - \frac{\kappa + \delta(1)}{2} ||\sigma||^2 + \lambda h(1)\theta\cdot\sigma
    \label{eq:interpolating_H_t=1}
\end{equation}
where:
\begin{equation}
    S_{1\mu} = \lambda y_\mu m(1) + v_{\mu} \sqrt{q(1)} + \xi_\mu \sqrt{\rho(1) - q(1)}
    \label{eq:interpolating_gen_spin_t=1}
\end{equation}

It then follows from eq.~\eqref{def:interpolating_pressure}, that the interpolating pressure decomposes into two independent contributions: the first depends only on the pattern index, while the second depends only on the spin variables. More precisely,

\begin{equation}
\begin{split}
    p^1_N (\beta, \lambda) &=\frac{1}{N} \mathbb{E}_{y, v}\log \mathbb{E}_{\xi} \ e^{\beta\sum_{\mu = 1}^M u_{y_\mu}\left( \lambda y_\mu m(1) + v_{\mu} \sqrt{q(1)} + \xi_\mu \sqrt{\rho(1) - q(1)} \right) }\\ 
    &+ \frac{1}{N} \mathbb{E}_{z,\theta}\log \int d^N\sigma e^{\beta \left( z\cdot \sigma \sqrt{r(1)} -\sum_{i \leq N} \phi(\sigma_i) - \frac{\kappa + \delta(1)}{2} ||\sigma||^2 + \lambda h(1)\theta\cdot\sigma\right)}
\end{split}
\end{equation}
The first contribution factorizes over the samples, whereas the second factorizes over the spin coordinates. Using $\alpha_N=M/N$, we obtain
\begin{equation}
\begin{split}
    p^1_N (\beta, \lambda) &= \alpha_N \mathbb{E}_{y,v}\log \mathbb{E}_{ \xi} \ e^{\beta u_{y}\left( \lambda y m(1) + v \sqrt{q(1)} + \xi \sqrt{\rho(1) - q(1)} \right)  } +\\
    & + \mathbb{E}_{z,\theta}\log \int d\sigma e^{-\beta\left(  \phi(\sigma) + \frac{\kappa + \delta(1)}{2} \sigma^2 -(z \sqrt{r(1)}+ \lambda h(1)\theta) \sigma\right)}
\end{split}
\end{equation}
This is precisely the decomposition into the variational potentials $\Psi_{\beta u}$ and $\Lambda_{\phi}$ introduced in Definition~\ref{def:var_pot}, thereby yielding \eqref{eq:pressure_t=1}.
\end{proof}

\begin{prop}[Sum rule] Given Lemma \ref{lemma:endpoints}, the pressure is:
    \begin{align}
        p_N(\beta,\lambda)&=
        \alpha_N\Psi_{\beta u}(\rho(1),q(1),m(1))
        +\Lambda_\phi (r(1),\delta(1),h(1))\nonumber\\
        &+\beta\int_0^1 dt\Big[\frac{\dot\delta(t)\dot\rho(t)-\beta \dot r(t)(\dot\rho(t)-\dot q(t))}{2}-\lambda\dot h(t) \dot m(t)\Big]+\Omega_N \label{eq:p_N}
    \end{align}where the remainder is 
    \begin{align}
        \Omega_N&=\frac{1}{2}\int_0^1 dt \ \mathbb{E} \Big \langle \left(\dot{\rho}(t) - Q_{11}\right) ( \beta(\beta\dot{r}(t) - \dot{\delta}(t)) - S_{11} ) - \left(\dot{q}(t) - Q_{12}\right)(\beta^2\dot{r}(t) - S_{12}) \Big \rangle_t \nonumber \\
    & +\beta\lambda \int_0^1 dt \ \mathbb{E} \Big\langle   (\dot{h}(t) - H_1)(\dot{m}(t) - M_1)\Big\rangle_t. \label{eq:omega_N}
    \end{align}\label{prop:sum_rule}
\end{prop}

\begin{proof} The proof of Proposition~\ref{prop:sum_rule} is based on the sum rule:
\begin{equation}
    p^0_N(\beta, \lambda) = p^1_N(\beta, \lambda) - \int_0^1 \dot{p}^t_N (\beta, \lambda) \ dt \label{eq:sum-rule}
\end{equation}
From Lemma~\ref{lemma:endpoints}, the endpoints of the interpolation path can be expressed explicitly as

\begin{equation}
    p_N(\beta, \lambda) = \alpha_N \Psi_{\beta u}(\rho(1), q(1), m(1)) + \Lambda_\phi(r(1), \delta(1), h(1)) - \int_0^1 \dot{p}^t_N (\beta, \lambda) \ dt \label{eq:p_N_implicit}
\end{equation}
Therefore, it remains to compute the integral of the derivative of the interpolating pressure along the interpolation path. Differentiating eq.~\eqref{def:interpolating_pressure} yields
\begin{equation}
    \dot{p}_N^t(\beta, \lambda) = \frac{1}{N}\mathbb{E} \Big\langle - \beta\dot{H}^t_N (\sigma,\xi) \Big\rangle_t 
\label{eq:p_dot}
\end{equation}
where $\langle\cdot\rangle_t$ denotes expectation with respect to the Gibbs measure associated with the interpolating Hamiltonian. Differentiating the interpolating Hamiltonian in eq.~\eqref{eq:interpolating_H} gives

\begin{equation}
\begin{split}
     - \beta \dot{H}^t_N (\sigma,\xi) =  \beta\sum_{\mu = 1}^M u^{\prime}_{y_\mu}\left( S_{t\mu} \right)\dot{S}_{t\mu} + \beta z \cdot \sigma \frac{\dot{r}(t)}{2\sqrt{r(t)}} - \beta \frac{\dot{\delta}(t)}{2}||\sigma||^2 + \beta \lambda \dot{h}(t)\theta \cdot \sigma
\end{split} \label{eq:dHt/dt}   
\end{equation}
where, by differentiating \eqref{eq:interpolating_gen_spin},
\begin{equation}
\begin{split}
    \dot{S}_{t\mu} = -\frac{1}{2\sqrt{(1-t)N}}\sigma \cdot Z_{\mu} - \lambda y_\mu \frac{\theta \cdot \sigma}{N} +\lambda y_\mu \dot{m}(t)+v_\mu\frac{\dot{q}(t)}{2\sqrt{q(t)}} + \xi_\mu \frac{\dot{\rho}(t) - \dot{q}(t)}{2\sqrt{\rho(t) - q(t)}}
\end{split} \label{eq:dst/dt} 
\end{equation}
Substituting \eqref{eq:dHt/dt} and \eqref{eq:dst/dt} into \eqref{eq:p_dot}, we obtain an explicit expression for the derivative of the interpolating pressure, composed of several terms:
\begin{equation}
    \dot{p}^t_N (\beta, \lambda) = I + II + III + IV + V + VI + VII + VIII
    \label{eq:p_dot_in_terms}
\end{equation}

We analyze each contribution separately. The first term couples the pattern noise with the spin variables:
\begin{equation}
    \begin{split}
        I = -\frac{\beta}{N} \mathbb{E} \Big\langle \sum_{\mu = 1}^M u_{y_\mu}^{\prime}\left( S_{t\mu} \right)\frac{ \sigma \cdot Z_{\mu}}{2\sqrt{(1-t)N}} \Big\rangle_t
    \end{split}
\end{equation}
At this stage, we apply Stein's lemma, $\mathbb{E}_x [xf(x)] = \mathbb{E}_x[f^{\prime}(x)]$, with $x = Z_{\mu}$. Differentiating with respect to all the components $Z_{\mu}$ we get:
\begin{equation}
\begin{split}
     I &= -\frac{1}{2N} \mathbb{E} \left[\Big \langle \sum_{\mu = 1}^M \left(\beta u^{\prime \prime}_{y_\mu}(S_{t\mu}) 
    + (\beta u^{\prime}_{y_\mu}(S_{t\mu}))^2\right) \frac{|| \sigma||^2}{N} \Big \rangle_t  + \right.\\
    &\hspace{55mm}\left. - \Big\langle \sum_{\mu = 1}^M \beta^2 u^{\prime}_{y_\mu}\left(S^{(1)}_{t\mu}\right)  u^{\prime}_{y_\mu}\left(S^{(2)}_{t\mu}\right)  \frac{\sigma^{(1)} \cdot \sigma^{(2)}}{N} \Big\rangle_t\right].
\end{split}
\end{equation}

The second and third contributions are treated analogously. The only difference is that the quenched Gaussian variables are $v_\mu$ and $z_i$ instead of $z_{\mu}$,
\begin{equation}
    II = \frac{\beta}{N} \mathbb{E}  \Big\langle\sum_{\mu = 1}^M u_{y_\mu}^{\prime}\left( S_{t\mu} \right)\frac{\dot{q}(t)}{2\sqrt{q(t)}}v_\mu \Big\rangle_t \hspace{10mm} III=\frac{\beta}{N} \mathbb{E} \Big\langle \frac{\dot{r}(t)}{2\sqrt{r(t)}} z \cdot \sigma\Big\rangle_t
\end{equation}
Applying Stein's lemma once again, now with $x=v_\mu$ and $x=z_i$ respectively, yields
\begin{equation}
\begin{split}
    &II = \frac{1}{2N} \mathbb{E}  \left[ \Big \langle \sum_{\mu = 1}^M\left(\beta u^{\prime \prime}_{y_\mu}(S_{t\mu}) + (\beta u^{\prime}_{y_\mu}(S_{t\mu}))^2\right) \dot{q}(t) \Big \rangle_t +\right.\\
    & \hspace{55mm} \left.- \Big\langle \sum_{\mu = 1}^M \left(\beta^2 u^{\prime}_{y_\mu}\left(S^{(1)}_{t\mu}\right) u^{\prime}_{y_\mu}\left(S^{(2)}_{t\mu}\right)\right) \dot{q}(t) \Big\rangle_t\right] \\
    &III =  \frac{1}{2} \mathbb{E} \left[ \beta^2\left( \Big\langle \frac{||\sigma||^2}{N} \Big\rangle_t - \Big\langle \frac{\sigma^{(1)} \cdot \sigma^{(2)}}{N} \Big\rangle_t \right)\dot{r}(t)\right]
\end{split}
\end{equation}

The fourth contribution is slightly different, as it involves the annealed random variable $\xi_\mu$ rather than a quenched one:
\begin{equation}
    IV=\frac{\beta}{N} \mathbb{E}  \Big\langle\sum_{\mu = 1}^M u^{\prime}(S_{t\mu}) \frac{\dot{\rho}(t) - \dot{q}(t)}{2\sqrt{\rho(t) - q(t)}} \xi_\mu \Big \rangle_t
\end{equation}
Stein's lemma can be applied once again, now with $x=\xi_\mu$. In contrast to the previous cases, however, the derivative acts only on the Gibbs weight and on prefactors of $\xi_\mu$, since the partition function is independent of $\xi_\mu$. Differentiating with respect to $\xi_\mu$ therefore yields
\begin{equation}
    IV=\frac{1}{2N} \sum_{\mu = 1}^M\mathbb{E} \left[\Big \langle \left( \beta u^{\prime \prime}_{y_\mu}(S_{t\mu}) + (\beta u^{\prime}_{y_\mu}(S_{t\mu}))^2 \right) \left( \dot{\rho}(t) - \dot{q}(t) \right) \Big \rangle_t\right]
\end{equation}

In each of the preceding contributions, Stein's lemma introduces overlap terms between distinct spin configurations, thereby considerably simplifying the resulting expressions. The remaining contributions, however, are already naturally expressed in terms of the overlap with the centroid or the norm of the weight vector. Consequently, no integration by parts is required, and these terms are left unchanged. More precisely,
\begin{equation}
    \begin{split}
        V &= -\frac{1}{N} \mathbb{E} \Big\langle \sum_{\mu = 1}^M \beta u_{y_\mu}^{\prime}\left( S_{t\mu} \right) \lambda y_\mu\frac{\theta \cdot \sigma}{N} \Big\rangle_t \hspace{10mm}  VI = \frac{1}{N} \mathbb{E} \Big\langle \sum_{\mu = 1}^M \beta u^\prime_{y_\mu}(S_{t\mu}) \lambda y_\mu\dot{m}(t)\Big\rangle_t\\
        &VII = - \frac{1}{N} \mathbb{E}\Big\langle \beta||\sigma||^2 \frac{\dot{\delta}(t)}{2}\Big\rangle_t \hspace{10mm} VIII=\frac{1}{N} \mathbb{E}\Big\langle \beta\lambda (\theta \cdot \sigma) \dot{h}(t) \Big\rangle_t 
    \end{split}
\end{equation}

Substituting all these terms into \eqref{eq:p_dot_in_terms}, using the definitions \eqref{eq:def_S11-S12}-\eqref{eq:def_Ma_Ha}, and integrating the derivative of the interpolating pressure along the interpolation path as in \eqref{eq:p_N_implicit}, we obtain \eqref{eq:p_N} and \eqref{eq:omega_N}. This completes the proof.
\end{proof}

\begin{remark}
    Observe that in the proof of the sum rule we did not rely on the knowledge of the prior measure of the spins, which distinguishes our approach from the previous ones and thus enlarges the class of perceptron models we can tackle.
\end{remark}

\begin{remark}
    \label{rem:concentration_remainder} The form of the remainder $\Omega_N$ is in a particularly convenient form. In fact, in order to deal effectively with it we just need to prove concentration w.r.t.\ the quenched measure of the order parameters $Q_{11},Q_{12},M_{1}$ and boundedness of the second moments of $S_{11},S_{12},H_1$. These properties hold and are proved in Propositions \ref{pro:linear growth}, \ref{prop:conc} in the Appendix. We can thus split the expectations in \eqref{eq:omega_N} simply using Cauchy-Schwartz inequality paying with a negligible error that, as we shall see, is uniform in $t$ and can thus be discarded. 

    Since these results are obtained with methods analogous to those of \cite{Barbier_regression} and \cite{Tala_vol1}, the proofs of these bounds and concentrations are deferred to Appendix \ref{app:concentrations}

\end{remark}

\subsection{Convexities}
Here we list some further preliminary results that shall be used throughout. We start with some crucial convexity properties of the variational potential:

\begin{prop}[Convexity properties of $\Phi$]\label{prop:propertiesPhi}
Assume $u$ concave and non-affine and $\lambda>0$. $\Phi$ has the following properties:
\begin{enumerate}[label=\roman*)]
  \item $\Lambda_\phi$ and $\Phi$ are strictly  convex in $(\delta,h)$;
  \item $\Psi_{\beta u}$ and $\Phi$ are strictly concave in $(q,m)$.
\end{enumerate}
\end{prop}
\begin{proof}
\emph{(i)}. Recall the definitions \eqref{eq:1body_Hamiltonian} and \eqref{eq:bracket_phi}, and introduce
\begin{align}
    F_\phi(r,\delta,h)= \log\int d\sigma e^{-\beta\mathcal{H}_\phi (\sigma;r,\delta,h)}\,.
\end{align}
With this notation we have
\begin{align}
    \beta^{-2} \text{Hess}_{\delta,h}F_\phi=\begin{pmatrix}
        \frac{1}{4}\mathbb{V}[\sigma^2]&-\frac{\lambda\theta}{2}[\langle\sigma^3\rangle_\phi-\langle\sigma^2\rangle_\phi\langle\sigma\rangle_\phi]\\
        -\frac{\lambda\theta}{2}[\langle\sigma^3\rangle_\phi-\langle\sigma^2\rangle_\phi\langle\sigma\rangle_\phi]&\lambda^2\theta^2\mathbb{V}[\sigma]\,.
    \end{pmatrix}
\end{align}
where variances and covariance will be intended w.r.t.\ $\langle\cdot \rangle_\phi$. Its determinant reads
\begin{align}
    \det \beta^{-2}\text{Hess}_{\delta,h}F_\phi=\frac{\lambda^2\theta^2}{4}\big[\mathbb{V}[\sigma]\mathbb{V}[\sigma^2]-\text{Cov}[\sigma,\sigma^2]^2\big]\,.
\end{align}
Whenever $\theta\neq 0$ the above is always non-negative thanks to Cauchy-Schwartz. Moreover, it is degenerate if and only if
\begin{align}
    \sigma^2-\langle\sigma^2\rangle_\phi=a (\sigma-\langle\sigma\rangle_\phi)\,,\text{ for some }a\in\mathbb{R}
\end{align}i.e.\ if the two centered variables involved in the variances are proportional to each other. This can happen iff $\sigma$ is supported on the two points that solve the above quadratic equation. It is not the case since $\sigma$ is a continuous variable.

Finally, since the diagonal elements of $\text{Hess}_{\delta,h}F_\phi$ are positive with positive probability, by Sylvester's criterion $\text{Hess}_{\delta,h}F_\phi$ is positive definite with positive probability, and this grants joint strict convexity of $\Lambda_\phi$.

\vspace{5pt}

\emph{(ii)}. Consider the function
\begin{align}
    F_u(x,s)=\log \E_\xi e^{u(x+\sqrt{s}\xi)}\,.
\end{align}
The argument of the $\log$ is a convolution of a log-concave function with a Gaussian measure, which is itself log-concave. Prékopa-Leindler theorem (see, e.g.,
\cite{BrascampLieb1976}) indeed asserts that convolution preserves log concavity. As a consequence, $F(x,s)$ is concave in $x$. Observe also that the very same object is the solution to a heat equation, being the convolution with a Gaussian kernel. We can thus show that its logarithm satisfies
\begin{align}
    \partial_x F_u=\langle u'(x+\sqrt{s}\xi)\rangle_u,\quad&\partial_x^2 F_u=\langle u''(x+\sqrt{s}\xi)+(u'(x+\sqrt{s}\xi))^2\rangle_u-\langle u'(x+\sqrt{s}\xi)\rangle_u^2\nonumber\\
    &\partial_s F_u=\frac{1}{2}\langle u''(x+\sqrt{s}\xi)+(u'(x+\sqrt{s}\xi))^2\rangle_u
\end{align}where
\begin{align}
    \langle\cdot\rangle_u=\frac{\E_\xi e^{u(x+\sqrt{s}\xi)}(\cdot)}{\E_\xi e^{u(x+\sqrt{s}\xi)}}\,.
\end{align}Combining derivatives together we get
\begin{align}
    \partial_x^2 F_u+(\partial_x F_u)^2=2\partial_s F_u\,,
\end{align}
from which we also get
\begin{align}
    \partial_x\partial_s F_u=\frac{1}{2} \partial^3_x F+\partial_x^2 F_u\,\partial_x F_u\,.
\end{align}
Then, using these definitions, we can recast 
\begin{align}
\Psi_{u}(\rho, q, m)=\E_{y,v}F_u(\lambda ym+\sqrt{q}v,\rho-q)\,,
\end{align}and leverage the relations among derivatives to evaluate the Hessian of the above function in $(q,m)$.

To begin with,
\begin{align}
    \partial_m\Psi_{u}=\lambda \E_{y,v} y\,\partial_x F_u(\lambda ym+\sqrt{q}v,\rho-q)\,,\quad \partial_m^2\Psi_{u}=\lambda^2 \E_{y,v}\,\partial_x^2 F_u(\lambda ym+\sqrt{q}v,\rho-q)
\end{align}where we used $y^2=1$. Observe that, being $u$ non-affine, it is not difficult to prove that $F_u$ cannot be affine either, and $\partial_x^2 F_u<0$, entailing in turn $\partial_m^2\Psi_{u}< 0$. Then, using integration by parts of $v$
\begin{align}\label{eq:Psi_u_Lipschitz}
    \partial_q\Psi_{u}&=\E_{y,v}\big[
    \frac{v}{2\sqrt{q}}\partial_x F_u-\partial_s F_u
    \big]=
    \E_{y,v}\big[
    \frac{1}{2}\partial_x^2 F_u-\frac{1}{2}\partial_x^2 F_u-\frac{1}{2}(\partial_x F_u)^2
    \big]\nonumber\\
    &=-\frac{1}{2} \E_{y,v} (\partial_x F_u)^2\leq 0\,.
\end{align}
A further derivation yields
\begin{align}
    \partial_q^2\Psi_{u}&=-\E_{y,v} \partial_x F_u\big[
    \frac{v}{2\sqrt{q}}\partial_x^2 F_u-\partial_s  \partial_x F_u
    \big]\nonumber\\
    &=
    -\E_{y,v}\big[
    \frac{1}{2}\partial_x^3 F\partial_x F_u+\frac{1}{2}(\partial_x^2 F_u)^2-\frac{1}{2}\partial_x F_u\partial_x^3 F-(\partial_x F_u)^2\partial_x^2F 
    \big]\nonumber\\
    &=-\E_{y,v}\big[
    \frac{1}{2}(\partial_x^2 F_u)^2-(\partial_x F_u)^2\partial_x^2F 
    \big]<0
\end{align}where strict negativity follows again from strict concavity fo $F$ in $x$.

The mixed derivative finally reads
\begin{align}
    \partial_q\partial_m\Psi_{u}&=-\E_{y,v}\partial_x F_u \,\lambda y \partial_x^2 F_u\,.
\end{align}
The determinant of the Hessian thus reads
\begin{align}
    \det \text{Hess}\Psi_u& =\lambda^2\E_{y,v}\big[
    -\frac{1}{2}(\partial_x^2 F_u)^2+(\partial_x F_u)^2\partial_x^2 F_u 
    \big]\,\E_{y,v} \partial_x^2 F_u-\lambda^2\big(\E_{y,v}y\partial_x F_u\,\partial_x^2F\big)^2\nonumber\\
    &\geq \lambda^2\E_{y,v}\big[
    -\frac{1}{2}(\partial_x^2 F_u)^2+(\partial_x F_u)^2\partial_x^2 F_u 
    \big]\,\E_{y,v} \partial_x^2 F_u-\lambda^2 \E_{y,v}(\partial_x F_u)^2\,\partial_x^2F\,\E_{y,v}\partial_x^2F\nonumber\\
    &=
    -\frac{1}{2}\lambda^2\E_{y,v}\big[
    (\partial_x^2 F_u)^2\big]\E_{y,v} \partial_x^2 F_u
    >0
\end{align}through Cauchy-Schwartz inequality.
Thanks to Sylvester's criterion, $\text{Hess}\Psi_u$ is negative definite. Hence $\Psi_u$ is strictly concave. 

\end{proof}

We conclude with a consequence of log-concavity that will allow us to restrict the range of our interpolating functions:
\begin{lemma}\label{lem:inequality_S}
    The following inequality holds:
    \begin{align}
        \langle S_{12}\rangle_t\geq\langle S_{11}\rangle_t\,.
    \end{align}
\end{lemma}
\begin{remark}
    The above inequality is the log-concave substitute for the positivity of the interpolation remainder. It is the only sign condition used in the construction of the lower bound, where it guarantees that both $r(t)$ and $\delta(t)$ remain non-decreasing along the interpolation path, i.e. inside the physical region. In this sense it plays here the role that first-kind correlation inequalities play for the direct problem in mean-field spin glasses.
\end{remark}

\begin{proof}
    The statement is again a consequence of Prékopa-Leindler theorem. In fact, one can define
\begin{align}
    Y(x)=\frac{1}{N}\log\int d^N\sigma\E_\xi e^{-\beta H^t_N(\sigma,x)}\,,
\end{align}where $x=(x_\mu)_{\mu\leq M}$, and
\begin{align}
    -H^t_N(\sigma,x)=\sum_{\mu\leq M}u_{y_\mu}(S_{t\mu}+x_\mu)+z\cdot\sigma\sqrt{r(t)} -\sum_{i\leq N}\phi(\sigma_i)-\frac{\kappa+\delta(t)}{2}\|\sigma\|^2+\lambda h(t)\theta\cdot\sigma\,.
\end{align}The joint Boltzmannfaktor for $\sigma,\xi$ is straightforwardly log-concave. Hence $Y(x)$ is concave by Prékopa-Leindler theorem, as it is the logarithm of the convolution of log-concave functions. This in particular entails
\begin{align}
    \text{Tr}\,\text{Hess} \,Y(x=0)= \langle S_{11}\rangle_t-\langle S_{12}\rangle_t\leq 0\,,
\end{align}for every fixed realization of the disorder.
\end{proof}

\subsection{Bounds for the quenched pressure}

\begin{proof}[Proof of Theorem \ref{thm:main}] First, note that we can replace $\alpha_N$ by its limit $\alpha$ up to a vanishing remainder. In the following we shall thus use $\alpha$ only.

\vspace{5pt}

\textbf{Upper bound}: Let us make the following initial choice: $\dot r(t)=r>0$. Then, using Proposition \ref{prop:propertiesPhi}, in particular the joint convexity in $(\delta,h)$ of the functional, we get
\begin{align}
    p_N(\beta,\lambda)&\leq
        \alpha\Psi_{\beta u}(\rho(1),q(1),m(1))-\frac{\beta^2 r(\rho(1)-q(1))}{2}
        \nonumber\\
        &\qquad\qquad\quad+\int_0^1 dt\Big[\Lambda_\phi(r,\dot\delta(t),\dot h(t))+\frac{\beta \dot\delta(t)\dot\rho(t)}{2} -\beta\lambda \dot h(t)\dot m(t)\Big]+\Omega_N\,,
\end{align}where we used Jensen's inequality on $\Lambda_\phi$.
Now we choose the pair $\dot\delta,\dot h$ such that
\begin{align}
    (\dot\delta(t),\dot h(t))=\text{arg}\inf_{\delta,h\in T_r} \;\Lambda_\phi(r,\delta,h)+\frac{\beta \delta \dot\rho(t)}{2}-\beta\lambda h\dot m(t)=: F_{\delta,h}(T_r,r;\dot\rho(t),\dot m(t))
\end{align}which is uniquely identified due to strict convexity. $T_r\subset [0,\infty)\times\mathbb{R}$ here is a compact convex set to be chosen later. Over this set $\dot{\delta}\geq 0$ and then $\delta(t)\geq0$. Therefore, when inserted in the Gibbs measure, $\delta(t)$ corresponding to this choice preserves log-concavity and convergence simultaneously. Furthermore, since $T_r$ is compact, $F_{\delta,h}(T_r,r;\dot\rho(t),\dot m(t))$ is jointly continuous in $r,\dot\rho(t),\dot m(t)$ and bounded in $T_r$ by definition.

Then
\begin{align}
    p_N(\beta,\lambda)&\leq
        \alpha\Psi_{\beta u}(\rho(1),q(1),m(1))-\frac{\beta^2r(\rho(1)-q(1))}{2}
        \nonumber\\
        &\qquad\qquad\qquad\qquad+\int_0^1 dt\inf_{\delta,h\in T_r}\Big[\Lambda_\phi(r,\delta,h)
        +\frac{\beta \delta \dot\rho(t)}{2}-\beta\lambda h\dot m(t)\Big]+\Omega_N\nonumber\\
        &\leq
        \alpha\Psi_{\beta u}(\rho(1),q(1),m(1))-\frac{\beta^2r(\rho(1)-q(1))}{2}
        \nonumber\\
        &\qquad\qquad\qquad\qquad +\inf_{\delta,h\in T_r}\Big[\Lambda_\phi(r,\delta,h)
        +\frac{\beta \delta\rho(1)}{2}-\beta\lambda hm(1)\Big]+\Omega_N\nonumber\\
        &\leq \sup_{\rho\geq0} \sup_{q,m\in D_\rho}\inf_{\delta,h\in T_r}\Phi_u(\rho,q,m,r,\delta,h)+\Omega_N\,,
\end{align}
where $D_\rho=[\max(0,\rho-(\beta\kappa)^{-1}),\rho]\times[-\sqrt{\rho},\sqrt{\rho}]$ is compact and convex.
It is now time to choose the other interpolating functions $m(t),\rho(t),q(t)$. We choose them according to the following system of coupled ODEs in $t\in[0,1]$:
\begin{align}\label{eq:ODEs_upper}
\begin{split}
    &\dot m(t)=\E\langle M_1\rangle_t=F_m(t;r;\rho(t),q(t),m(t),\delta(t),h(t))\\
    &\dot \rho(t)=\E\langle Q_{11}\rangle_t=F_\rho(t;r;\rho(t),q(t),m(t),\delta(t),h(t))\\
    &\dot q(t)=\E\langle Q_{12}\rangle_t=F_q(t;r;\rho(t),q(t),m(t),\delta(t),h(t))\\
    &(\dot\delta(t),\dot h(t))=F_{\delta,h}(T_r,r;\dot\rho(t),\dot m(t))\,.
\end{split}
\end{align}
Observe that $F_m^2, F_q\leq F_\rho$ by Cauchy-Schwartz, which is consistent with the constraints imposed by $D_\rho$. Furthermore, using Brascamp-Lieb inequality \cite{BrascampLieb1976}, and that $\delta(t)\geq0$, one readily gets
\begin{align*}
    F_\rho-F_q=\frac{1}{N}\sum_{i=1}^N\E\langle\sigma_i^2-\langle\sigma_i\rangle_t^2\rangle_t\leq (\beta\kappa)^{-1}\quad\Rightarrow\quad \rho(t)-q(t)\leq (\beta\kappa)^{-1}\;\forall\,t\in[0,1]\,,
\end{align*}and hence, consistently with $D_\rho$, $q(t)\geq\rho(t)-(\beta\kappa)^{-1}$.

The $\delta,h$-components of the velocity field are bounded in $T_r$ by definition. Thanks to Proposition \ref{pro:linear growth} in Appendix, the velocity field grows at most as $K(1+m^2(t)+\rho(t)+rt)$ with a proper constant $K>0$. Therefore, by Lemma \ref{lem:global_existence}, we are guaranteed a global solution in $t\in[0,1]$ exists.

By plugging the choices in \eqref{eq:ODEs_upper} into the remainder, following Remark \ref{rem:concentration_remainder} and using Proposition \ref{prop:conc} in the Appendix, the remainder vanishes once we take the limit. We thus take the $\limsup$ on both sides of the bound, obtaining
\begin{align}
    \limsup_{N\to\infty}p_N(\beta,\lambda)\leq \sup_{\rho\geq0}\sup_{q,m\in D_\rho}\inf_{\delta,h\in T_r}\Phi_u(\rho,q,m,r,\delta,h)
\end{align}for all $r>0$. We then optimize the bound over $r$.
\vspace{5pt}

\textbf{Lower bound}: Let us start by choosing $\dot\rho=\rho>0$ constant. Then we can use the joint concavity of $\Psi_{\beta u}$ and the sum rule to get
\begin{align}
    p_N(\beta,\lambda)&\geq
    \Lambda_\phi(r(1),\delta(1),h(1)) +\frac{\beta\delta(1)\rho}{2} \nonumber\\
        &+\int_0^1 dt\Big[\alpha\Psi_{\beta u}(\rho,\dot q(t),\dot m(t)) - \frac{\beta^2\dot r(t)(\rho-\dot q(t))}{2}-\beta\lambda\dot h(t) \dot m(t)\Big]+\Omega_N\,,
\end{align}where we used Jensen's inequality for $\Psi_{\beta u}$.
Now we choose 
\begin{align}
    (\dot{q}(t),\dot m(t))&=\text{arg}\sup_{q,m\in D_\rho}\;
    \alpha\Psi_{\beta u}(\rho,q,m) - \frac{\beta^2\dot r(t)(\rho-q)}{2}-\beta\lambda\dot h(t) m \nonumber\\
    &=: F_{q,m}(D_\rho,\rho;\dot r(t),\dot h(t)).
\end{align}
Thanks to compactness $F_{q,m}(D_\rho,\rho;\dot r(t),\dot h(t))$ is jointly continuous in $\dot r(t),\dot h(t)$ and bounded in $D_\rho$ by definition. By plugging this choice in the sum rule we get
\begin{align}
    p_N(\beta,\lambda)&\geq
    \Lambda_\phi(r(1),\delta(1),h(1)) +\frac{\beta\delta(1)\rho}{2} \nonumber\\
        &+\int_0^1 dt\sup_{q,m\in D_\rho}\Big[\alpha\Psi_{\beta u}(\rho,q,m) - \frac{\beta^2\dot r(t)(\rho-q)}{2}-\beta\lambda\dot h(t) m\Big]+\Omega_N\nonumber\\
        &\geq \Lambda_\phi(r(1),\delta(1),h(1))+\frac{\beta\delta(1)\rho}{2} \nonumber\\
        &+\sup_{q,m\in D_\rho}\Big[\alpha\Psi_{\beta u}(\rho,q,m) - \frac{ \beta^2 r(1)(\rho-q)}{2}-\beta\lambda h(1) m\Big]+\Omega_N\nonumber\\
        &\geq \inf_{r\geq 0} \inf_{\delta,h\in T_r} \sup_{q,m\in D_\rho}\;\Big[\Lambda_\phi(r,\delta,h) +\frac{\beta \delta\rho}{2}+\alpha\Psi_{\beta u}(\rho,q,m) - \frac{ \beta^2 r(\rho-q)}{2}-\beta\lambda h m\Big]+\Omega_N\nonumber\\
        &\geq \inf_{r\geq 0}\sup_{q,m\in D_\rho}\inf_{\delta,h\in T_r}\Phi_u(\rho,q,m,r,\delta,h)+\Omega_N
\end{align}
Now we choose the remaining interpolating functions according to the following ODEs
\begin{align}\label{eq:ODEs_lower}
\begin{split}
    \beta^2\dot{r}(t)& =\E\langle S_{12}\rangle_t= F_r(t;\rho;r(t),q(t),m(t),\delta(t),h(t))\\
    \beta\dot\delta(t)&=\E\langle S_{12}\rangle_t-\E\langle S_{11}\rangle_t=F_\delta(t;\rho;r(t),q(t),m(t),\delta(t),h(t))\,\\
    \dot h(t)&=\E\langle
    H_1
    \rangle_t= F_h(t;\rho;r(t),q(t),m(t),\delta(t),h(t))\\
    (\dot q(t),\dot m(t))&=F_{q,m}(D_\rho,\rho;\dot{r}(t),\dot{h}(t))\,.
\end{split}
\end{align}
Note that $F_r\geq0$ and $F_\delta\geq 0$ by Lemma \ref{lem:inequality_S}. Furthermore, $F_h^2\leq \alpha F_r/\beta^2$ by Cauchy-Scwhartz inequality and $\beta^{-1}F_\delta\leq - \alpha \E\langle u''_{y_\mu}(S_{t\mu})\rangle_t\leq  \alpha C$ by an application of Jensen's inequality and H3. This allows us to take $T_r=[0,\alpha C] \times [-\sqrt{\alpha r},\sqrt{\alpha r}]$. Finally, by Proposition \ref{pro:linear growth} in Appendix, the velocity field of the above ODE system grows at most as $K(1+\rho t+h^2(t)+r(t))$. Therefore, following Lemma \ref{lem:global_existence}, there exists a global solution.

Plugging these choices in the remainder, following Remark \ref{rem:concentration_remainder} and leveraging our Proposition \ref{prop:conc}, we make the remainder vanish by taking the $\liminf$ on both sides of the main inequality:
\begin{align}
    \liminf_{N\to\infty} p_N(\beta,\lambda)\geq \inf_{r\geq 0}\sup_{q,m\in D_\rho}\inf_{\delta,h\in T_r}\Phi_u(\rho,q,m,r,\delta,h)
\end{align}
which can then be optimized w.r.t. $\rho$.
\end{proof}

\begin{remark}\label{rem:opt_contraints}
  Recall the fixed point equations for stationary points in Proposition \ref{prop:FP_equations}. By a simple application of Cauchy-Schwartz one immediately proves that $\rho\geq q,m^2$ and that $h^2\leq \alpha r$. Secondly, 
\begin{align}
    \delta=-\frac{2\alpha}{\beta}\big[\partial_\rho\Psi_{\beta u}(\rho,q,m)+\partial_q\Psi_{\beta u}(\rho,q,m)\big] =-\frac{\alpha}{\beta\lambda^2 }\partial_x^2 \Psi_{\beta u}(\rho,q,m+x)|_{x=0}\geq 0
\end{align}by Prékopa-Leindler theorem (recall $y^2=1$). In addition, a simple use of Jensen's inequality yields
\begin{align}
    \delta\leq-\alpha \E\langle u''_y(\nu)\rangle_{\beta u}\leq \alpha C
\end{align}by H1. 
Furthermore, since $\langle\cdot\rangle_\phi$ is log-concave, we can use Brascamp-Lieb inequality to show:
\begin{align}\label{eq:bound_rho-q}
        0\leq \langle\sigma^2\rangle_\phi-\langle\sigma\rangle_\phi^2\leq \beta^{-1}\langle\big(\phi''(\sigma)+\kappa+\delta\big)^{-1}\rangle_\phi\leq (\beta\kappa)^{-1}
\end{align}where we used that $\phi$ is convex and $\delta\geq 0$.

Therefore the constraints in $T_r$ and $D_\rho$ are automatically satisfied by stationary points.
\end{remark}

\subsection{Ground state energy}
To begin with, we prove the following:
\begin{lemma}[Reparameterization]
    The two bounds \eqref{eq:lower_bound} and \eqref{eq:upper_bound} can be recast as
    \begin{align}
    \liminf_{N\to\infty} p_N(\beta,\lambda)\geq &\sup_{\rho\geq0}\inf_{r\geq0} \sup_{m^2\leq \rho}\sup_{X\in[0,\kappa^{-1}]} \inf_{\delta\in[0,\alpha C]}\inf_{h^2\leq\alpha r}\;\Phi_u(\rho,\rho-\frac{X}{\beta},  m, r, \delta, h)\,,\\
    \limsup_{N\to\infty} p_N(\beta,\lambda)\leq& \inf_{r\geq 0} \sup_{q\geq 0} \sup_{X\in[0,\kappa^{-1}]}\sup_{m^2\leq q+X/\beta} \inf_{\delta\in[0,\alpha C]}\inf_{h^2\leq\alpha r}\;\Phi_u (q+\frac{X}{\beta}, q, m, r, \delta, h).
    \end{align}
    The fixed point equation for $X$ then becomes
    \begin{align}\label{eq:bound_X}
    X=\beta\big(\E\langle\sigma^2\rangle_\phi-\E\langle\sigma\rangle_\phi^2\big).
    \end{align}In particular, at stationarity $X\leq \kappa^{-1}$ is verified.
\end{lemma}
\begin{proof}
    The rewriting of the two bounds and of the fixed point equation for $X$ corresponds to a simple reparameterization $X=\beta(\rho-q)$ in the original variational function. 

    The bound $X\leq \kappa^{-1}$ is proved as in \eqref{eq:bound_rho-q}.
\end{proof}

Afterwards, we need the following scaling limits:
\begin{lemma}\label{lem:ZeroT_lims_Psis}
    Let $X\in[0,\kappa^{-1}]$ and $\delta\in[0,\alpha C]$. The following scaling relations hold
    \begin{align}
        &\lim_{\beta\to\infty}\beta^{-1}\Psi_{\beta u}(q+\frac{X}{\beta},q,m)=\E\sup_\xi\big[u_y(\lambda y m+\sqrt{q}v+\sqrt{X}\xi)-\frac{\xi^2}{2}\big]\,,\\
        &\lim_{\beta\to\infty}\beta^{-1}\Psi_{\beta u}(\rho,\rho-\frac{X}{\beta},m)=\E\sup_\xi\big[u_y(\lambda y m+\sqrt{\rho}v+\sqrt{X}\xi)-\frac{\xi^2}{2}\big]\,,\\
        &\lim_{\beta\to\infty}\beta^{-1}\Lambda_{\phi}(r,\delta,h) =\E\sup_\sigma\big[
        -\phi(\sigma)-\frac{\kappa+\delta}{2}\sigma^2+(z\sqrt{r}+\lambda h \theta)\sigma
        \big]\,.
    \end{align}
    The convergence is uniform over the variational parameters.
\end{lemma}
\begin{proof}
By a change of variable $\xi\mapsto\xi\sqrt{\beta}$ we get
\begin{align}
    \beta^{-1}\Psi_{\beta u}(q+\frac{X}{\beta},q,m)=\E A_\beta +\frac{1}{2\beta}\log\beta\,.
\end{align}with
\begin{align}
    A_\beta:=\frac{1}{\beta}\log\int \frac{d\xi}{\sqrt{2\pi}} e^{\beta U_q(\xi)}\,,\quad U_q(\xi):=
    u(\lambda y m+\sqrt{q}v+\sqrt{X}\xi)-\frac{\xi^2}{2}\,.
\end{align}
Thanks to H1 we have $-(1+C/\kappa)\leq-(1+CX)\leq U_q''(\xi)\leq-1$ and hence
\begin{align*}
    -\frac{1+C/\kappa}{2}(\xi-\xi^*)^2\leq U_q(\xi)-U_q(\xi^*)\leq -\frac{1}{2}(\xi-\xi^*)^2\,.
\end{align*}
with $\xi^*=\text{arg}\sup_\xi U_q(\xi)$. This in turn entails
\begin{align}
    U_q(\xi^*)-\frac{1}{2\beta}\log\beta(1+C/\kappa)\leq A_\beta\leq U_q(\xi^*)-\frac{1}{2\beta}\log\beta
\end{align}and thus
\begin{align}
    \Big|\beta^{-1}\Psi_{\beta u}(q+\frac{X}{\beta},q,m)-\E \sup_\xi U_q(\xi)\Big|\leq 
    \frac{1}{\beta} \log\beta (1+C/\kappa)\,.
\end{align}The r.h.s.\ is independent on $q,X,m$ impliying also uniform convergence if $X\leq\kappa^{-1}$. Note that this was possible only thanks to uniform bounds on the derivatives of $u$.

The second part of the statement is analogous a part from minor fixes. One just needs to prove that
\begin{align}\label{eq:uniform_aux_lemma}
    \Big|\beta^{-1}\Psi_{\beta u}(\rho,\rho-\frac{X}{\beta},m)-\frac{1}{\beta}\E\log\int \frac{d\xi}{\sqrt{2\pi}} e^{\beta\big[
    u(\lambda y m+\sqrt{\rho}v+\sqrt{X}\xi)-\frac{\xi^2}{2}
    \big]}\Big|\xrightarrow[]{\beta\to\infty}0\,,
\end{align}uniformly over variational parameters. This follows from a uniform Lipschitz property of $\Psi_{\beta}$. Define 
\begin{align}
    g(t):=\frac{1}{\beta}\E\log\int\frac{d\xi}{\sqrt{2\pi/\beta}}e^{\beta\big[
    u(\lambda y m+\sqrt{\rho-\frac{tX}{\beta}}v+\sqrt{X}\xi)-\frac{\xi^2}{2}
    \big]}\,,
\end{align}whence it is evident that $g(1)=\beta^{-1}\Psi_{\beta u}(\rho,\rho-\frac{X}{\beta},m)$, whereas $g(0)$ is our target op to $O(\beta^{-1}\log\beta)$. Deriving w.r.t\ $t$ we get
\begin{align}
\dot g(t)=-\frac{X}{2\beta\sqrt{\rho-\frac{tX}{\beta}}}\E \,v\,\langle u'(\nu(t))\rangle_{\beta u}
\end{align}
with $\nu(t):=\lambda y m+\sqrt{\rho-\frac{tX}{\beta}}v+\sqrt{X/\beta}\xi$. Note that in the square root in front of $\xi$, $X$ has been divided by $\beta$ to reuse the definition of  $\langle\cdot \rangle_{\beta u}$. An integration by parts yields
\begin{align}
    \dot g(t)=-\frac{X}{2\beta}\E \,\langle u''(\nu(t))\rangle_{\beta u}-\frac{X}{2}\E\big\langle
    \big(u'(\nu(t))\big)^2-\langle u'(\nu(t))\rangle_{\beta u}^2
    \big\rangle_{\beta u}\,.
\end{align}Thanks to H1 and that $X\in[0,\kappa^{-1}]$, the first term is uniformly bounded by an $O(\beta^{-1})$. For the second term, since it is a variance, we can use Brascamp-Lieb inequality for $\langle\cdot\rangle_{\beta u}$:
\begin{align}
    \E\big\langle
    \big(u'(\nu(t))\big)^2-\langle u'(\nu(t))\rangle_{\beta u}^2
    \big\rangle_{\beta u}\leq\frac{X C^2}{\beta}\leq \frac{C^2}{\kappa\beta}
\end{align}where we used again $|u''|\leq C.$ \eqref{eq:uniform_aux_lemma} has thus been proved. The rest of the proof proceeds as for the first claim.

For the third limit, call $\sigma^*=\text{arg}\inf_\sigma\mathcal{H}_\phi(\sigma)$. Observe that $-K:=-(\kappa+ C(1+\alpha))\leq -(\kappa+ \delta+C)\leq-\mathcal{H}_\phi''(\sigma)\leq-\kappa$, where we used $\delta\in[0,\alpha C]$. Hence
\begin{align}
    -\frac{K}{2}(\sigma-\sigma^*)^2\leq -\mathcal{H}_\phi(\sigma)+ \mathcal{H}_\phi(\sigma^*)\leq -\frac{\kappa}{2}(\sigma-\sigma^*)^2
\end{align}
which in turn yields
\begin{align}
    -\frac{1}{2\beta}\log\frac{\beta K}{2\pi}+\sup_\sigma-\mathcal{H}_\phi(\sigma)\leq \frac{1}{\beta}\log\int d\sigma e^{-\beta\mathcal{H}_\phi(\sigma)}\leq -\frac{1}{2\beta}\log\frac{\beta\kappa }{2\pi}+\sup_\sigma-\mathcal{H}_\phi(\sigma)\,.
\end{align}Finally
\begin{align}
    \Big|\beta^{-1}\Lambda_\phi(r,\delta,h)-\E\sup_\sigma-\mathcal{H}_\phi(\sigma)\Big|\leq \frac{1}{2\beta}\log\frac{\beta K}{2\pi}\,.
\end{align}
This is again uniform in the variational parameters $r,\delta,h$, provided $\delta\in[0,\alpha C]$. 
\end{proof}

\begin{lemma}\label{lemma:beta_scalings}
    The following scalings hold
    \begin{align}
        &\beta^{-1}\Phi_u(\rho,\rho-\frac{X}{\beta},  m, r, \delta, h)\xrightarrow[]{\beta\to\infty}\tilde\Phi_u(X,\rho,m,r,\delta,h)\\
        &\beta^{-1}\Phi_u(q+\frac{X}{\beta}, q, m, r, \delta, h)\xrightarrow[]{\beta\to\infty}\tilde\Phi_u(X,q,m,r,\delta,h)
    \end{align}uniformly over the variational parameters if $\delta\in[0,\alpha C]$, $X\in[0,\kappa^{-1}]$. Furthermore, $\tilde \Phi$ is jointly concave in $(m,X)$.
\end{lemma}
\begin{proof}By direct inspection:
    \begin{align}
        \frac{1}{\beta}\Phi_u(\rho,\rho-\frac{X}{\beta},  m, r, \delta, h)=\frac{\delta\rho}{2}-\frac{rX}{2}+\frac{\alpha}{\beta}\Psi_{\beta u}(\rho,\rho-\beta^{-1}X,m)+\frac{1}{\beta}\Lambda_\phi(r,\delta,h)-\lambda mh\,.
    \end{align}
    Using then Lemma \ref{lem:ZeroT_lims_Psis} we get the first scaling. The second is completely analogous.

    Note that $\Phi_u(\rho,\rho-\frac{X}{\beta},  m, r, \delta, h)$ is the composition of $\Phi$ with an affine transformation in the $q$-entry, w.r.t.\ which $\Phi$ is strictly concave, jointly with $m$. Hence, it is jointly concave in $(m,X)$, and so it remains when divided by $\beta$. Therefore, being $\tilde\Phi$ the limit of a jointly $(m,X)$-concave function, it remains jointly concave in $(m,X)$.
\end{proof}

\begin{remark}
    Since by Lemma \ref{lemma:beta_scalings} we have uniform convergence of the variational potential over the allowed optimization sets we also have that
    \begin{align}
        &\lim_{\beta\to\infty} \sup_{\rho}\inf_{r} \sup_{m,X} \inf_{\delta,h}\; \beta^{-1}\Phi_u(\rho,\rho-\frac{X}{\beta},  m, r, \delta, h)=\sup_{\rho}\inf_{r} \sup_{m,X} \inf_{\delta,h}\;\tilde\Phi_u(X,   \rho,m,r,\delta,h)\\
        &\lim_{\beta\to\infty} \inf_{r} \sup_{q} \sup_{m,X} \inf_{\delta,h}\; \beta^{-1}\Phi_u (q+\frac{X}{\beta}, q, m, r, \delta, h)=
        \inf_{r} \sup_{\rho} \sup_{m,X} \inf_{\delta,h} \tilde\Phi_u(X,   \rho,m,r,\delta,h),
    \end{align}
    where we omitted the optimization sets for brevity. However, keep in mind that the set $m^2\leq q+X/\beta\to m^2\leq q$ uniformly over the variational parameters other than $q$ for $\beta\to\infty$, thanks to the fact $X\in[0,\kappa^{-1}]$.
    
    Note that in the last line we have re-baptized $q$ as $\rho$, as it is a dummy variable, to have a more symmetric formula. Indeed, from a physical point of view, the self overlap $\rho$ and the overlap $q$ become equal in the $\beta\to\infty$ limit.
\end{remark}

\begin{lemma}
The following estimate holds
    \begin{align}
        -e_N-\frac{1}{2\beta}\log\frac{\beta K}{2\pi}\leq\frac{p_N}{\beta}\leq -e_N-\frac{1}{2\beta}\log\frac{\beta\kappa}{2\pi}
    \end{align}for a proper $K>\kappa$. As a consequence
    \begin{align}
        \lim_{\beta\to\infty}\liminf_{N\to\infty} \frac{p_N}{\beta}=\liminf_{N\to\infty} -e_N\,,\quad \lim_{\beta\to\infty}\limsup_{N\to\infty} \frac{p_N}{\beta}=\limsup_{N\to\infty} -e_N\,.
    \end{align}
\end{lemma}
\begin{proof}
    We start with a simple expansion of the Hamiltonian with Lagrange remainder, around its minimum point $\sigma^"$:
    \begin{align}
        -H_N(\sigma)=-H_N(\sigma^*)-\frac{1}{2}(\sigma-\sigma^*)\cdot\text{Hess}\,H_N(\bar \sigma)(\sigma-\sigma^*)
    \end{align}
    where
    \begin{align}
        -\text{Hess}H_N(\sigma)=\sum_\mu \frac{g_\mu g_\mu^\intercal}{N}u''_{y_\mu}(S_\mu)-\text{diag}(\phi''(\sigma_i))_{i\leq N}-\mathbbm{1}\kappa\,.
    \end{align}
    Using H1, the convexity of $\phi$ and concavity of $u_y$ we get
    \begin{align}
        -\mathbbm{1}(\kappa+C+C\lambda_{\rm max}(G^\intercal G/N))\preceq-\text{Hess}H_N(\sigma)\preceq-\mathbbm{1}\kappa\,.
    \end{align}
    From these bounds we readily get
    \begin{align}
        \frac{p_N}{\beta}&\geq-e_N+\frac{1}{N\beta}\E\log \int d^N\sigma e^{-\frac{\beta (\kappa+C+C\lambda_{\rm max}(G^\intercal G/N))}{2}\|\sigma-\sigma^*\|^2}\\
        \frac{p_N}{\beta}&\leq -e_N+\frac{1}{N\beta}\E\log \int d^N\sigma e^{-\frac{\beta\kappa}{2}\|\sigma-\sigma^*\|^2}\,.
    \end{align}The result then follows by Gaussian integration. In particular,
    \begin{align}
        \frac{1}{N\beta}\E\log \int d^N\sigma &e^{-\frac{\beta (\kappa+C+C\lambda_{\rm max}(G^\intercal G/N))}{2}\|\sigma-\sigma^*\|^2}=-\frac{1}{2\beta}\E\log\frac{\beta(\kappa+C+C\lambda_{\rm max}(G^\intercal G/N))}{2\pi}\nonumber\\
        &\geq-\frac{1}{2\beta}\log\frac{\beta}{2\pi}
        (\kappa+C+C\E\lambda_{\rm max}(G^\intercal G/N))\,.
    \end{align}
    Since $\E\lambda_{\rm max}(G^\intercal G/N)$ is bounded by a constant depending on $M/N$, the lower bound is proved. The upper bound is proved with another simple Gaussian integration.
\end{proof}

The proof of Proposition \ref{prop:GS} then follows.

\subsection{Training and generalization errors}
\begin{proof}[Proof of Corollary \ref{cor:Training_error}] Define the following auxiliary problem:
\begin{align}\label{eq:bar_p_N}
    \bar p_N(\beta,\lambda)=\frac{1}{N}\E\log\int d^N\sigma\exp\beta\Big[\sum_{\mu=1}^M \frac{\beta_t}{\beta}u_{y_\mu}\Big(\frac{g_\mu\cdot\sigma}{\sqrt{N}}\Big)-\sum_{i=1}^N\phi(\sigma_i)-\frac{\kappa\|\sigma\|^2}{2}\Big]
\end{align}
The above, after denoting $\bar u_y=u_y\beta_t/\beta$, is again a perceptron model, which can be treated exactly as in Theorem \ref{thm:main}. Moreover, the training error can be identified as 
\begin{align}
    -\varepsilon_{N,t}=\frac{1}{\alpha}\frac{\partial \bar p_N}{\partial\beta_t}\Big|_{\beta_t=\beta},
\end{align}and $\bar p_N$ is convex in $\beta_t$. In addition,
\begin{align}
    0\leq \frac{\partial^2 \bar p_N}{\partial\beta_t^2}=\frac{1}{N}\E\Big\langle\Big(\sum_\mu u_{y_\mu}\big(S_{\mu0}\big)\Big)^2-\Big\langle\sum_\mu u_{y_\mu}\big(S_{\mu0}\big)\Big\rangle^2\Big\rangle
\end{align}Using Brascamp-Lieb inequality one readily gets
\begin{align}
    \frac{\partial^2 \bar p_N}{\partial\beta_t^2}\leq\E\Big\langle\frac{1}{\kappa N}\sum_{\mu,\nu=1}^Mu'_{y_\mu}(S_{\mu0})u'_{y_\nu}(S_{\nu0})\frac{g_\mu\cdot g_\nu}{N}
    \Big\rangle\leq \E\Big\langle\frac{1}{\kappa N}\sum_{\mu=1}^M(u'_{y_\mu}(S_{\mu0}))^2
    \Big\rangle\|\frac{GG^\intercal}{N}\|_{op}.
\end{align}
Recall that $\|\frac{GG^\intercal}{N}\|_{op}\leq C(\alpha)$ with probability exponentially close to $1$. On the other hand, thanks to H1, $(u'_y(s))^2\leq 2C^2(1+s^2)$. Using this, combined with Proposition \ref{pro:linear growth} we conclude that $\partial^2_{\beta_t} \bar p_N$ is bounded. This implies directly that $\partial_{\beta_t} \bar p_N$ converge \cite[Theorem 1.1]{AC98}. Therefore, the exchange of the $N$-limit and the $\beta_t$ derivative is allowed:
\begin{align}
    -\varepsilon_t=\frac{1}{\alpha}\lim_{N\to\infty}\frac{\partial \bar p_N}{\partial\beta_t}\Big|_{\beta_t=\beta} =\frac{1}{\alpha} \frac{\partial \bar p}{\partial\beta_t}\Big|_{\beta_t=\beta}\,.
\end{align}Under the hypothesis of the corollary:
\begin{align}
    \bar p=\sup_\rho \inf_r\sup_{q,m}\inf_{\delta,h}\Phi_{\bar u}(\rho,q,m,r,\delta,h).
\end{align}Hence,
\begin{align}
    -\varepsilon_t= \alpha^{-1}\partial_{\beta_t}\bar p|_{\beta_t=\beta}=\E\frac{\E_\xi u_y(\nu^*)e^{\beta u_y(\nu^*)}}{\E_\xi e^{\beta u_y(\nu^*)}}\,,\quad \nu^*=\lambda y m^*+\sqrt{q^*}v+\sqrt{\rho^*-q^*}\xi
\end{align}with $\rho^*,q^*,m^*$ and $r^*,\delta^*,h^*$ solving the variational problem at $\beta_t=\beta$.
\end{proof}

For the proof of Corollary \ref{cor:gen_error} we first need to introduce an auxiliary model. Fix a positive number $\epsilon>0$ and consider a model with additional training data
\begin{align}
    -\tilde H_N(\sigma)=\sum_{\mu=1}^{M}u_{y_\mu} \Big(\frac{g_\mu\cdot\sigma}{\sqrt{N}}\Big)
    +\gamma \sum_{\mu=M+1}^{M+\lfloor\epsilon M\rfloor} \Big[ tW_{y_\mu}\Big(\frac{g_\mu\cdot\sigma}{\sqrt{N}}\Big)+u_{y_\mu} \Big(\frac{g_\mu\cdot\sigma}{\sqrt{N}}\Big)\Big]-\sum_{i=1}^N\phi(\sigma_i)-\frac{\kappa\|\sigma\|^2}{2}\,.
\end{align}
$\gamma\in[0,\eta)$ here is chosen as in the statement of Corollary \ref{cor:gen_error} and $t\in[-\eta,\eta]$. The above can also be recast as 
\begin{align}
    -\tilde H_N(\sigma)=\sum_{\mu=1}^{M+\lfloor\epsilon M\rfloor}\tilde u^\mu_{y_\mu} \Big(\frac{g_\mu\cdot\sigma}{\sqrt{N}}\Big)
    -\sum_{i=1}^N\phi(\sigma_i)-\frac{\kappa\|\sigma\|^2}{2}\,,\quad \tilde u^\mu_{y}=\begin{cases}
        u_y\quad \text{if }\mu\leq M\\
        \gamma(u_{y}+tW_y)\quad \text{otherwise}.
    \end{cases}
\end{align}
From the above we can define also a quenched pressure per particle
\begin{align*}
    \tilde p_N(t,\gamma):=\frac{1}{N}\E\log\int d^N\sigma e^{-\beta \tilde H_N(\sigma)}\,.
\end{align*}
From the above definitions, it is not difficult to verify that
\begin{align}
    \varepsilon_{g,N}^W:=\E\big\langle
    W_{y'}\Big(\frac{g'\cdot\sigma}{\sqrt{N}} \Big)
    \big\rangle=\lim_{\gamma\to 0^+}
    \frac{1}{\beta\alpha\epsilon \gamma} \partial_t \tilde p_N(t=0,\gamma)=\frac{1}{\beta\alpha\epsilon}\partial_t\partial_\gamma \tilde p_N(t=0,\gamma=0)\,.
\end{align}
We shall thus need to adapt Theorem \ref{thm:main} to $\tilde p_N$. 
\begin{lemma}
    Under the hypotheses of Corollary \ref{cor:gen_error} one has
    \begin{align}
        \tilde p_N\xrightarrow[]{N\to\infty}\tilde p:=\sup_\rho\inf_r\sup_{q,m}\inf_{\delta,h}\Phi_{\tilde u,\gamma}(\rho,q,m,r,\delta,h)\,
    \end{align}with $\Phi_{\tilde u,\gamma}$ defined in \eqref{eq:phi_aux_gen_error} and $\alpha'=\epsilon\alpha$.
\end{lemma}
\begin{proof}
    The proof is identical to that of Theorem \ref{thm:main}. For uniform convergence, it suffices to check that the remainder is uniformly bounded in $t\in[0,1]$.
\end{proof}
We are now ready for the proof.

\begin{proof}[Proof of Corollary \ref{cor:gen_error}]
For reasons completely analogous to those discussed for $\bar p_N$, defined in \eqref{eq:bar_p_N}, $\partial_\gamma^2 \tilde p_N$ is bounded by a constant. Furthermore, $\tilde p_N$ is convex in $t$. Therefore, for any $s>0$
\begin{align}
    \frac{\tilde p_N(t,\gamma)-\tilde p_N(t-s,\gamma)}{\gamma s}\leq \frac{\partial_t \tilde p_N(t,\gamma)}{\gamma}\leq 
    \frac{\tilde p_N(t+s,\gamma)-\tilde p_N(t,\gamma)}{\gamma s}\,.
\end{align}
We can now expand $\tilde p_N$ up to second order, obtaining
\begin{align}
    \tilde p_N(t,\gamma)=\tilde p_N(t,0)+\gamma \partial_\gamma \tilde p_N(t,0)+O(\gamma^2)\,,
\end{align}thanks again to the fact that $\partial_\gamma^2 \tilde p_N\leq \bar C$ uniformly in $t\in[-\eta,\eta]$, and hence its derivative converges. Observe also that $\tilde p_N(t,0)=\tilde p_N(0,0)$ is actually independent on $t$, leading to
\begin{align}
    -\bar C\frac{\gamma}{s}+\frac{\partial_\gamma \tilde p_N(t,0)-\partial_\gamma \tilde p_N(t-s,0)}{ s}\leq \frac{\partial_t \tilde p_N(t,\gamma)}{\gamma}\leq 
    \frac{\partial_\gamma \tilde p_N(t+s,0)-\partial_\gamma \tilde p_N(t,0)}{s}+\bar C\frac{\gamma}{s}\,.
\end{align}
After sending $\gamma\to 0$ we get
\begin{align}
    \frac{\partial_\gamma \tilde p_N(t,0)-\partial_\gamma \tilde p_N(t-s,0)}{ s}\leq \partial_\gamma\partial_t \tilde p_N(t,0)\leq 
    \frac{\partial_\gamma \tilde p_N(t+s,0)-\partial_\gamma \tilde p_N(t,0)}{s}\,.
\end{align}Now, recall that from the lower bound
\begin{align}
    \liminf_{N\to\infty} \partial_\gamma\partial_t \tilde p_N(t,0)
    \geq s^{-1}\big(
    \partial_\gamma \tilde p(t,0)-\partial_\gamma \tilde p(t-s,0)
    \big)
\end{align}where we have again exchanged limit and derivative. Analogously
\begin{align}
    \limsup_{N\to\infty} \partial_\gamma\partial_t \tilde p_N(t,0)
    \leq s^{-1}\big(
    \partial_\gamma \tilde p(t+s,0)-\partial_\gamma \tilde p(t,0)
    \big)\,.
\end{align}
By letting $s\to0^+$ we then get two matching bounds, leading to
\begin{align}
    \lim_{N\to\infty}\lim_{\gamma\to0^+} \gamma^{-1} \partial_t\tilde p_N(0,\gamma) =\partial_\gamma\partial_t\tilde p\,.
\end{align}
The computations of the derivatives yield
\begin{align}
    \partial_t\tilde p(0,\gamma) =\beta \epsilon\alpha \gamma\E\frac{\E_\xi W_{y'}(\nu')e^{\beta\gamma u_{y'}(\nu')}}{
    \E_\xi e^{\beta\gamma u_{y'}(\nu')}
    }\,, \quad \partial_\gamma\partial_t\tilde p(0,0)=\beta\alpha\epsilon\E\E_\xi W_{y'}(\nu')\,,
\end{align}
which in turn finally entails
\begin{align}
    \lim_{N\to\infty}\varepsilon_{g,N}^W=: \varepsilon_g^W=\E\E_\xi W_{y'}(\lambda y' m^*+\sqrt{q^*}v'+\sqrt{\rho^*-q^*}\xi).
\end{align}

\end{proof}

\section{Numerical checks in notable learning settings}\label{sec:numerics}

As already noted in Section~\ref{sec:definitions}, the two bounds in Theorem~\ref{thm:main} differ only in the order of the optimizations over $\rho$ and $r$, namely in the exchange of $\sup_\rho$ and $\inf_r$. The optimization over these two variables is therefore the only possible source of a discrepancy between the two bounds. Since we have not yet established their equality under sufficiently general assumptions, we provide numerical evidence that the two optimization orders can be exchanged in several learning regimes of interest in optimization and machine learning.

To this end, we optimize the variational potential over $q$, $m$, $\delta$, and $h$ at fixed values of $\rho$ and $r$, thus obtaining the reduced variational potential \eqref{eq:reduced_potential} for different values of $\alpha$, $\beta$, and $\kappa$, as well as for different regularization schemes, namely choices of $\phi$. We chose these parameters to cover the main learning regimes previously investigated using statistical-physics methods~\cite{mignacco20a_GMM_Gordon,Loureiro2021GMM}. We then plot the resulting surface in the $(\rho,r)$ plane and examine its critical-point structure. In particular, we investigate the existence of a saddle point, thereby providing numerical evidence that the $\sup_\rho$ and $\inf_r$ operations can be interchanged.

\begin{figure}[t]
    \centering

    \begin{subfigure}[t]{0.32\textwidth}
        \centering
        \includegraphics[width=\textwidth]{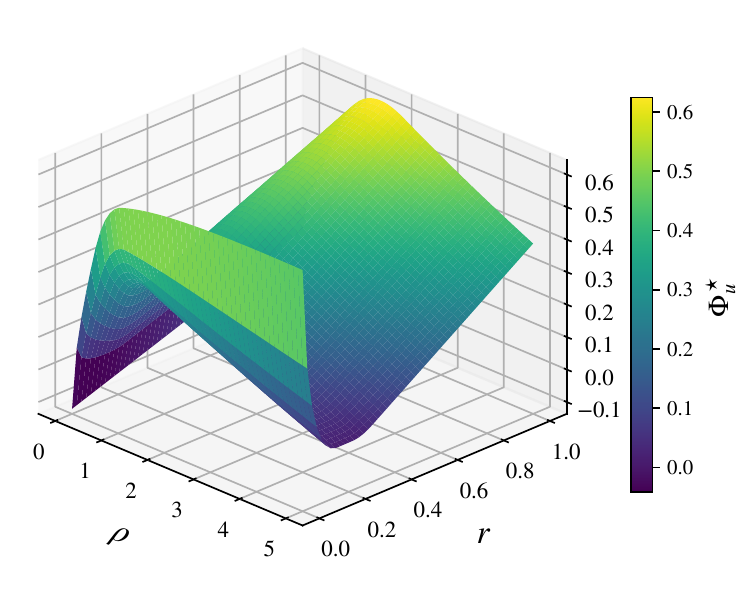}
        \caption{$\alpha = 1$, $\beta = 1$, $\kappa = 1$.}
        \label{fig:surface_1}
    \end{subfigure}
    \hfill
    \begin{subfigure}[t]{0.32\textwidth}
        \centering
        \includegraphics[width=\textwidth]{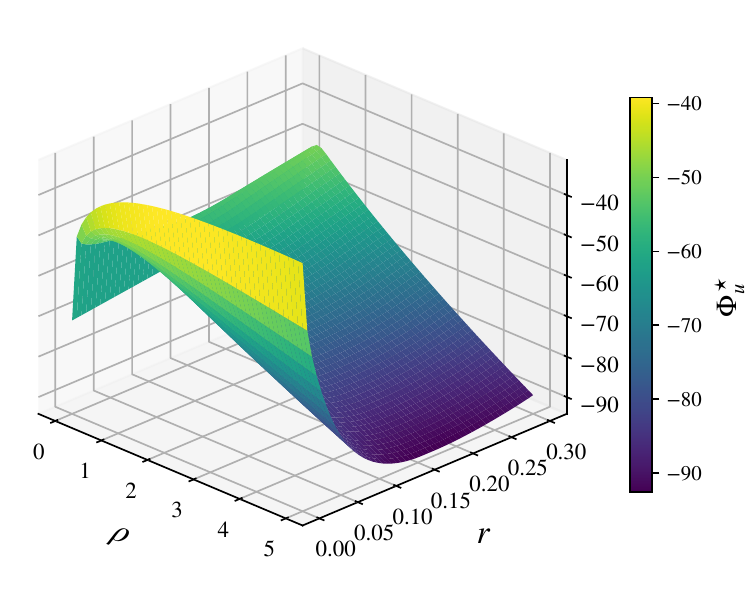}
        \caption{$\alpha = 1$, $\beta = 100$, $\kappa = 1$.}
        \label{fig:surface_2}
    \end{subfigure}
    \hfill
    \begin{subfigure}[t]{0.32\textwidth}
        \centering
        \includegraphics[width=\textwidth]{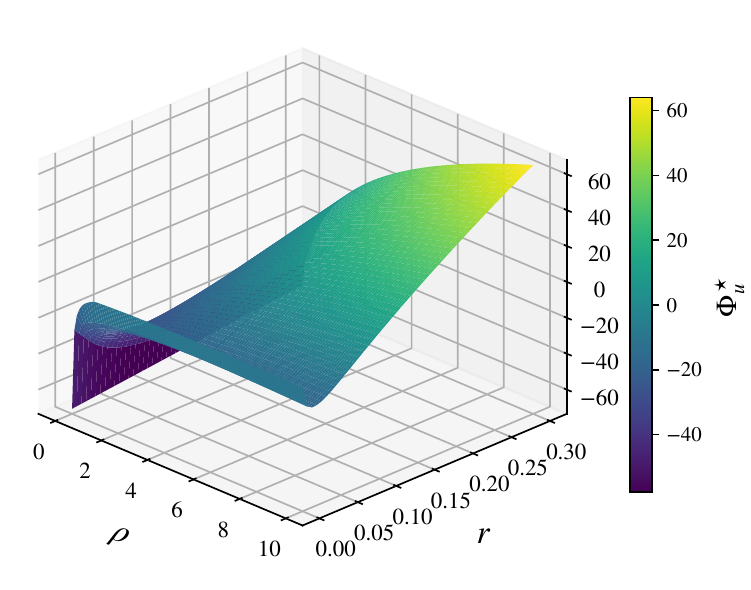}
        \caption{$\alpha = 1$, $\beta = 100$, $\kappa = 0.01$.}
        \label{fig:surface_3}
    \end{subfigure}

    \vspace{0.4cm}

    \begin{subfigure}[t]{0.32\textwidth}
        \centering
        \includegraphics[width=\textwidth]{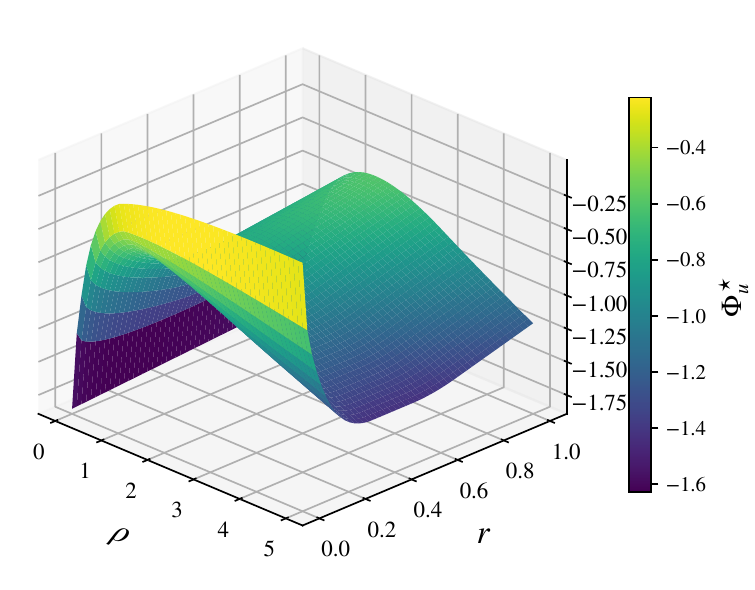}
        \caption{$\alpha = 3$, $\beta = 1$, $\kappa = 1$.}
        \label{fig:surface_4}
    \end{subfigure}
    \hfill
    \begin{subfigure}[t]{0.32\textwidth}
        \centering
        \includegraphics[width=\textwidth]{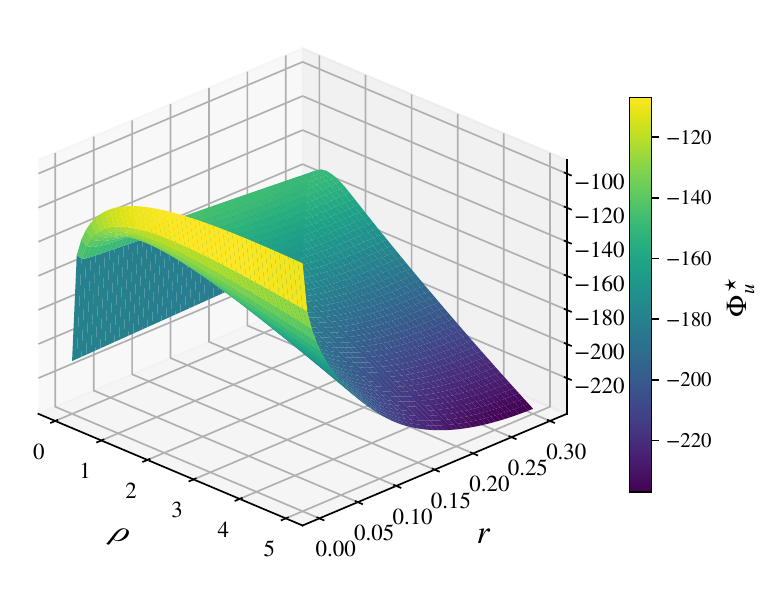}
        \caption{$\alpha = 3$, $\beta = 100$, $\kappa = 1$.}
        \label{fig:surface_5}
    \end{subfigure}
    \hfill
    \begin{subfigure}[t]{0.32\textwidth}
        \centering
        \includegraphics[width=\textwidth]{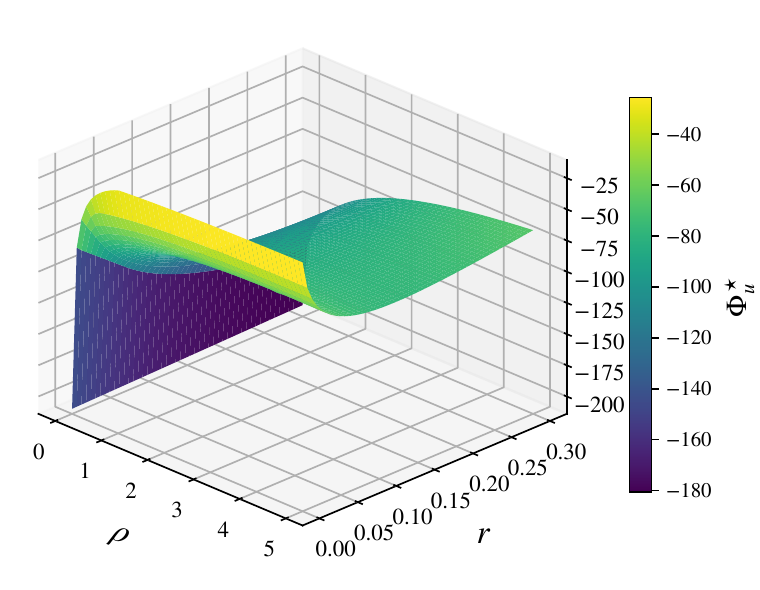}
        \caption{$\alpha = 3$, $\beta = 100$, $\kappa = 0.01$.}
        \label{fig:surface_6}
    \end{subfigure}

    \caption{
        Numerical surfaces of the reduced variational potential $\Phi_u^\star(\rho,r)$ for different choices of the model parameters, with logistic utility $u_y(s)=-\log\left(1+e^{-ys}\right)$, no additional separable regularization, $\phi(\sigma_i)=0$, and Gaussian centroid distribution $\theta\sim\mathcal{N}(0,I_N)$. $\lambda=1$ in all plots. The inner minimization over $(\delta,h)$ and the outer maximization over $(q,m)$ are both performed using the \texttt{L-BFGS-B} method from \texttt{scipy.optimize}. In the degenerate case $r=0$, where the constraint $h^2\leq\alpha r$ forces $h=0$, the remaining scalar minimization over $\delta$ is performed using the \texttt{bounded} method from \texttt{scipy.optimize}. For both optimization routines, the tolerance is set to $10^{-9}$ and the maximum number of iterations to $100$. The surface is evaluated on a $50\times 50$ grid in the $(\rho,r)$ plane.
    }
    \label{fig:six_surfaces}
\end{figure}

Figure~\ref{fig:six_surfaces} shows $\Phi_u^\star$ for the logistic, or cross-entropy, loss with $L_2$ regularization. In this setting, the model is known to exhibit a learning transition at the threshold $\alpha=2$, above which the model cannot perfectly fit the training set~\cite{mignacco20a_GMM_Gordon}. Accordingly, the first and second rows display $\Phi_u^\star$ below and above this threshold, for $\alpha=1$ and $\alpha=3$, respectively. The first two columns compare a finite-temperature regime, with $\beta=1$, and a low-temperature regime, with $\beta=100$. The latter approximates the zero-temperature limit, which captures the asymptotic behavior of optimization algorithms. Indeed, as $\beta\to\infty$, the Gibbs measure concentrates on the ground-state configurations of the Hamiltonian. For the convex optimization problem considered here (logistic regression with $L_2$ regularization) these ground states coincide with the global minimizers to which gradient-based optimization methods converge.

For weak regularization, the generalization error of this learning model is also known to exhibit a peak near the interpolation threshold, indicating overfitting~\cite{mignacco20a_GMM_Gordon}. This peak gradually disappears as the regularization strength $\kappa$ increases. The last two columns of Figure~\ref{fig:six_surfaces} therefore compare a relatively strong, nearly optimal regularization regime, with $\kappa=1$, and a weakly regularized regime, with $\kappa=0.01$, in which the peak begins to emerge.

Across all the parameter regimes considered, the surface $\Phi_u^\star$ exhibits a saddle structure in the variables $\rho$ and $r$. This supports that the two optimizations $\sup_\rho$ and $\inf_r$ can be interchanged and, consequently, that the lower and upper bounds in Theorem~\ref{thm:main} coincide in these regimes.

\begin{figure}[t]
    \centering

    \begin{subfigure}[t]{0.48\textwidth}
        \centering
        \includegraphics[width=\textwidth]{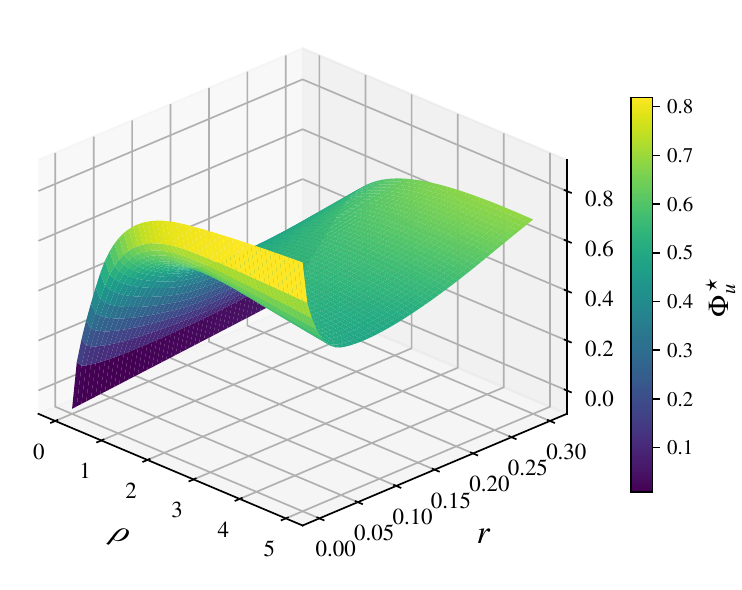}
        \caption{$\alpha=1$, $\beta=1$, $\kappa = 1e-4$.}
        \label{fig:four_surface_1}
    \end{subfigure}
    \hfill
    \begin{subfigure}[t]{0.48\textwidth}
        \centering
        \includegraphics[width=\textwidth]{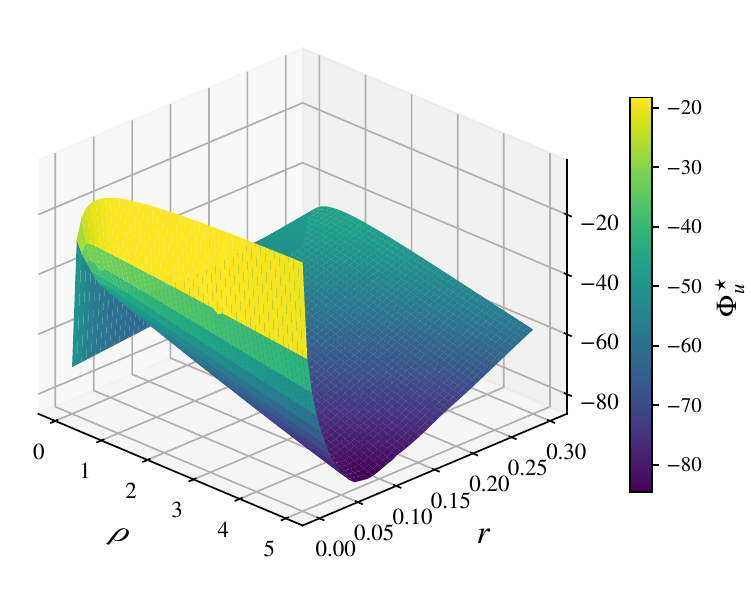}
        \caption{$\alpha=1$, $\beta=100$, $\kappa = 1e-4$.}
        \label{fig:four_surface_2}
    \end{subfigure}

    \vspace{0.5cm}

    \begin{subfigure}[t]{0.48\textwidth}
        \centering
        \includegraphics[width=\textwidth]{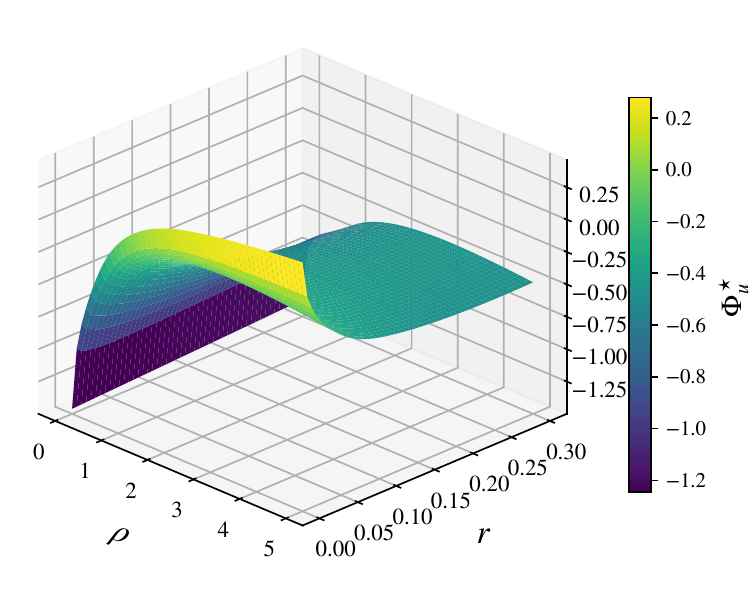}
        \caption{$\alpha=3$, $\beta=1$, $\kappa = 1e-4$.}
        \label{fig:four_surface_3}
    \end{subfigure}
    \hfill
    \begin{subfigure}[t]{0.48\textwidth}
        \centering
        \includegraphics[width=\textwidth]{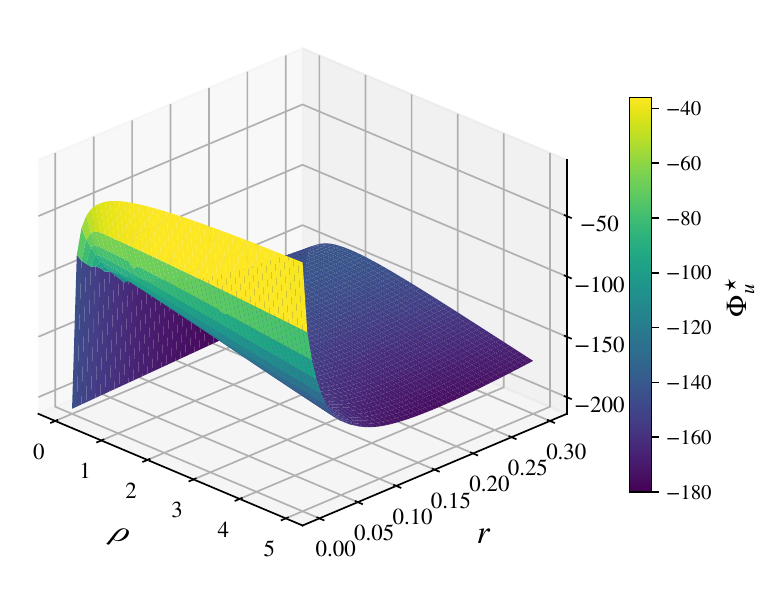}
        \caption{$\alpha=3$, $\beta=100$, $\kappa = 1e-4$.}
        \label{fig:four_surface_4}
    \end{subfigure}

    \caption{
    Numerical surfaces of the reduced variational potential $\Phi_u^\star(\rho,r)$ for different choices of the model parameters, with logistic utility $u_y(s)=-\log\left(1+e^{-ys}\right)$, no additional separable regularization, $\phi(\sigma_i)=\log\cosh \sigma_i $, and Gaussian centroid distribution $\theta\sim\mathcal{N}(0,I_N)$. $\lambda=1$ in all plots. The parameters of the optimization in $\delta, h, q$ and $m$ are the same as in Figure \ref{fig:six_surfaces}.
    }
    \label{fig:four_surfaces}
\end{figure}

Figure~\ref{fig:four_surfaces} shows that the same saddle-point structure persists under smooth $L_1$ regularization, across different values of $\alpha$ and at different temperatures. This provides further numerical evidence that the lower and upper bounds coincide in these regimes.

\section{Conclusions and perspectives}

In this work we developed a variational approach to a finite-temperature continuous-spin perceptron trained on Gaussian-mixture data. Under log-concavity assumptions, adaptive interpolation yields lower and upper bounds for the limiting quenched pressure that differ only in the order of two outer scalar optimizations. The remaining extrema are controlled by the concave--convex structure of a single potential, whose stationarity conditions recover the corresponding fixed-point equations. The same framework also gives access to the zero-temperature limit and, under additional regularity and uniqueness assumptions, to training and generalization observables. Beyond the specific model considered here, the main methodological contribution is therefore to organize the thermodynamics of the perceptron within a global variational structure.

A natural question is whether the present strategy can be extended to multilayer architectures. The main obstruction is that the composition of several layers generally destroys the log-concavity properties used here to prove concentration of the relevant overlaps and to control the interpolation remainder. The single-layer analysis nevertheless shows that the loss of positivity does not by itself rule out a variational description: in the present setting, the convex--concave structure of the potential provides an alternative organizing principle. A multilayer extension would therefore require either a different mechanism for identifying and concentrating the relevant overlaps, or a regime in which enough of this convex structure survives for the order parameters to be incorporated into a single variational potential. Whether this can occur is likely to depend on the architecture, scaling, and phase of the model.

Low-rank multilayer perceptrons, in which the weight matrices have a rank assumed to be fixed and not scaling with the system's size, provide one possible setting in which such a reduction may remain tractable. Preliminary, non-rigorous replica calculations in the proportional limit, where the network layers and the number of training data diverge proportionally (see \cite{camilli-tiepova-barbier2023fundamental,CamilliTieplovaBergaminBarbier2025,rotondo2022statmechDNN,gerace2020GET_ICML,Gerace2024Universality,cui2023optimal}), suggest that the problem may close in terms of a finite collection of overlaps. This perspective is also consistent with numerical observations on trained neural networks: the singular-value spectrum of the first-layer weight matrix exhibits a small number of isolated singular values above an approximately Gaussian-like bulk \cite{decelle2017spectral,martin2021implicit}. Removing the bulk appears to leave classification performance nearly unchanged. These observations do not establish a general low-rank principle, but they suggest that an effective finite-dimensional structure may capture a substantial part of the learned representation and motivate a systematic study of this regime.

\appendix

\section{Lemmata}
\subsection{Existence criterion for ODEs}

\begin{lemma}\label{lem:global_existence}
    Consider a system of ODEs $\dot{\mathbf{x}}=f(t;\mathbf{x})$ with $\mathbf{x}\in\mathbb{R}^n$ with $t\in[0,1]$ and $\mathbf{x}(0)=\mathbf{x}_0$, with $x_{20}\geq0$. Suppose the velocity field is continuous in $t$ and $\mathbf{x}$ and that
    \begin{align}\label{eq:inequality_velocity_field}
        0\leq f_1(t;\mathbf{x})^2\leq f_2(t;\mathbf{x})
    \end{align}
    Then, if
    \begin{align}
        \|f(t,\mathbf{x})\|\leq K\big(1+x_1^2+\sum_{i\geq2}^n |x_i|\big)
    \end{align}the system admits a global solution in $t\in [0,1]$.
\end{lemma}
\begin{proof}
    Define the function
    \begin{align}
        W(t)=K(1+x_1^2(t)+x_2(t)+ \sum_{i=3}^n|x_i(t)|)\,.
    \end{align}The function $x_2(t)$ is always nonnegative since its initial condition and the related component of the velocity field are both nonnegative. Denote the upper Dini derivative as
    \begin{align}
        D^+g(t):=\limsup_{h\downarrow 0}\frac{g(t+h)-g(t)}{h}\,.
    \end{align} 
    Then $D^+|x_i(t)|\leq|\dot x_i(t)|=|f_i|$. Moreover
    \begin{align}
        D^+ W\leq K(2x_1 \dot x_1+\dot x_2+\sum_{i=3}^n|\dot x_i|)=K(2x_1 f_1+f_2+\sum_{i=3}^n|f_i|)\leq K(x_1^2+ f_1^2+f_2+\sum_{i=3}^n|f_i|).
    \end{align}
    By definition of $W$ one has
    \begin{align}
        K^{-1}W\geq x_1^2 \,,\quad W\geq |f_i|\;\forall 3\leq i\leq n,\quad W\geq f_2\geq f_1^2\,.
    \end{align}Hence
    \begin{align}
        D^+ W(t)\leq W(1+nK)
    \end{align}which by Gr\"onwall inequality yields $W(t)\leq e^{(1+nK)t}W(0)$. 
    Therefore, $W(t)$ remains bounded in finite time, which is enough to guarantee existence of a global solution.
\end{proof}
\begin{remark}
    The above Lemma generalizes to the presence of any number of constraints of the type \eqref{eq:inequality_velocity_field}.
\end{remark}

\subsection{Moment generating function estimates}
Let  $Y$ be  a random variable with a density of the form
\be
p_Y(y)\,=\,\dfrac{1}{c_{\kappa}} e^{-\phi(y)-\frac{\kappa}{2}y^2} \,,\quad c_{\kappa}=\int_\mathbb{R} e^{-\phi(y)-\frac{\kappa}{2}y^2} \,dy
\ee
where $\kappa>0$, and $\phi$ is a $L$-Lipschitz  function.
Notice that one can assume without loss that $\phi(0)=0$. 

\begin{lemma}\label{lem:momentphi}
For  any $z \in \mathbb{R}$ one has 
\be
 \frac{1}{c_{\kappa}}\sqrt{\frac{2\pi}{\kappa+L}} e^{-\frac{L}{2}+\frac{z^2}{2(\kappa+L)}}\,\leq\,
\mathbb{E}_Y e^{zY}\,\leq\, 
\frac{2}{c_{\kappa}}\sqrt{\frac{2\pi}{\kappa}}\,e^{\frac{1}{2\kappa}\left(L^2+z^2(1+\frac{L^2}{\kappa})\right)}
\ee

\end{lemma}
\begin{proof}
Since $\phi$ is Lipschitz and $\phi(0)=0$ one has $\phi(y)\geq -L|y|$ for every $y\in\mathbb{R}$, hence

\be\label{eq:phiestimate}
\begin{aligned}
&\mathbb{E}_Y\exp(z Y)\leq  \frac{1}{c_{\kappa}}\sqrt{\frac{2\pi}{\kappa}}\,\mathbb{E}_{g\sim\mathcal{N}(0,\kappa^{-1})}\exp\Big(z g+ L|g|\Big)\\
&\leq \frac{1}{c_{\kappa}}\sqrt{\frac{2\pi}{\kappa}}\,\Big[\mathbb{E}_{g\sim\mathcal{N}(0,\kappa^{-1})}\exp((z+L) g)+\mathbb{E}_{g\sim\mathcal{N}(0,\kappa^{-1})}\exp((z-L) g)\Big]\\
&
=\frac{1}{c_{\kappa}}\sqrt{\frac{2\pi}{\kappa}}\,\Big( e^{\frac{(z+L)^2}{2\kappa}}+e^{\frac{(z-L)^2}{2\kappa}}\Big)\\
&= \frac{2}{c_{\kappa}}\sqrt{\frac{2\pi}{\kappa}}\,\exp\left(\frac{1}{2\kappa}\left(L^2+z^2\right)\right)\cosh(\frac{zL}{\kappa})\\
&\leq \frac{2}{c_{\kappa}}\sqrt{\frac{2\pi}{\kappa}}\,\exp\left(\frac{1}{2\kappa}\left(L^2+z^2\right)\right)\exp\left(\frac{L^2z^2}{2\kappa^2}\right)
\end{aligned}
\ee

On the other hand $-\phi(y)\geq -L|y|$ and then

\be\label{eq:phiestimate}
\begin{aligned}
&\mathbb{E}_Y\exp(z Y)\geq \frac{1}{c_{\kappa}}\sqrt{\frac{2\pi}{\kappa}}\,\mathbb{E}_{g\sim\mathcal{N}(0,\kappa^{-1})}\exp(z g- L|g|)\\
&\geq \frac{1}{c_{\kappa}}\sqrt{\frac{2\pi}{\kappa}}\,\mathbb{E}_{g\sim\mathcal{N}(0,\kappa^{-1})}\exp(zg-\frac{L}{2}(g^2+1))\\
&
=\frac{1}{c_{\kappa}}\sqrt{\frac{2\pi}{\kappa+L}} e^{-\frac{L}{2}+\frac{z^2}{2(\kappa+L)}}
\end{aligned}
\ee
\end{proof}

\subsection{Random matrices}

Here some basic facts  on random matrices. For the proof see for example \cite{Tala_vol1,Vershynin2018}.
If  $A$ is an $M \times N$ matrix we denote by

\begin{itemize}
\item Frobenius norm:
\[
\|A\|_F = \left( \sum_{k=1}^M \sum_{i=1}^N A_{k,i}^2 \right)^{1/2}.
\]

\item Operator norm:
\[
\|A\|_{\mathrm{op}} = \sup_{\|y\| \le 1} \|Ay\|.
\]
\end{itemize}

Recall also that
\[
\|A\|_{\mathrm{op}} \le \|A\|_F.
\]

Consider $\underline{Z}=\frac{1}{\sqrt{N}}Z$ where $Z=(Z_{\mu,i})^{\mu\leq M}_{i\leq N}$ and $Z_{\mu,i}\iid \mathcal{N}(0,1)$.

\begin{lemma} For every $\delta > 0$,
\be\label{eq:conceG}
\mathbb{P}\left(
\|\underline{Z}\|_{\mathrm{op}} 
\ge 1 + \sqrt{\alpha}  + \delta
\right)
\le 2 \exp\left(-\frac{N \delta^2}{2}\right),
\ee
where $\alpha = \frac{M}{N}$. Therefore, for every $p \in \mathbb{N}$,
\be\label{eq:pmomentG}
\mathbb{E}\big[ \|\underline{Z}\|_{\mathrm{op}}^p \big] < \infty.
\ee
\end{lemma}

\subsection{Concentration inequalities}

\begin{lemma}[Tensorization]
Let $X_1\perp X_2$ be independent. Suppose that, for every fixed $x_2$ and every smooth $h$,
\[
\mathrm{Var}_{X_1}\big(h(X_1,x_2)\big) \le C_1\,\mathbb E_{X_1}\big[\|D_1h(X_1,x_2)\|^2\big],
\]
and similarly for block $2$ with constant $C_2$. Then for every smooth $F$
\[
\mathrm{Var}\big(F(X_1,X_2)\big) \le C_1\,\mathbb E\|D_1F\|^2 + C_2\,\mathbb E\|D_2F\|^2.
\]
\end{lemma}

\begin{lemma}[Rademacher-Poincarè inequality]\label{lem:radpoinc}
Let $X=(X_1,\ldots,X_n)$ be i.i.d.\ Rademacher random variables  and let $f:\{-1,1\}^n\to\mathbb R$. For $x\in\{-1,1\}^n$ and $i\in\{1,\ldots,n\}$, let $x^{+i}$ (resp.\ $x^{-i}$) denote $x$ with its $i$-th coordinate set to $+1$ (resp.\ $-1$), and define the $i$-th discrete derivative
\be\label{def:differenceoperator}
D_if(x) := \frac{f(x^{+i})-f(x^{-i})}{2}.
\ee
Then
\[
\mathrm{Var}\big(f(X)\big) \;\le\; \sum_{i=1}^n \mathbb E\big[(D_if(X))^2\big].
\]
\end{lemma}

\begin{proof}
Since $X_1,\ldots,X_n$ are mutually independent, then for each coordinate (with the other coordinates frozen), i.e.\ for every fixed $x_{-i}$ and every $h$,
\[
\mathrm{Var}_{X_i}\big(h(X_i,x_{-i})\big) \le \mathbb E_{X_i}\big[(D_ih(X_i,x_{-i}))^2\big],
\]
which is exactly the Poincaré hypothesis for block $i$ with constant $C_i=1$ in the tensorization theorem proved previously. Applying that theorem with $C_1=\cdots=C_n=1$ gives directly
\[
\mathrm{Var}(f(X)) \le \sum_{i=1}^n \mathbb E\big[(D_if(X))^2\big].
\]
\end{proof}

\section{Consequences of Brascamp-Lieb and Cramer-Rao inequalities}\label{app:Brascamp-Lieb_consequences}
We start recalling the main definitions of the interpolating Hamiltonian using a more compact notation. Fix the values of a collection of  real parameters $t_y,t_{\theta},t_m\in \R$ and  $t_Z,t_z,t_v,t_{\xi}\geq 0$, 
let 
$G=(g_{\mu})_{\mu\leq M}=(g_{\mu,i})_{\mu\leq M}^{i\leq N}$ with \begin{align}\label{eq:GMM2}
    g_\mu=\frac{1}{\sqrt{N}}t_y y_\mu \theta+  t_Z Z_\mu\
\end{align} where  $ y=(y_{\mu})_{\mu\leq M}$,  $Z=(Z_{\mu})_{\mu\leq M}$  and   $\theta=(\theta_i)_{i\leq N}$ are defined  in \eqref{eq:GMM}. Consider also two families of independent standard Gaussian vectors 

\be\vz=(z_{i})_{i\leq N},\quad  
\uv=(v_{\mu})_{\mu\leq M}
\ee

From now on  we will denote  rescaled variables with an underline, as example
$\underline{G}=N^{-1/2}G, \underline{\vz}=N^{-1/2}\vz$ and $\underline{\uv}=N^{-1/2}\uv$.

Given a realization of $\mathcal{D}=(y,Z,z,v,\theta)$ the  (random) Hamiltonian we will consider is defined  on the configurations space $\R^N\times\R^M\ni x=(\vs,\xi)=((\sigma_i)_{i\leq N},(\xi_{\mu})_{\mu\leq M})$ 
as

\be\label{eq: Hamiltonian2}
-H (x) = \sum_{\mu \leq M} 
u\left( S_{\mu}(x)\right) - \sum_{i\leq N} \phi(\sigma_i) -\frac{\kappa}{2}\|\vs\|^2 - \frac{1}{2}\|\uxi\|^2 -\frac{M}{2}\log(2\pi)+ (t_z\vz+t_{\theta}\theta) \cdot \vs 
\ee
 where $\kappa>0$  and 
 
\be
S_{\mu}(x)= \frac{1}{\sqrt{N}}{g}_{\mu}\cdot\vs + y_\mu t_m+t_v v_{\mu}+t_{\xi}\xi_{\mu}\,,\quad \mu\leq M
\ee
\begin{remark}
In order to lighten the notation we dropped the  dependency of the function $u$ from the realization of $y_{\mu}$, namely in this section $u\equiv u_{y_\mu}$. It's clear that given $t\in[0,1]$ there is  a choice of  $\kappa$ and  all the $t_{\bullet}$ such that  \eqref{eq: Hamiltonian2} equals the  interpolating Hamiltonian $H^t_N$ in  \eqref{eq:interpolating_H}. 

\be\label{eq:conversion}
\begin{aligned}
&\kappa\to\kappa+\delta(t)\quad t_{\xi}=\sqrt{\rho(t)-q(t)}\quad t_y= (1-t)\lambda\quad t_m=\lambda m(t)\\
& t_\theta=\lambda h(t) \quad t_Z=\sqrt{1-t}\quad t_z=\sqrt{r(t)}\quad t_v=\sqrt{q(t)}.\\
\end{aligned}
\ee
\end{remark}

We will work with a more general Hamiltonian than \eqref{eq: Hamiltonian2}, defined as follows. Consider two copies $\bsigma=(\vs^1,\vs^2)$, $\bxi=(\uxi^1,\uxi^2)$
and set  $x^1=(\vs^1,\uxi^1)$  $x^2=(\vs^2,\uxi^2)$, and  $\bx =(x^1,x^2)\in \Big(\mathbb{R}^{N}\times \mathbb{R}^{M}\Big)^2$. Consider an Hamiltonian function coupling the two $x^1,x^2$ copies trough the overlap between $\sigma^1,\sigma^2$: 

\be\label{interaction}
\mathcal{H}_a(\bx)=H(x^1)+H(x^2)+a \,(\sigma^1\cdot\sigma^2)
\ee
where $a\in\R$ is a parameter and
$H(x)$ is defined in \eqref{eq: Hamiltonian2}. 
Given a realization of  $G,\uv,\vz$ we denote by $\langle\,\,\rangle_{(a)}$ the average w.r.t. the induced  (random) Gibbs measure 
\be\label{eq:randomGibbsa}
\nu_a(\bx)d\bx=\frac{1}{{Z}_a}\,e^{-\mathcal{H}_a(\bx)}\,d \bx\ee

\begin{remark}

Notice that by definition 
with the choice \eqref{eq:conversion} one has $\langle \cdot\rangle_{(a)}|_{a=0}\equiv\langle \cdot \rangle_{t}$ where $\langle\cdot  \rangle_{t}$ is defined in \eqref{eq:randomGibbsa} and represents the average w.r.t. $\nu_t^{\otimes2}$, namely the (product)measure induced by the interpolating Hamiltonian.
Clearly this correspondence holds  for $\beta=1$, however the general case cab easily recovered. 
\end{remark}

\begin{lemma}[Log-concavity] \label{LEM:logconcave}

If $\kappa-|a|> 0$ then the function  
\be\label{eq:phia}
\bx\mapsto\mathcal{H}_a(\bx)
\ee
is strictly convex  and then $\nu_a$ is log-concave measure.
Moreover \\
for any $\boldsymbol{x} \in \Big(\mathbb{R}^{N}\times \mathbb{R}^{M}\Big)^2$ one has 
\be\label{boundHessian}
\begin{aligned}
&\bx^T \mathrm{Hess}_{\mathcal{H}_a}\bx \geq (\kappa-|a|) \|\boldsymbol{\sigma}\|^2+\|\boldsymbol{\xi}\|^2\\
&\bx^T \mathrm{Hess}_{\mathcal{H}_a}\bx\leq \left(2C\|\underline{G}\|^2_{op}+\kappa+C+|a|\right)\|\boldsymbol{\sigma}\|^2+(2Ct^2_{\xi}+1)\|\boldsymbol{\xi}\|^2
\end{aligned}
\ee
\end{lemma}

\begin{proof}
By direct computation 
\be
\begin{aligned}
\bx^T \mathrm{Hess}_{\mathcal{H}_a}\bx&=-\sum_{\ell=1,2} \sum_{\mu\leq M} \Big[u''(S_\mu(x^{\ell}))\Big( \frac{\vec{g}_\mu\cdot \vs^{\ell}}{\sqrt{N}}+t_{\xi}\xi_\mu^{\ell}\Big)^2\\
&+
(\vs^{\ell})^T  \mathrm{Hess}_{\phi}\vs^{\ell} +\kappa\|\vs^{\ell}\|^2 +\|\uxi^{\ell}\|^2\Big]+2a(\vs^1\cdot \vs^2)
\end{aligned}
\ee

Now  since  $u$ is concave and $\phi$ is convex  one has 
\be
\bx^T \mathrm{Hess}_{\mathcal{H}_a}\bx \geq (\kappa-|a|) \|\boldsymbol{\sigma}\|^2+\|\boldsymbol{\xi}\|^2
\ee

On the other hand   H3  implies that $u'',-\phi''\geq -C$ and then 

\be
\begin{aligned}
\bx^T \mathrm{Hess}_{\mathcal{H}_a}\bx &\leq \sum_{\ell=1,2} \sum_{\mu\leq M} \Big[C\Big( \frac{{g}_\mu\cdot \vs^{\ell}}{\sqrt{N}} + t_{\xi}\xi_\mu^{\ell}\Big)^2+(\kappa+C+|a|)\|\boldsymbol{\sigma}\|^2 +\|\boldsymbol{\xi}\|^2\\
&\leq (2C\|\underline{G}\|^2_{op}+(\kappa+C+|a|)\|\boldsymbol{\sigma}\|^2+(2Ct^2_{\xi}+1)\|\boldsymbol{\xi}\|^2
\end{aligned}
\ee
\end{proof}

We are now ready to prove the following:
\begin{prop}[Control of the variance]\label{prop:concentrationover}
Assume that the functions $u,\phi$ satisfies hypothesis H1-H4, then given $t\in[0,1]$ and the choice of  $\mathcal{P}_t$ and for any realization of  $\mathcal{D} = (y, Z, z, v,\theta)$  one has 
\be\label{eq:boundgapoverlap}
\Big(2C\|\underline{G}\|^2_{op}+\kappa+\delta(t)+C\Big)^{-1}\leq \Big\langle  Q_{11} - Q_{12}\Big\rangle_t\leq \frac{1}{\kappa+\delta(t)}
\ee
where $\underline{G}=\frac{1}{\sqrt{N}}G$. 
\end{prop}

\begin{proof}
Keeping in mind the choice \eqref{eq:conversion} one 
\be
\Big\langle  Q_{11} - Q_{12}\Big\rangle_{t}=\Big\langle  Q_{11} - Q_{12}\Big\rangle_{(0)}=\frac{1}{N}\sum_{i\leq N}  \,\mathrm{Var}_0
 (\sigma_i) 
\ee
where $\mathrm{Var}_a$ denotes the variance w.r.t. the random Gibbs measure $\nu_a$ associated to the Hamiltonian $\mathcal{H}_a$. 
We know from Lemma \ref{LEM:logconcave} that if  $\kappa-|a|>0$ then $\nu_a$ is log-concave the Hessian is bounded by \eqref{boundHessian}. Therefore  $\nu_a|_{a=0}$ is  the a product measure with log-concave marginals   with potential $H$.  Hence from Brascamb-Lieb bound one has 
\be\label{eq:BL}
\,\mathrm{Var}_0
 (\sigma_i)\leq  \Big\langle[\mathrm{Hess}_{H} ]^{-1}_{ii}\Big\rangle_{(0)}\leq \frac{1}{\kappa}\,,\quad \forall\,i\leq N
\ee
On the other hand by Cramer-Rao bound one has 
\be \label{eq:CRao}
\mathrm{Var}_0
 (\sigma_i) \geq \left[\Big\langle \mathrm{Hess}_{H} \Big\rangle_{(0)}\right]^{-1}_{ii} \geq \Big(2C\|\underline{G}\|^2_{op}+\kappa+C \Big)^{-1}\quad \forall\,i\leq N
\ee
Clearly \eqref{eq:BL} and \eqref{eq:CRao} implies \eqref{eq:boundgapoverlap}. 
\end{proof}

\section{Bounds on the order parameters}
Let $\mathcal{H}_a$ the Hamiltonian in \eqref{interaction}. Let  $f_a$ be the free energy associated to $\mathcal{H}_a$,  its expectation is 
\be\label{eq:perturbed}
\E f_a = \E \frac{1}{N} \log Z_a
= \E \frac{1}{N} \mathbb{E} \log \int \, d {\bx} \,\exp\big(-\mathcal{H}_a(\bx)\big),
\ee
where the expectation $\mathbb{E}$ is taken over all sources of quenched randomness $\mathcal{D}=(y,Z,\uv,\vz,\theta)$.

\begin{lemma}[Bounds for $f_a$]\label{lem:boundZa}
If the functions $u,\phi$ satisfies hypothesis H1,H2 and H3, and if $|a|$  is small enough then
for any realization of $\mathcal{D}$ one has 
\be\label{eq:boundZa}
-c_1(1+\|\underline{G}\|^2_{op}+t_m^2+t^2_{\xi}+t^2_v\|\underline{\uv}\|^2) \leq  f_a \leq c_2\Big( 1+\|\underline{t_z\vz+t_{\theta}\theta}\|^2\Big)
 \ee
for some $c_1,c_2>0$.
\end{lemma}
\begin{proof}
We start noticing that since $u \leq 0$ and $2|\vs^1\cdot\vs^2|\leq \|\boldsymbol{\sigma}\|^2 $ one has
\be
-\mathcal{H}_a(\bx) \leq -\sum_{\ell=1,2} \phi(\sigma^{\ell}_{i})+\frac{-\kappa+|a|}{2}\, \|\vs^{\ell}\|^2-\frac{1}{2} \|\uxi^{\ell}\|^2+(t_z z+t_{\theta}\theta) \cdot \vs^{\ell}
\ee
and then 
\be
Z_a\leq \left(\prod_{i\leq N} \int d\sigma_i \exp\left( -\phi(\sigma_i)+\frac{-\kappa+|a|}{2}\sigma^2_i+(t_zz_i+t_{\theta}\theta_i)\sigma_i\right )\right)^{2}
\ee
Therefore if $a$ is such that   $-\kappa+|a|<0$ , by Lemma \ref{lem:momentphi} one has 
\be\label{eq:upboundZ}
\begin{aligned}
Z_a &\leq \left(\frac{2\sqrt{\kappa}}{\sqrt{\tau}}\right)^{2N}  \exp\left(\frac{N}{\tau}\left[L^2+(1+\frac{L^2}{\tau})\|\underline{t_z\vz+t_{\theta}\theta}\|^2\right]\right)
\end{aligned}\ee
where $\tau=\kappa-|a|$. 
Therefore \eqref{eq:boundZa} holds with  
\be c_2= \max\left\{\log (4\kappa/\tau),\frac{L^2}{\tau}, \frac{1}{\tau}+\frac{L^2}{\tau^2}\right\}
\ee 
 On the other hand the hypothesis H1  on the function $u$ implies that 
\be
\begin{aligned}
\sum_{\mu\leq M}\,u(S_{\mu}(x))&\geq -\sum_{\mu\leq M}C\Big(1+\Big(\frac{1}{\sqrt{N}}{g}_{\mu}\cdot\vs + y_\mu t_m+t_v v_{\mu}+t_{\xi}\xi_{\mu}\Big)^2\Big) \\
&\geq-  4CN\Big(\alpha/4+\|\underline{G}\|^2_{op}\,\|\underline{\vs}\|^2+\alpha t^2_m+t^2_{\xi}\|\underline{\uxi}\|^2+t^2_{v}\|\underline{\uv}\|^2\Big)
\end{aligned}
\ee
Hence
\be
\begin{aligned}
-\mathcal{H}_a(\bx) &\geq - 8NC\left[ \alpha/4+\alpha t^2_m+ t_v^2\|\underline{v}\|^2\right] -\frac{1}{2}\sum_{\ell=1,2}(1+8Ct^2_{\xi})\|\uxi^{\ell}\|^2-M\log(2\pi)\\
&-\sum_{\ell=1,2}\left[ R\,\|\sigma^{\ell}\|^2 -\sum_{i\leq N}\phi(\sigma^{\ell}_i)-(t_zz+t_{\theta} \theta)\cdot\sigma^{\ell}\right]
\end{aligned}
 \ee
where $R=4C\|\underline{G}\|^2_{op}+\frac{\kappa+|a|}{2}$ and then 
\be\label{eq:upperbound}
\begin{aligned}
Z_a&
\geq e^{-8CN(\alpha/4+\alpha t^2_m +t_v^2\|\underline{v}\|^2)}
\Big( \int \,\frac{dy}{\sqrt{2\pi}}\, e^{ -\frac{1}{2}(1+8Ct^2_\xi)y ^2}\Big)^{2M} \Big(\prod_{i\leq N }\int\, d\sigma_i\,\exp(-R\sigma_i^2-\phi(\sigma_i)+(t_zz_i+t_{\theta}\theta_i \sigma_i)\Big)^2\\
&\geq e^{-N\Big( 2C\alpha+2C\alpha t^2_m+4Ct_v^2\|\underline{v}\|^2+\alpha \log(1+8Ct_{\xi}^2)\Big)}
\Big(\frac{2\pi}{R+L}\Big)^N \exp\left(-N(L-\frac{\|\underline{t_z z+t_{\theta}\theta}\|^2}{(R+L)})\right)\,
\end{aligned}
\ee
where the last inequality is a consequence of Lemma \ref{lem:momentphi}. Therefore
\be 
\frac{1}{N}\log Z_a\geq -2C\alpha-2C\alpha t^2_m -4Ct^2_v\|\underline{v}\|^2-\alpha\log(1+8Ct_{\xi}^2)+\log(2\pi)-L- \log(R+L)
\ee
Now in order to obtain the explicit  constant $c_1$ in \eqref{eq:boundZa} one can use the inequality $-\log(x+y)\geq -\log(x)-\frac{y}{x}$ valid for $x,y>0$ obtaining
\be 
\frac{1}{N}\log Z_a\geq (-2C\alpha+\log2\pi-L-\log L) -2C\alpha t^2_m -4Ct^2_v\|\underline{v}\|^2-8\alpha C t^2_{\xi}-\frac{R}{L}
\ee

therefore 

\be\label{eq:c1constant}
c_1=\max\Big\{2C\alpha-\log2\pi+L+\log L+\frac{\kappa+|a|}{2L}, 2C\alpha,8\alpha C, \frac{4C}{L} \Big\}
\ee

Clearly the quantity \eqref{eq:c1constant} is meaningless for $L=0$ that corresponds to the case $\phi\equiv0$. This latter case be treated directly replacing \eqref{lem:momentphi} with the exact moment generating function of a Gaussian r.v.
\end{proof}

Now we apply  the same strategy to gain control on the norm of the vector $\bx$ w.r.t. the Gibbs measure $\nu_a$ by means of its moment generating function $ s\mapsto 
\left\langle \exp\big( s \|\bx\|^2\big) \right\rangle_{(a)}$.

\begin{lemma}[Moment generating function of the norms ]\label{lem:concnorm}
There exist a number  $\varepsilon > 0$ and a costant $ K_x>0$ such that, for $|s|, |a| \le \varepsilon$,
\be\label{eq:exponentialmoment}
\left\langle \exp\big( s \|\bx\|^2\big) \right\rangle_{(a)}
\le \exp\left( N K_x\, B_N\right )
\ee
where
\be
B_N\equiv B_N(\underline{G},\underline{v}, \underline{t_zz+t_{\theta}\theta}, t_m,t_{\xi},t_v) 
= 1+\|\underline{G}\|_{op}^2 +t_m^2+t_{\xi}^2 +t^2_v\|\underline{v}\|^2 + \|\underline{t_z\vz+t_{\theta}\theta}\|^2 
\ee
The constant $ K_x\equiv K_x(\kappa,\alpha,C,L, a)$  is continuous at any $\kappa>0$.
\end{lemma}

\begin{remark}
The constant $ K_x$ in the previous statement is a key quantity for the rest of the proof so from now on $ K_x$ denotes always this constant.
\end{remark}

\begin{proof}
Notice that by definition  $\|\bx\|^2=\|\boldsymbol{\sigma}\|^2+\|\boldsymbol{\xi}\|^2=\sum_{\ell=1,2}\|\sigma^{\ell}\|^2+\|\xi^{\ell}\|^2$ and 
\be\label{eq:moment norm}
\left\langle \exp\big( s\|\bx\|^2 \big) \right\rangle_{(a)}=\dfrac{1}{Z_a}\, \int \,d\bx\, e^{-\mathcal{H}_a(\bx)+s\|\bx\|^2}
\ee

For the denominator, using Lemma  \ref{lem:boundZa} one has 

\be\label{concnorm2}
\frac1{Z_a} \;\leq\; \exp\Big(Nc_1\big(1+\|\underline G\|_{op}^2+t_m^2+t_\xi^2+t_v^2\|\underline{\uv}\|^2\big)\Big).
\ee

An upper bound for the numerator can obtained trough the same arguments used in the proof \eqref{eq:upboundZ}, but with $\kappa$ replaced by $\kappa_s=\kappa-2s$  and the $\uxi$-quadratic coefficient shifted from $1$ to $1-2s$. Therefore 
\be
-\mathcal H_a(\bx)+s\|\bx\|^2 \;\le\; \sum_{\ell=1,2}\Big[\phi(\sigma_i^\ell)+\frac{-\kappa_s+|a|}{2}\|\vs^\ell\|^2-\frac{1-2s}{2}\|\uxi^\ell\|^2+(t_zz+t_\theta\theta)\cdot\vs^\ell\Big].
\ee
Provided $1-2s>0$, integrating out $\uxi^\ell\in\R^M$ for each of the two copies gives a factor $(1-2s)^{-M}$, and provided $\tau_s:=\kappa_s-|a|=\kappa-2s-|a|>0$, Lemma \ref{lem:momentphi} applies to the $\vs$-integral  as before, yielding
\be\label{concnorm1}
\begin{aligned}
\int\,d\bx\, e^{-\mathcal{H}_a(\bx)+s\|\bx\|^2} &\leq 
 (1-2s)^{-M}\left(\frac{2\sqrt{\kappa_s}}{\sqrt{\tau_s}}\right)^{2N}\exp\left(\frac {N}{\tau_s}\Big[L^2+\Big(1+\frac{L^2}{\tau_s}\Big)\|\underline{t_z\vz+t_\theta\theta}\|^2\Big]\right)\\
&\leq \exp\Big(Nc_s\big(1+\|\underline{t_z\vz+t_{\theta}\theta}\|^2\big)\Big)
\end{aligned}
\ee
where, using $M=\alpha N$,
\be
c_s= \max\left\{-\alpha\log(1-2s),\;\log(4\kappa_s/\tau_s),\;\frac{L^2}{\tau_s}, \;\frac{1}{\tau_s}+\frac{L^2}{\tau_s^2}\right\}.
\ee
The integrability conditions $1-2s>0$ and $\tau_s>0$ hold for all $|s|,|a|\le\varepsilon$ provided
\be
\varepsilon < \min\Big(\frac12,\;\frac\kappa3\Big).
\ee
Set  $\bar{c}=\sup_{|s|\leq\epsilon}c_s$
then combining \eqref{concnorm1} with \eqref{concnorm2} one obtains 
\be\label{eq:expmomentfiner}
\left\langle e^{s\|\bx\|^2}\right\rangle_{(a)}\ \le\ \exp\Big(Nc_1\big(1+\|\underline G\|_{op}^2+t_m^2+t_\xi^2+t_v^2\|\underline{\uv}\|^2\big)+N \bar c\big(1+\|\underline{t_z\vz+t_\theta\theta}\|^2\big)\Big).
\ee
Writing $A:=\|\underline G\|_{op}^2+t_m^2+t_\xi^2+t_v^2\|\underline{\uv}\|^2\ge0$ and $B:=\|\underline{t_z\vz+t_\theta\theta}\|^2\ge0$, so that $B_N=1+A+B$, and using

\be
\begin{aligned}
&c_1(1+A)+\bar c(1+B)\ \le\ \max(c_1,\bar c)\big[(1+A)+(1+B)\big]\ \\&=\ \max(c_1,\bar c)(2+A+B)\ \le\ 2\max(c_1,\bar c)(1+A+B),
\end{aligned}
\ee
therefore
\be
\left\langle e^{s\|\bx\|^2}\right\rangle_{(a)}\ \le\, \exp(N K_x B_N)
\ee
with $ K_x=2\max(c_1,\bar c)$, proving \eqref{eq:exponentialmoment}. Notice since $c_s$ is continuous in $(\kappa,s)$ then $\bar c$  is continuous in $\kappa$ and then same holds for $ K_x$.

\end{proof}

A direct consequence of the previous result  is the following
\begin{lemma} [Concentration of the norm]\label{lem:norm_control}
There exists $\varepsilon>0$ such that if $|a|\leq \varepsilon$ one has:

\be
\left\langle 
\id \left(
\|\bx\|\ge  \sqrt{\frac{2K_x}{\varepsilon} N B_N}\right)
\right\rangle_{(a)}
\le \exp(- K_x N).
\ee
and given   $p \ge 1$ one has 
\be\label{eq:momentbound}
\left\langle \|\bx\|^{2p} \right\rangle_{(a)}
\le C_{p,\varepsilon}\Big(1+ K_x^p N^p \left(1+\|\underline{G}\|_{op}^{2p} +t_m^{2p}+t_{\xi}^{2p} +t^{2p}_v\|\underline{v}\|^{2p} + \|\underline{t_z\vz+t_{\theta}\theta}\|^{2p} 
\right)\Big)
\ee
for some $C_{p,\varepsilon}>0$.

\end{lemma}
\begin{proof}

Fix $s,\delta > 0$, by exponential Markov inequality,
\be
\left\langle 
\id(\|\bx\| \ge \delta)
\right\rangle_{(a)}
\le e^{-s   \delta^2}
\left\langle \exp\big( s \|\bx\|^2 \big) \right\rangle_{(a)}.
\ee

Using Lemma \ref{lem:concnorm}
and choosing $s=\varepsilon$ and 
\be
\delta = \sqrt{\frac{2K_x}{\varepsilon} N B_N}
\ee
yields the first result. By  Lemma 3.1.8 in \cite{Tala_vol1}
one has that if $X\geq0$  is a r.v. then 
\be
\mathbb{E}X^p\leq  2^p \Big(p^p+ (\log\mathbb{E}e^X)^p\Big).
\ee
and then \eqref{eq:exponentialmoment} and Jensen inequality implies there exists $\varepsilon>0$ such that for any  $|s|\leq \varepsilon$ one has 

\be
\langle \|\bx\|^{2p}\rangle_{(a)} 
\le \Big(\frac{2}{s}\Big)^p\Big(p^p+6^{p-1} K_x^p N^p \left(1+\|\underline{G}\|_{op}^{2p} +t_m^{2p}+t_{\xi}^{2p} +t^{2p}_v\|\underline{v}\|^{2p} + \|\underline{t_z\vz+t_{\theta}\theta}\|^{2p} 
\right)\Big).
\ee

hence for $s=\varepsilon$ \eqref{eq:momentbound} holds  with $C_{p,\varepsilon}=(\frac{2}{\varepsilon})^p\max(p^p,6^{p-1})$.
\end{proof}

We are now ready to prove the following:
\begin{prop}\label{pro:linear growth}
Assume that the functions $u,\phi$ satisfies hypothesis H1,H2,H3, then given $t\in[0,1]$ and the choice of  $\mathcal{P}_t$

\begin{itemize}

\item[\emph{(i)}] There exists $K_Q=K_Q(\kappa,\delta(t),\alpha,\lambda,C,L)>0$ independent on $N$ and  continuous  w.r.t. the entry  $\delta(t)$,  such that
\be\label{eq:momentboundglobal}
\E\langle Q_{11}\rangle_t\ \leq K_Q\big(1+m^2(t)+h^2(t)+\rho(t)+r(t)\big).
\ee

\item[\emph{(ii)}] Assuming $\rho(t)-q(t) \ \le\ \kappa^{-1} $ for all  $t\in[0,1]$   then there exists $K_S=K_S(\kappa,\delta(t),\alpha,\lambda,C,L)$ independent of $N$ and  continuous  w.r.t. the entry  $\delta(t)$, such that
\be\label{eq:momentglobal3}
\E\langle S_{12}\rangle_t,\ \ |\E\langle S_{11}\rangle_t| \ \le\ K_S\big(1+m^2(t)+h^2(t)+\rho(t)+r(t)\big) 
\ee

\item[\emph{(iii)}] All the order parameters have finite second moment, namely  there exists $K_2<\infty$ such that 
\begin{equation}\label{eq:S2-bound}
\mathbb E\big\langle S_{11}^2\big\rangle_t\;,\;
\mathbb E\big\langle S_{12}^2\big\rangle_t\;,\; \mathbb E\big\langle Q_{11}^2\big\rangle_t\;,\;\E\big\langle Q_{12}^2\big\rangle_t\;,\; \mathbb E\big\langle H_{1}^2\big\rangle_t\;,\;
\mathbb E\big\langle M_{1}^2\big\rangle_t,\;\le\;K_2 
\end{equation}
uniformly for $t\in[0,1]$ and $N$.
\end{itemize}

\end{prop}

\begin{proof}\emph{(i)} Let us start with a basic fact: for every $p\ge 1$ there is $K_p<\infty$, independent of $N$ (and obviously  also from $t$ and the choice of the path $\mathcal{P}_t)$) , such that
\be\label{eq:trivialbounds}
\E\|\underline\theta\|^{2p},\ \E\|\underline{\vz}\|^{2p},\ \E\|\underline{\uv}\|^{2p},\ \E\|\underline y\|^{2p},\ \E\|\underline Z\|_{op}^{2p}\ \le\ K_p .
\ee
This follows from standard concentration/moment estimates for i.i.d.\ Gaussian vectors and matrices. Notice also that $\langle\,  \rangle_t$ is the average w.r.t. the random Gibbs measure induced by the interpolating Hamiltonian, hence $\langle\, \rangle_t\equiv \langle \,\rangle_{(a)}|_{a=0}$ and all the previous results (in particular Lemma \ref{lem:concnorm} and \ref{lem:norm_control}) apply.

By definition of $Q_{11}$ and using \eqref{eq:momentbound} with $p=1$ one gets

\be\label{eq:proofboundQ11}
\begin{aligned}
&\E \langle Q_{11} \rangle_t=\E\left \langle \frac{\|\sigma\|^2}{N}\right\rangle_t\leq  \E \left\langle \frac{\|\bx\|^2}{N}\right \rangle_t\\
&\leq C_0 K_x\Big(1+  \,\E\left (\|\underline{G}\|_{op}^{2} +t_m^{2}+t_{\xi}^{2} +t^{2}_v\|\underline{v}\|^{2} +\|\underline{t_z\vz+t_{\theta}\theta}\|^{2} \right)\Big)+O(1/N)
\end{aligned}
\ee
for some absolute constant $C_{0}$ independent from $N$ and $t,\mathcal{P}_t$ and $K_x\equiv K_x(\kappa+\delta(t),\alpha,C,L)$ is the same costant of Lemma \ref{lem:concnorm}. It is not difficult to check that 
\be\label{eq:Gop-rescaled}
\|\underline G\|_{op}^2 \;\le\; 2\alpha t_y^2\,\|\underline y\|^2\|\underline\theta\|^2 \;+\; 2t_Z^2\,\|\underline Z\|_{op}^2 \,,\quad \|\underline {t_zz+t_{\theta}\theta}\|^2\leq 2 t^2_z\|\underline {z}\|^2+2t^2_\theta \|\underline{\theta}\|^2
\ee
and then \eqref{eq:proofboundQ11} in combination with \eqref{eq:trivialbounds} proves \eqref{eq:momentboundglobal}.\\

\emph{(ii)}. Consider the quantity
\be
\Sigma:=\E\left\langle\frac{1}{N}\sum_{\mu\leq M}(u'(S_{\mu}))^2 \right\rangle_t 
\ee
We claim that from an upper bound for $\Sigma$ follows directly an upper bound for   $S_{11}$ and $S_{12}$. 
Indeed for any $\mu\leq M$ and replicas $x^1,x^2$ one has 

\be
u'(S^{1}_{\mu})u'(S^{2}_{\mu})\leq \frac{1}{2} \Big((u'(S^{1}_{\mu}))^2+(u'(S^{2}_{\mu}))^2\Big)
\ee
and hence, taking $\E\langle \,\,\rangle$, using that the two replicas are i.i.d.\ under the product measure, and summing over $\mu$, one gets

\be
\E\langle S_{12}\rangle_t \ \le\ \,\Sigma
\ee
On the other hand, using H3 ($|u''|\le C$) one gets 
\be
|\E\langle S_{11}\rangle_t| \ \le\ C\alpha+\Sigma\ 
\ee

Now by H1, $(u'(s))^2\le 2C^2(1+s^2)$, thus entailing

\be\Sigma\le 2C^2\alpha+2C^2\E \left\langle \frac{1}{N}\sum_{\mu\leq M} S_{\mu}(x)^2\right\rangle_t 
\ee
 Recall that $S_\mu(x)=g_\mu\cdot\vs/\sqrt N+y_\mu t_m+t_vv_\mu+t_\xi\xi_\mu$, therefore, by Cauchy--Schwarz,
\be\label{eq:Smuexpansion}
\frac{1}{N}\sum_{\mu\le M}S_\mu(x)^2 \ \le\ 4\Big(\|\underline G\|^2_{op}\|\underline\vs\|^2+t_m^2\,\|\underline y\|^2+t_v^2\|\underline{\uv}\|^2+t_\xi^2\|\underline\uxi\|^2\Big),
\ee
where $\underline\vs=\vs/\sqrt N$. We now take $\E\langle\,\cdot\,\rangle_t$ of \eqref{eq:Smuexpansion} term by term. The second and the third  terms are handled directly
\be\label{eq:handle1}
\E\big[t_m^2\|\underline y\|^2\big]=\alpha t_m^2,\quad  \E\big[t_v^2\langle\|\underline{\uv}\|^2\rangle_t\big]=\alpha t_v^2
\ee
For the last term notice 
that $\|\underline\uxi\|^2\le\|\underline\bx\|^2$) and then one has the same bound of  \eqref{eq:proofboundQ11} and keeping in mind  the hypothesis $t_\xi^2=\rho(t)-q(t)\le\kappa^{-1}$, one gets
\be\label{eq:handle2}
t_\xi^2\,\E\langle\|\underline\uxi\|^2\rangle_t \ \le\ \frac{K_Q}{\kappa}\big(1+m^2(t)+h^2(t)+\rho(t)+r(t)\big),
\ee

For the remaining term, $\|\underline G\|_{op}^2\|\underline\vs\|^2=\|\underline G\|_{op}^2 Q_{11}$,  by the Cauchy--Schwarz inequality applied to the expectation over the disorder, one has 
\be
\E\big\langle\|\underline G\|_{op}^2\,Q_{11}\big\rangle_t \ \le\ \big(\E\|\underline G\|_{op}^4\big)^{1/2}\big(\E\langle Q_{11}^2\rangle_t\big)^{1/2}.
\ee
The factor $\E\|\underline G\|_{op}^4$ is bounded by a constant independent of $N$ and $t,\mathcal P_t$. The second factor is controlled by the $p=2$ case of \eqref{eq:momentbound} which gives
\be\label{eq:handle3}
\begin{aligned}
&\E \langle Q_{11}^2 \rangle_t=\E\left \langle \frac{\|\sigma\|^4}{N^2}\right\rangle_t\leq  \E \left\langle \frac{\|\bx\|^4}{N^2}\right \rangle_t\\
&\leq C'_0\Big(1+ K_x^2 \,\E\left (1+\|\underline{G}\|_{op}^{4} +t_m^{4}+t_{\xi}^{4} +t^{4}_v\|\underline{v}\|^{4} +\|\underline{t_z\vz+t_{\theta}\theta}\|^{4} \right)\Big)
\end{aligned}
\ee
hence 
\be
\left(\E\langle Q_{11}^2\rangle_t\right)^{1/2} \ \le\ \sqrt{C'_0} \Big(1+K_x(1+t^2_{y}+t^2_{Z}+t^2_{m}+t^2_{\xi}+t^2_{v}+t^2_{z}+t^2_{\theta})\Big)
\ee
Combining  \eqref{eq:handle1} \eqref{eq:handle2} and \eqref{eq:handle3} one gets \eqref{eq:momentglobal3}.

\emph{(iii)}. Given a configuration $x=(\sigma,\xi)$, set 
\begin{equation}
\Sigma (x)\;:=\;\frac1N\sum_{\mu\le M} S_\mu(x)^2 .
\end{equation}
By H3 we have $|u''|\le C$, and by H1, $(u'(s))^2\le 2C^2(1+s^2)$. Hence
\begin{equation}
| S_{11}|\;\le\;\frac1N\sum_{\mu\le M}\Big(|u''(S_\mu)|+(u'(S_\mu))^2\Big)
\;\le\;(C+2C^2)\,\alpha \;+\; 2C^2\, \Sigma(x),
\end{equation}
so that
\begin{equation}\label{eq:S11-red}
S_{11}^2\;\le\;2(C+2C^2)^2\alpha^2\;+\;8C^4\,\Sigma(x)^2 .
\end{equation}
For $S_{12}$, the inequality $2|ab|\le a^2+b^2$ gives
\begin{equation}
|S_{12}|\;\le\;\frac1{2N}\sum_{\mu\le M}
\Big[(u'(S_\mu(x^1)))^2+(u'(S_\mu(x^2)))^2\Big]
\;\le\; C^2\Big(2\alpha+\Sigma(x^1)+\Sigma(x^2)\Big),
\end{equation}
whence, using $(a+b+c)^2\le 3(a^2+b^2+c^2)$ and the fact that under 
$\nu_t^{\otimes 2}$ the two replicas are identically distributed,
\begin{equation}\label{eq:S12-red}
\big\langle S_{12}^2\big\rangle_t
\;\le\;12C^4\alpha^2\;+\;6C^4\,\big\langle \Sigma(x)^2\big\rangle_t .
\end{equation}
By \eqref{eq:S11-red}--\eqref{eq:S12-red} it suffices to prove
\begin{equation}\label{eq:goal}
\sup_{t\in[0,1]}\sup_N\;\mathbb E\big\langle \Sigma(x)^2\big\rangle_t <\infty
\end{equation}

Since $S_\mu(x)=\underline g_\mu\cdot\sigma+y_\mu t_m+t_v v_\mu+t_\xi\xi_\mu$ 
with $\underline g_\mu=g_\mu/\sqrt N$, and $y_\mu^2=1$, Cauchy--Schwarz yields 
$S_\mu^2\le 4\big[(\underline g_\mu\cdot\sigma)^2+t_m^2+t_v^2v_\mu^2
+t_\xi^2\xi_\mu^2\big]$. Summing over $\mu\le M$ and dividing by $N$, and using 
$\sum_\mu(\underline g_\mu\cdot\sigma)^2=\|\underline G\sigma\|^2
\le \|\underline G\|_{\rm op}^2\|\sigma\|^2 = N\,\|\underline G\|_{\rm op}^2
\|\underline\sigma\|^2$,
\begin{equation}
\Sigma(x)\;\le\;4\Big[\|\underline G\|_{\rm op}^2\,\|\underline\sigma\|^2
+\alpha t_m^2+t_v^2\|\underline v\|^2+t_\xi^2\|\underline\xi\|^2\Big],
\end{equation}
and therefore
\begin{equation}\label{eq:T2}
\Sigma(x)^2\;\le\;64\Big[\|\underline G\|_{\rm op}^4\,\|\underline\sigma\|^4
+\alpha^2 t_m^4+t_v^4\|\underline v\|^4+t_\xi^4\|\underline\xi\|^4\Big].
\end{equation}

We need to bound  $\mathbb E\langle\cdot\rangle_t$ for all the terms in \eqref{eq:T2}.

For the second and the third term is enough to exploit the concentration of $v,z$ and \eqref{eq:hipbpund}. For the last term  note that $\|\underline\xi\|^4\le \|x\|^4/N^2$, and 
 then applying  \eqref{eq:momentbound}  with $p=2$:
\begin{equation}
\mathbb E\big\langle\|\underline\xi\|^4\big\rangle_t
\;\le\;\frac{C_{2,\varepsilon}}{N^{2}}
+C_{2,\varepsilon} K_x^{2}\,
\mathbb E\Big[1+\|\underline G\|_{\rm op}^{4}+t_m^{4}+t_\xi^{4}
+t_v^{4}\|\underline v\|^{4}
+\|t_z\underline z+t_\theta\underline\theta\|^{4}\Big]
\end{equation}
This implies   $t_\xi^4\,\mathbb E\langle\|\underline\xi\|^4\rangle_t<\infty$ uniformly in $t\in[0,1]$ and $N$. For the first term, Cauchy--Schwarz with respect to the disorder 
expectation and Jensen's inequality 
$\langle\|\underline\sigma\|^4\rangle_t^2\le\langle\|\underline\sigma\|^8\rangle_t$ 
give
\begin{equation}
\mathbb E\Big[\|\underline G\|_{\rm op}^4\,
\big\langle\|\underline\sigma\|^4\big\rangle_t\Big]
\;\le\;\Big(\mathbb E\|\underline G\|_{\rm op}^8\Big)^{1/2}
\Big(\mathbb E\big\langle\|\underline\sigma\|^8\big\rangle_t\Big)^{1/2}.
\end{equation}
The first factor is finite uniformly in $N,t$. For the 
second, $\|\underline\sigma\|^8\le\|x\|^8/N^4$ and 
\eqref{eq:momentbound} with $p=4$ give, exactly as before,
\begin{equation}
\mathbb E\big\langle\|\underline\sigma\|^8\big\rangle_t
\;\le\;\frac{C_{4,\varepsilon}}{N^{4}}
+C_{4,\varepsilon} K_x^{4}\,
\mathbb E\Big[1+\|\underline G\|_{\rm op}^{8}+t_m^{8}+t_\xi^{8}
+t_v^{8}\|\underline v\|^{8}
+\|t_z\underline z+t_\theta\underline\theta\|^{8}\Big]
\;<\infty,
\end{equation}
again uniformly in $N,t$.
Bounds for the second moment of $Q_{11},Q_{12},M_1,H_1$ can be easily obtained by the same methods.

\end{proof}

\section{Concentrations}\label{app:concentrations}
The aim of this section is to establish concentration properties for the order parameters defined in \eqref{eq:def_S11-S12}  with respect to  the (random) Boltzmann-Gibbs measure induced by the interpolating  Hamiltonian \eqref{eq:interpolating_H}. In particular we are interested in dependence of the above properties on the parameters of the interpolation, namely the real   numbers  $\kappa,\alpha,\lambda$ and the interpolation path which is determined by a  collection of  functions $\mathcal{P}_t=(\rho(t), q(t), m(t), r(t), \delta(t), h(t))$ obeying to the constraints  $\rho\geq q\geq0$  and $r,\delta\geq 0$. In this section we assume that these function are continuous which implies that 
\be\label{eq:hipbpund}
\sup_{t\in[0,1]} \max\{\rho(t),r(t),\delta(t), h(t),m(t)\}<\infty
\ee
\subsection{Concentration of the free energy}

We will prove that the perturbed free energy $f_a$ in \eqref{eq:perturbed} 
 associated with  the Hamiltonian \(H_a\) in
\eqref{interaction} is a self-averaging quantity.

\begin{lemma}[Self-averaging of $f_a$]\label{pro:selfavf}
There exists $\varepsilon>0$ such that if $|a|\leq \varepsilon$ then there exists $K_f>0$ independent form $t,N,a $ such that

\be\label{eq:variancebound}
\mathrm{Var}(f_a)\leq \frac{K_f}{N}
\ee

\end{lemma}

\begin{proof}
Recall that 
\be f_a\equiv f_a(y,Z,v,z,\theta)
    :=
    \frac1N\log Z_a(y,Z,v,z,\theta)
\ee

 namely $f_a$ is a function of the random variables $\mathcal D=(y,Z,\vz,\uv,\theta)$  that are  mutually independent. By the law of total variance
\be \operatorname{Var}(f_a) = \mathbb E_{y,Z,v,z} \left[ \operatorname{Var}_{\theta} \bigl(f_a\mid y,Z,v,z\bigr) \right] + \operatorname{Var}_{y,Z,v,z} \left( \mathbb E_{\theta} \bigl[f_a\mid y,Z,v,z\bigr] \right). \label{eq:total-variance-fa} \ee

The next lemma control the first term in the r.h.s. of \eqref{eq:total-variance-fa}.

\begin{lemma}\label{lem:conctheta}
Given a realization of $(y,Z,v,z)$ consider the function \[
    f_a(\theta)
    :=
    \frac1N\log Z_a(\theta),
    \qquad |a|\leq\varepsilon,
\]
where $\varepsilon>0$ is such that \eqref{eq:momentbound} holds. Assume that coordinates $\theta =(\theta_i)_{i\leq N}$ are i.i.d. with  a single component  satisfies a Poincar\'e inequality with
constant $C_{\theta}$ independent of $N$.Then there exists a constant $K_{\theta}<\infty$, independent of $N$, $t$ and
$|a|\leq\varepsilon$, such that
\be\label{eq:vartheta}
    \mathbb E_{y,Z,v,z}
    \operatorname{Var}_{\theta}
    \bigl(f_a\mid y,Z,v,z\bigr)
    \leq \frac{K_{\theta}}{N}.
\ee
\end{lemma}

\begin{proof}
We first derive a bound for the gradient of $f_a$ with respect to
$\theta$. A direct computation gives 
\begin{equation}
    \partial_{\theta_i}f_a
    =
    \frac1N
    \left\langle
        \sigma_i^1A_1+\sigma_i^2A_2
    \right\rangle_{(a)}.
    \label{eq:theta-gradient-coordinate}
\end{equation}

where  
\be
    A_\ell
    :=
    t_\theta+
    \frac{t_y}{N}
    \sum_{\mu=1}^{M}
    y_\mu u'_{y_\mu}(S_\mu^\ell)\,\quad  \ell=1,2,
\ee

By Jensen's inequality,
\be
    \left|\partial_{\theta_i}f_a\right|^2
    \leq
    \frac1{N^2}
    \left\langle
        \left(
            \sigma_i^1A_1+\sigma_i^2A_2
        \right)^2
    \right\rangle_{(a)}.
\ee
Summing over $i$ and using $(x+y)^2\leq2x^2+2y^2$, we obtain
\begin{align}
    \|\nabla_\theta f_a\|^2
    &=
    \sum_{i=1}^{N}
    \left|\partial_{\theta_i}f_a\right|^2
    \notag\\
    &\leq
    \frac1{N^2}
    \left\langle
        \|A_1\sigma^1+A_2\sigma^2\|^2
    \right\rangle_{(a)}
    \notag\\
    &\leq
    \frac{2}{N^2}
    \sum_{\ell=1}^{2}
    \left\langle
        A_\ell^2\|\sigma^\ell\|^2
    \right\rangle_{(a)}.
    \label{eq:global-gradient-bound}
\end{align}

By Cauchy--Schwarz and $y_\mu^2=1$,
\begin{align}
    A_\ell^2
    &\leq
    2t_\theta^2+
    \frac{2t_y^2}{N^2}
    \left(
        \sum_{\mu=1}^{M}
        y_\mu u'_{y_\mu}(S_\mu^\ell)
    \right)^2
    \notag\\
    &\leq
    2t_\theta^2+
    2\alpha t_y^2
    \frac1N
    \sum_{\mu=1}^{M}
    \bigl(u'_{y_\mu}(S_\mu^\ell)\bigr)^2,
    \label{eq:A-square}
\end{align}
By hypothesis H1 one has 
\[
    A_\ell^2
    \leq C\left(1+\frac1N\sum_{\mu=1}^{M}|S_\mu^\ell|^2
    \right).
\]

Hence using \eqref{eq:Smuexpansion}
one gets 

\be
A_\ell^2\|\sigma^\ell\|^2
    \leq C_{\ell}
    \Bigg[
        \|\sigma^\ell\|^2+
        \|\underline G\|_{\mathrm{op}}^2\frac{\|\sigma^\ell\|^4}{N}+\frac{\|v\|^2}{N}\|\sigma^\ell\|^2+\frac{\|\xi^\ell\|^2}{N}\|\sigma^\ell\|^2
    \Bigg].
    \label{eq:A-sigma-bound}
\ee
for some $C_{\ell}<\infty$ an then using the moment estimates in \eqref{eq:momentbound} one obtains uniformly in $t\in[0,1]$ and
$|a|\leq\varepsilon$,
\begin{equation}
    \mathbb E
    \left\langle
        A_\ell^2\|\sigma^\ell\|^2
    \right\rangle_{(a)}
    \leq \tilde C_{\ell} N.
    \label{eq:A-sigma-expectation}
\end{equation}
for some $\tilde C_{\ell}>0$. Substituting \eqref{eq:A-sigma-expectation} into
\eqref{eq:global-gradient-bound} yields
\begin{equation}
    \mathbb E\|\nabla_\theta f_a\|^2
    \leq \frac{C_{\theta}}{N}.
    \label{eq:theta-gradient-final}
\end{equation}

for some $C_{\theta}>0$. Conditionally  on $(y,Z,v,z)$, since the product law
$P_\theta^{\otimes N}$ satisfies the tensorized Poincar\'e inequality
with the same constant $C_{\theta}$,
\[
    \operatorname{Var}_{\theta}
    \bigl(f_a\mid y,Z,v,z\bigr)
    \leq
    C_{\mathrm P}
    \mathbb E_\theta
    \left[
        \|\nabla_\theta f_a\|^2
        \,\middle|\,y,Z,v,z
    \right].
\]
Averaging over $(y,Z,v,z)$ and using
\eqref{eq:theta-gradient-final}, we obtain
\[
    \mathbb E_{y,Z,v,z}
    \operatorname{Var}_{\theta}
    \bigl(f_a\mid y,Z,v,z\bigr)
    \leq \frac{K_{\theta}}{N}.
\]

for some $K_{\theta}>0$.

\end{proof}

For the second term  in \eqref{eq:total-variance-fa} we notice that  conditionally on $\theta$,  each of the remaining variables   satisfies a Poincar\'e inequality. Indeed $Z,z,v$ are independent standard Gaussian vectors,  while $y$ satisfies a Rademacher-Poincar\'e inequality (see Lemma \ref{lem:radpoinc} in the Appendix ) for a suitable difference operator. Hence by tensorization  and Jensen inequality  one gets

\be\label{eq:varddeco}
\operatorname{Var}_{y,Z,v,z} \left( \mathbb E_{\theta} \bigl[f_a\mid y,Z,v,z\bigr] \right) \le  \sum_\mu\mathbb E\big[(D_{y_\mu}f_a)^2\big]+\mathbb{E}\|\nabla_Z f_a\|^2 + \mathbb{E}\|\nabla_z f_a\|^2  + \mathbb{E}\|\nabla_v f_a\|^2 
\ee

where  $D_{y_\mu}$ denotes the difference operator  defined in \eqref{def:differenceoperator}.
Let's start computing the gradient in the last three terms.

\be
\begin{aligned}
\|\nabla_{Z} f_a\|^2 &=\sum_{\mu,i}(\partial_{Z_{\mu,i}}f_a)^2=\frac{t_Z^2}{N^3}\sum_{\mu,i}\Big(\left\langle\sum_{\ell=1,2} u'(S_{\mu}(x^{\ell})\sigma^{\ell}_i\right \rangle_{(a)} \Big)^2 \\
&\leq \frac{t_Z^2}{N^3}\sum_{\mu,i}\left\langle \Big(\sum_{\ell=1,2}u'(S_{\mu}(x^{\ell}) \sigma^{\ell}_i\Big)^2\right \rangle_{(a)} \leq \frac{t_Z^2}{N^3}\left\langle\sum_{\mu,\ell}\Big(u'(S_{\mu}(x^{\ell})\Big)^2\sum_{i,\ell}(\sigma_i^{\ell})^2\right\rangle\\
&\leq \frac{4C^2t_Z^2}{N^3}\Big(M+\left\langle \sum_{\ell}\Big(\|\underline{G}\|^2_{op}\|\sigma^{\ell}\|^2+\|\xi^{\ell}\|^2+\|v\|^2\Big)\,\sum_{\ell}\|\sigma^{\ell}\|^2\right\rangle_{(a)}\Big)\\
&= \frac{4C^2t_Z^2}{N^3}\Big(M+\left\langle \Big(\|\underline{G}\|^2_{op}\|\boldsymbol{\sigma}\|^2+\|\boldsymbol{\xi}\|^2+2\|v\|^2\Big)\,\|\boldsymbol{\sigma}\|^2\right\rangle_{(a)}\Big)\\
&\leq \frac{8C^2t_Z^2}{N^3}\Big(M+\left\langle \|\underline{G}\|^2_{op}\|\boldsymbol{\sigma}\|^4+ \|\boldsymbol{\xi}\|^4 +\|\boldsymbol{\sigma}\|^4 +2\|v\|^2\|\boldsymbol{\sigma}\|^2 \right\rangle_{(a)}\Big)
\end{aligned}
\ee
where in the third line  we used that $|u'(s)|\leq C(1+|s|)$ and subsequently the inequality $(a+b)^2\leq 2(a^2+b^2)$. Taking the expectation and using the moment bounds   \eqref{eq:momentbound} and \eqref{eq:pmomentG} one obtains
\be
\E \|\nabla_{Z} f_a\|^2 \leq \frac{K_Z}{N}
\ee
for some $K_Z>0$. It's also clear from  \eqref{eq:momentbound}
 that $\sup_{t\in[0,1]} K_Z<\infty$

In the same manner one finds

\be\label{eq: combining}
\begin{aligned}
&\E\,\|\nabla_{v} f_a\|^2\leq \E\, \frac{8C^2t^2_v}{N^2}\left\langle \|\underline{G}\|^2_{op}\|\boldsymbol{\sigma}\|^2+ \|\boldsymbol{\xi}\|^2 +\|v\|^2\right\rangle_{(a)} \leq \frac{K_v}{N}\\
&\E\, \|\nabla_{z} f_a\|^2\leq\E\,\frac{2t^2_z}{N^2}\left\langle\|\boldsymbol{\sigma}\|^2\right\rangle_{(a)}\leq \frac{K_z}{N}\\
&\E\, \|\nabla_{\theta} f_a\|^2\leq\E  \frac{4}{N^2}\,\big\langle \Big(t_\theta+\frac{t_y}{N}\sum_{\mu\le M}y_\mu\,u'(S_\mu)\Big)^2\,\|\sigma\|^2\big\rangle_{(a)}
\leq \frac{K_{\theta}}{N}
\end{aligned}
\ee

for some $K_{v}, K_{z}, K_{\theta} >0$ with finite supremum norm for $t\in[0,1]$ .\\

It remains to control the first term on the right-hand side of
\eqref{eq:varddeco}, namely the norm of the difference operator $D_{y_\mu}f_a$.  Fix \(\mu\leq M\), and keep all the disorder variables other than
\(y_\mu\) fixed. For $\eta\in \{-1,+1\}$ set 
\be
\mathcal H_{\mu,\eta}(\bx):= \mathcal H_a(\bx)|_{y_{\mu}=\eta} 
\ee 
and  denote the related partition function by 
\(Z_{\mu,\eta}\)  and its  Gibbs expectations by
$(\langle\cdot\rangle_{\mu,\eta}\). For  \(x=(\sigma,\xi)\), write $S_\mu^\eta(x)$
for the effective field obtained after setting \(y_\mu=\eta\), and define
\begin{equation}
    \Delta_\mu(x)
    :=
    u_{+1}\bigl(S_\mu^+(x)\bigr)
    -
    u_{-1}\bigl(S_\mu^-(x)\bigr).
    \label{eq:Delta-mu}
\end{equation}
namely the variation of the $u$ contribution to energy  under the $\mu$-label replacement.
Hence 
\begin{equation}
\mathcal H_{\mu,-}(\bx)-\mathcal H_{\mu,+}(\bx)
    = A_\mu(\bx)
    \label{eq:H-difference}
\end{equation}

where 
    $A_\mu(\bx)
    := \sum_{\ell=1}^{2}\Delta_\mu(x^\ell)$. Now observe that  $ Z_{\mu,+}
    =
    Z_{\mu,-}
    \left\langle e^{A_\mu}\right\rangle_{\mu,-}$ and then Jensen's inequality gives
\be
    \log Z_{\mu,+}-\log Z_{\mu,-}
    \geq
    \langle A_\mu\rangle_{\mu,-}.
\ee
Interchanging the roles of \(+\) and \(-\) one gets
\[
    \log Z_{\mu,+}-\log Z_{\mu,-}
    \leq
    \langle A_\mu\rangle_{\mu,+}.
\]
Therefore,
\begin{equation}
    \langle A_\mu\rangle_{\mu,-}
    \leq
    \log Z_{\mu,+}-\log Z_{\mu,-}
    \leq
    \langle A_\mu\rangle_{\mu,+}.
    \label{eq:logZ-sandwich}
\end{equation}
Since
\[
    D_\mu f_a
    =
    \frac{1}{2N}
    \left(
        \log Z_{\mu,+}-\log Z_{\mu,-}
    \right),
\]
we obtain
\begin{equation}
    |D_\mu f_a|
    \leq
    \frac{1}{2N}
    \max_{\eta\in\{-1,+1\}}
    \left|
        \left\langle
        A_{\mu}
        \right\rangle_{\mu,\eta}
    \right|\leq 
    \frac{1}{2N}
    \max_{\eta\in\{-1,+1\}}
\left\langle\sum_{\ell=1}^{2} \left|\Delta_\mu(x^\ell)\right|
        \right\rangle_{\mu,\eta}
    .
    \label{eq:Dmu-comparison}
\end{equation}

Observe that the dependence of the effective field $S_{\mu}$ on \(y_\mu\) is affine, indeed 
\begin{equation}
    S_\mu^\eta(x)
    =
    A_\mu^{(0)}(x)
    +
    \eta B(x),
    \qquad
    B(x)
    :=
    t_y\frac{\theta\cdot\sigma}{N}+t_m,
    \label{eq:field-label-decomposition}
\end{equation}
where \(A_\mu^{(0)}(x)\) is independent of $\eta$, hence
\begin{equation}
    S_\mu^+(x)-S_\mu^-(x)=2B(x).
    \label{eq:field-difference}
\end{equation}

From the definition \eqref{eq:Delta-mu} adding and subtracting  \(u_{-1}(S_\mu^+)\) one obtains 
\begin{align}
|\Delta_\mu(x)|
    &\leq
    \left|
        u_{+1}(S_\mu^+)-u_{-1}(S_\mu^+)
    \right|
    +
    \left|
        u_{-1}(S_\mu^+)-u_{-1}(S_\mu^-)
    \right|.
    \label{eq:Delta-plus-decomposition}
\end{align}
The first term is bounded by 
\eqref{eq:linear-label-discrepancy} in H4 . For the second one, H1 and the
fundamental theorem of calculus imply that, for all \(r,s\in\mathbb R\),
\begin{align}
    |u_y(r)-u_y(s)|
    &\leq
    |r-s|
    \int_0^1
    \left|
        u_y'\bigl(s+\lambda(r-s)\bigr)
    \right|\,d\lambda
    \notag\\
    &\leq
    C|r-s|\bigl(1+|r|+|s|\bigr).
    \label{eq:linear-derivative-MVT}
\end{align}
Therefore,
\begin{equation}
    |\Delta_\mu(x)|
    \leq
C \bigl(1+|S_\mu^+(x)|\bigr)
    +
    C|B(x)|
    \bigl(
        1+|S_\mu^+(x)|+|S_\mu^-(x)|
    \bigr).
    \label{eq:Delta-plus-initial}
\end{equation}

Since
\[
    |S_\mu^-(x)|
    \leq
    |S_\mu^+(x)|+2|B(x)|,
\]
we obtain
\begin{equation}
    |\Delta_\mu(x)|
    \leq
    K_{\Delta}\bigl(1+|S_\mu^+(x)|\bigr)
       \bigl(1+|B(x)|\bigr)
    +
    K_{\Delta}|B(x)|^2.
    \label{eq:Delta-plus-final}
\end{equation}
for some $K_{\Delta}>0$.
The same argument, adding and subtracting \(u_{+1}(S_\mu^-)\), gives

\begin{equation}
    |\Delta_\mu(x)|
    \leq K_{\Delta}
     \bigl(1+|S_\mu^-(x)|\bigr)
       \bigl(1+|B(x)|\bigr)
    +
    K_{\Delta} |B(x)|^2.
    \label{eq:Delta-minus-final}
\end{equation}
Thus, for either \(\eta\in\{-1,+1\}\),
\begin{equation}
    |\Delta_\mu(x)|
    \leq
    K_{\Delta}\bigl(1+|S_\mu^\eta(x)|\bigr)
       \bigl(1+|B(x)|\bigr)
    +
    K_{\Delta}|B(x)|^2.
    \label{eq:Delta-unified}
\end{equation}
and then from elementary inequalities that
\begin{equation}
    |\Delta_\mu(x)|^2
    \leq
    \tilde{K_{\Delta}}\bigl(1+|S_\mu^\eta(x)|^2\bigr)
       \bigl(1+|B(x)|^2\bigr)
    +
    \tilde{K_{\Delta}}|B(x)|^4.
    \label{eq:Delta-square}
\end{equation}

for some $\tilde{K_{\Delta}}>0$.  Applying Jensen's inequality to
\eqref{eq:Dmu-comparison}, and using
\[
    \max\{a_+,a_-\}^2\leq a_+^2+a_-^2,
\]
we obtain
\begin{align}
    \mathbb E(D_\mu f_a)^2
    &\leq
    \frac{\tilde K_{\Delta}}{N^2}
    \sum_{\eta\in\{-1,+1\}}
    \mathbb E
    \left\langle
        \sum_{\ell=1}^{2}
        \left[
            \bigl(1+|S_\mu^\eta(x^\ell)|^2\bigr)
            \bigl(1+|B(x^\ell)|^2\bigr)
            +
            |B(x^\ell)|^4
        \right]
    \right\rangle_{\mu,\eta}\\
    &
    = \frac{\tilde K_{\Delta}}{N^2} \sum_{\ell=1}^{2}\,
   \mathbb E
    \left\langle
        \left[\bigl(1+|S_\mu^\eta(x^\ell)|^2\bigr)
            \bigl(1+|B(x^\ell)|^2\bigr)
            +
            |B(x^\ell)|^4
        \right]
    \right\rangle_{(a)}
    \label{eq:Dmu-square-bound}
\end{align}

where we used the fact that for any measurable function $F$, one has  

\begin{equation}
    \frac12
    \sum_{\eta\in\{-1,+1\}}
    \mathbb E
    \left\langle
        F^\eta
    \right\rangle_{\mu,\eta}
    =
    \mathbb E\langle F\rangle_{(a)}.
    \label{eq:conditional-label-identity}
\end{equation}

Summing \eqref{eq:Dmu-square-bound} over \(\mu\)  we obtain
\begin{align}
    \sum_{\mu=1}^{M}\mathbb E(D_\mu f_a)^2
    &\leq
    \frac{\tilde K_{\Delta}}{N^2}
    \sum_{\ell=1}^{2}
    \mathbb E
    \left\langle
        \sum_{\mu=1}^{M}
        \bigl(1+|S_\mu(x^\ell)|^2\bigr)
        \bigl(1+|B(x^\ell)|^2\bigr)
        +
        M|B(x^\ell)|^4
    \right\rangle_{(a)}.
    \label{eq:sum-Dmu-initial}
\end{align}

Consider the quantity 
\begin{equation}
    \Sigma(x)
    :=
\frac1N\sum_{\mu=1}^{M}|S_\mu(x)|^2.
    \label{eq:def-Sigma}
\end{equation}
 then 
\[
    \sum_{\mu=1}^{M}
    \bigl(1+|S_\mu(x)|^2\bigr)
    =
    M+N\Sigma(x),
\]
and hence
\begin{align}
    \sum_{\mu=1}^{M}\mathbb E(D_\mu f_a)^2
    &\leq
    \frac{\tilde K_{\Delta}}{N^2}
    \sum_{\ell=1}^{2}
    \mathbb E
    \left\langle
        \bigl(M+N\Sigma(x^\ell)\bigr)
        \bigl(1+|B(x^\ell)|^2\bigr)
        +
        M|B(x^\ell)|^4
    \right\rangle_{(a)}.
    \label{eq:sum-Dmu-Sigma}
\end{align}

We now verify that every expectation on the right-hand side is $O(N)$. Since $B(x)= t_y\frac{\theta\cdot\sigma}{N}+t_m, $ and 
the interpolation coefficients \(t_y,t_m\) are uniformly bounded along
the admissible interpolation paths, there exists $K_B>0$ such that 
\begin{equation}
    |B(x)|
    \leq
    K_B\left(
        1+\frac{|\theta\cdot\sigma|}{N}
    \right).
    \label{eq:B-bound}
\end{equation}

and then 
\begin{equation}
    |B(x)|^4
    \leq
    C\left(
        1+\frac{|\theta\cdot\sigma|^4}{N^4}
    \right)
    \leq
    C\left(
        1+\frac{\|\theta\|^4\|\sigma\|^4}{N^4}
    \right).
    \label{eq:B-fourth-bound}
\end{equation}
By Cauchy--Schwarz,
\begin{align}
    \mathbb E
    \left\langle
        \frac{\|\theta\|^4\|\sigma\|^4}{N^4}
    \right\rangle_{(a)}
    &\leq
    \left(
        \mathbb E\frac{\|\theta\|^8}{N^4}
    \right)^{1/2}
    \left(
        \mathbb E
        \left\langle
            \frac{\|\sigma\|^8}{N^4}
        \right\rangle_{(a)}
    \right)^{1/2}
   \label{eq:B-fourth-moment}
\end{align}
The first factor is bounded because $\theta$  satisfy a  Poincar\'e inequality. The second  term is bounded  by 
\eqref {eq:momentbound} with \(p=4\) for $|a|\leq \varepsilon$. Thus
 $\E\langle |B(x)|^4\rangle_{(a)}$ and $
    \mathbb E\langle |B(x)|^2\rangle_{(a)}$ 
are uniformly bounded. Moreover using \eqref{eq:T2} one has that also $ \E\langle \Sigma(x)^2\rangle_{(a)}$ 
is uniformly bounded for $|a|$ small enough and $t\in[0,1]$ and then one can easily conclude that 
\be
    \sum_{\mu=1}^{M}\mathbb E(D_\mu f_a)^2
    \leq
    \frac{K_y}{N}.
    \label{eq:Rademacher-final}
\ee

Finally, combining \eqref{eq: combining}, \eqref{eq:vartheta}
 and
\eqref{eq:Rademacher-final}, we conclude that
\[
    \operatorname{Var}(f_a)
    \leq
    \frac{
        K_\theta+K_y+K_Z+K_v+K_z
    }{N}
    =
    \frac{K_f}{N}.
\]
The constants are uniform in \(N\), \(t\in[0,1]\), and
\(|a|\leq\varepsilon\), which completes the proof.
\end{proof}

\subsection{Concentration of order parameters}
\begin{prop}\label{prop:conc}
Assume that the functions $u,\phi$ satisfies hypothesis H1-H4 one has 
\be\label{eq:selfavq11}
\mathbb{E}\Big\langle \big| Q_{ab} - \mathbb{E}\langle Q_{ab} \rangle_t \big| ^2\Big\rangle_t
\leq \frac{K}{\sqrt{N}},\quad a,b=1,2
\ee
\be\label{eq:selfavq12}
\mathbb{E}\Big\langle \big| M_{1} - \mathbb{E}\langle M_{1} \rangle_t \big|^2 \Big\rangle_t
\leq \frac{K}{\sqrt{N}},
\ee
for some constant  $K<\infty$ uniformly in $t\in[0,1]$ and $N$.
  
\end{prop}

\begin{proof}
Let's focus now on the concentration properties \eqref{eq:selfavq11} and \eqref{eq:selfavq12}. 
We will prove only  \eqref{eq:selfavq11} for $a\neq b$ namely the concentration property of the overlap $Q_{12}$. The same argument (with only minor modifications) works also for $Q_{11}$ and $M_1$. We start with the following

We will now comlete the proof of \eqref{eq:selfavq12}, namely the concentration property  for  $Q_{12}$. Let's start with the inequality

\be\label{eq:thermaldisorder}
\begin{aligned}
&\E\left \langle \big( Q_{12} - \mathbb{E}\langle Q_{12}\rangle \big)^2\right\rangle\leq 2 \E \left \langle \big( Q_{12} - \langle Q_{12}\rangle)^2\right\rangle+2 \E\big( \langle Q_{12}\rangle  - \mathbb{E}\langle Q_{12}\rangle)^2 \\
&=2\,\E\mathrm{Var}_0(Q_{12})+ 2 \,\E(f'_0-\E f'_0)^2
\end{aligned}
\ee

Recall that for any $a\in R $  the quantity  $\mathrm{Var}_a$ denotes the variance  w.r.t. the random Gibbs measure $\nu_a$ in \eqref{eq:randomGibbsa}. Fix some $a\in\R$ and consider 

\be
\mathrm{Var}_a(Q_{12})=\frac{1}{N^2}\mathrm{Var}_a(\sigma^1\cdot\sigma^2)
\ee
 We already proved in Lemma \ref{LEM:logconcave} that for $|a|$ small enough the  $\nu_a$ is log-concave, therefore by Brascamb-Lieb inequality
\be
\mathrm{Var}_a (\sigma^1\cdot \sigma^2)\leq\int d\bx\mu_a(\bx)  \,\,\nabla^T_{\bx}(\sigma^1\cdot \sigma^2) (\mathrm{Hess} [\mathcal{H}_a])^{-1}\nabla_{\bx}(\sigma^1\cdot \sigma^2)
\ee

where $\mathcal{H}_a$ is defined in  \eqref{interaction}. By  Lemma \ref{LEM:logconcave}, in particular from \eqref{boundHessian}
one gets

\be\label{eq:BLbound]}
\nabla^T_{\bx}(\sigma^1\cdot \sigma^2) (\mathrm{Hess} [\mathcal{H}_a])^{-1}\nabla_{\bx}(\sigma^1\cdot \sigma^2)\leq \frac{1}{\kappa-|a|}\|\boldsymbol{\sigma}\|^2
\ee
Now combining  \eqref{eq:BLbound]} with \eqref{eq:momentbound} one gets that 

\be\label{eq:bound12a}
\begin{aligned}
&\mathrm{Var}_a(\sigma^1\cdot\sigma^2)\leq \frac{1}{\kappa-|a|} \left\langle \|\boldsymbol{\sigma}\|^2\right\rangle_{(a)}\\
&\leq\frac{N}{\kappa-|a|}  C_{a}\Big(1+ K_x N \left(1+\|\underline{G}\|_{op}^{2} +t_m^{2}+t_{\xi}^{2} +t^{2}_v\|\underline{v}\|^{2p} + \|\underline{t_z\vz+t_{\theta}\theta}\|^{2} 
\right)\Big)
\end{aligned}
\ee

Hence $\E\,\mathrm{Var}_0(Q_{12})\leq \frac{K_1}{N}$ for some $K_1>0$ such that $\sup_{t\in[0,1]}K_1<\infty$.
We treat now the second term of the r.h.s. in \eqref{eq:thermaldisorder}. Since the function $a\mapsto f_a$ is convex one has that for any $a>0$ 

\be
\big| f'_0-\E f'_0\Big|\leq f'_{a}-f'_{-a}+\frac{1}{a}\sum_{\delta=\pm a,0}\big|f_{\delta}-\E f_{\delta}\big|
\ee

Now if $|a|$ is small enough, by Proposition \ref{pro:selfavf} one has 
\be
\E \big |f_\delta-\E f_{\delta}\big|^2\leq \frac{K_f}{N}
\ee
for any $\delta=\pm a,0$ and then
\be\label{eq:finalcon1}
\E\Big(\frac{1}{a}\sum_{\delta=\pm a,0}\big|f_{\delta}-\E f_{\delta}\big|\Big)^2\leq\frac{K_2}{Na^2}
\ee

for some $K_2>0$. In order to bound the quantity 
$\E(f'_{a}-f'_{-a})^2$ we will use the identity 
$f'_{a}-f'_{-a}=\int_a^a f''(s)ds$ that implies

\be\E \left(f'_a-f'_{-a}\right)^2\leq 2a\int_{-a}^{a}ds \E\left[(f''_s)^2\right]
\ee
Since

\be
 f''_{s}= N\,\mathrm{Var}_s (Q_{12})= \frac{1}{N}\mathrm{Var}_s (\sigma^1\cdot \sigma^2)
\ee

from \eqref{eq:bound12a}
one obtains

\be
\E \left[(f''_s)^2\right]\leq K_1
\ee
uniformly for  $|s|\leq\varepsilon$  with  $K_1>0$  such  that  $ \sup_{t\in[0,1]}K_1<\infty $. Finally one gets  

\be\label{eq:finalcon2}
\E (f'_{a}-f'_{-a})^2\leq 4 K_ 1 a^2
\ee

Combining \eqref{eq:finalcon1} and \eqref{eq:finalcon2} one gets that for $|a|$ small enough

\be
\E \big(f'_0-\E f'_0 \big)^2\leq 4K_1a^2+\frac{K_2}{Na^2}
\ee
Choosing $a^2=\frac{1}{\sqrt{N}}$ one gets the result.

\end{proof}

\printbibliography

\end{document}